\documentclass{article}

\usepackage[main, preprint]{neurips_2026}

\usepackage[utf8]{inputenc} 
\usepackage[T1]{fontenc}    
\usepackage{url}            
\usepackage{amsfonts}       
\usepackage{nicefrac}       
\usepackage{microtype}      

\usepackage{amsmath,amssymb,amsthm}
\usepackage{mathtools}
\usepackage{bm}
\usepackage{enumitem}
\usepackage{booktabs}
\usepackage{xcolor}
\usepackage[colorlinks=true,linkcolor=blue,citecolor=blue,urlcolor=blue]{hyperref}
\usepackage{cleveref}
\usepackage{natbib}
\usepackage{graphicx}
\usepackage{subcaption}

\newtheorem{theorem}{Theorem}[section]
\newtheorem{proposition}[theorem]{Proposition}

\newtheorem{corollary}[theorem]{Corollary}
\theoremstyle{definition}

\newtheorem{example}[theorem]{Example}
\newtheorem{remark}[theorem]{Remark}

\newcommand{\R}{\mathbb{R}}
\newcommand{\N}{\mathbb{N}}
\newcommand{\bbeta}{\boldsymbol{\beta}}
\newcommand{\E}{\mathbb{E}}
\newcommand{\Cov}{\mathrm{Cov}}
\newcommand{\Normal}{\mathcal{N}}
\newcommand{\pdata}{p_{\mathrm{data}}}
\newcommand{\fdata}{f_{\mathrm{data}}}
\newcommand{\sfinf}{f_{\infty}^\star}
\newcommand{\truncpinf}{p_\infty^{(t_\infty)}}
\newcommand{\phat}{\hat{p}}
\newcommand{\spinf}{p_{\infty}^\star}

\newcommand{\tinf}{t_{\infty}}
\newcommand{\Id}{\mathbf{I}_d}
\newcommand{\Probspace}{\mathcal{P}}

\newcommand{\Law}{\mathrm{Law}}
\newcommand{\Prob}{\mathbb{P}}
	
\newcommand{\indep}{\perp \!\!\! \perp}

\newcommand{\bx}{\mathbf{x}}
\newcommand{\bn}{\mathbf{n}}
\newcommand{\be}{\mathbf{e}}
\newcommand{\by}{\mathbf{y}}
\newcommand{\bv}{\mathbf{v}}

\newcommand{\bmu}{\boldsymbol{\mu}}
\newcommand{\bSigma}{\boldsymbol{\Sigma}}
\newcommand{\bX}{\mathbf{X}}
\newcommand{\bI}{\mathbf{I}}
\newcommand{\I}{\mathrm{Id}}

\newcommand{\bY}{\mathbf{Y}}
\newcommand{\bZ}{\mathbf{Z}}
\newcommand{\bB}{\mathbf{B}}

\newcommand{\cU}{\mathcal{U}}
\newcommand{\cK}{\mathcal{K}}
\newcommand{\cS}{\mathcal{S}}
\newcommand{\cN}{\mathcal{N}}
\newcommand{\cT}{\mathcal{T}}

\newcommand{\bs}{\mathbf{s}}

\newcommand{\cW}{\mathcal{W}}
\newcommand{\md}{\mathrm{d}}

\definecolor{color_learned}{HTML}{E76F51}
\definecolor{color_ideal}{HTML}{2471A3}

\title{Recursively Trained Diffusion Models: Limiting Collapse Distribution and Spectral Characterization}

\author{
  Nail B. Khelifa  \\
  University of Cambridge\\
  \texttt{nbk24@cam.ac.uk} \\
  \And
  Richard E. Turner \\
  University of Cambridge \\
  \texttt{ret26@cam.ac.uk} \\
  \And
  Ramji Venkataramanan \\
  University of Cambridge \\
  \texttt{rv285@cam.ac.uk} \\
}

\begin{document}

\maketitle

\begin{abstract}
Recursive training of generative models on their own outputs can lead to model collapse, a compounding drift away from the true data distribution. Existing theoretical works bound finite-round error accumulation in the context of diffusion models, but two questions remain open:~what distribution does the recursion converge to, and how fast? We answer both, isolating a mechanism distinct from imperfect learning: even with perfect score estimation and exact sampling, the early stopping of the reverse diffusion (required for numerical stability) drives a progressive drift away from the data distribution. We prove that this recursion converges geometrically to a unique limiting distribution, which admits a closed-form characterization as an infinite mixture of increasingly Gaussian-smoothed versions of the data distribution. A Hermite spectral decomposition of this limit reveals that recursive training acts as a low-pass filter: higher-order modes, which encode fine non-Gaussian structure, are attenuated much more strongly than coarse modes. This spectral picture motivates annealed truncation schedules that progressively shrink truncation times across retraining rounds; we prove that any schedule converging to $0$ asymptotically eliminates recursive compounding. Finally, we show our idealized characterization is robust: in the presence of discretization and score estimation errors, the learned distribution remains in a Wasserstein-2 ball around the ideal limit, with mode-dependent contraction rates that contract high-order errors faster than low-order ones. We validate the theory on synthetic Gaussian mixtures and CIFAR-10.
\end{abstract}

\section{Introduction}
\label{sec:intro}

Following the success of generative modeling, the use of synthetic data for training has attracted considerable attention. Such generated data may be used deliberately to improve or fine-tune existing models \cite{gulcehre2023reinforcedselftrainingrestlanguage, alemohammad2024sims, zelikman2022star}, or may enter training sets inadvertently when data are polluted with outputs of generative models \cite{carlini2024datapoisoning}. It has been found that repeatedly retraining models on their own outputs causes systematic drift away from the target distribution \cite{alemohammad2023selfconsuminggenerativemodelsmad, mc_shumaylov}, a phenomenon referred to as \textit{model collapse}. Subsequent work characterized collapse
as a loss of the tails of the target distribution and a decrease in the diversity of generated samples \cite{dohmatob2024tale, shi2026a}. In settings that allow for a combination of synthetic and real data in each retraining iteration, different behaviors have been observed: in some cases, a small proportion of synthetic data can significantly degrade performance \cite{dohmatob2024strong}, while in others, it has been shown that collapse can be mitigated \cite{model-collapse-inevitable, model-collapse-demystified,dey2025universality, analyzing-mitigating-model-collapse, garg2026preventing} or completely avoided \cite{Barzilai26when-models-dont}. In some settings, synthetic data could even be beneficial when used judiciously \cite{jain2024scaling, dohmatob2024tale}. 

In the context of diffusion models \cite{seminal-diffusion-models, song2019generative, song2021scorebased, DDPM}, recent theory has quantified finite-round error propagation under population-level, finite-sample, or model-specific assumptions \cite{mc-diffusion-sample-level, mc-diffusion-solvable-model, khelifa2026errorpropagationmodelcollapse, analyzing-mitigating-model-collapse}. However, these works do not characterize the asymptotic behavior of iteratively retrained diffusion models, i.e., whether the model tends toward a collapsed regime in which further retraining no longer changes the learned distribution.

\paragraph{Setting} Let $\pdata$ denote the true target data distribution on $\mathbb{R}^d$. We consider a recursive procedure which, in each training round (generation), mixes fresh samples from $\pdata$ with synthetic samples generated by the current model. At generation $i \geq 0$, a proportion $\alpha \in (0,1)$ of the training data consists of fresh samples from $\pdata$, while the remaining proportion $(1-\alpha)$ consists of synthetic samples drawn from the learned model at the previous generation, denoted $\phat^i$. To capture a realistic setting where training is agnostic to how samples were generated, we assume samples are not labeled as fresh or synthetic. As a result, the effective training distribution at generation $i$ is the mixture
\begin{equation}
\label{eq:mixture}
q_i := \alpha\, \pdata + (1-\alpha)\,\phat^i.
\end{equation}
A score network is then trained on samples from $q_i$, and the resulting diffusion model defines the next synthetic distribution $\phat^{i+1}$. Starting from $\phat^0 =\pdata$, this defines the recursion
\begin{equation}
\phat^i 
\;\xrightarrow{\text{mix with } \pdata}\; 
q_i 
\;\xrightarrow{\text{train model}}\; 
\phat^{i+1}.
\label{eq:recursive-structure}
\end{equation}
This framework matches that of existing works \cite{khelifa2026errorpropagationmodelcollapse}, and may be generalized to a setting where each past generation $k \le i$ is sampled from in proportion $\alpha_k$ \cite{mc-diffusion-sample-level}. The goal of this work is to answer the following questions:
\begin{center}
\textit{In this recursive setting, does the sequence $(\phat^i)_{i\geq0}$ converge to a specific distribution as $i$ grows? If so, can it be characterized explicitly, and at what rate does the recursion converge to it?}
\end{center}

Understanding this asymptotic behavior is fundamental because it identifies the regime towards which a model is driven by recursive training. Analyzing the collapsed regime highlights the role of each source of error and the mechanisms by which collapse occurs.

\paragraph{Sources of error} \label{paragraph:sources-of-errors}
Training and sampling a diffusion model in this recursive setting introduce four distinct sources of error. The first, and the main focus of this paper, is \emph{truncation}: in practice, the reverse SDE is only integrated on $[t_0, T]$ for some $t_0 > 0$, because score estimation becomes numerically unstable as $t\downarrow0$ \cite{song2020improved_ncsn, kim2022soft} (formalized in Appendix \ref{app:self-regularization}). While modern samplers introduce various corrections (e.g. Tweedie-style denoising \cite{Efron01122011}, $x_0$-prediction \cite{karras2022elucidating}, and other post-truncation refinements \cite{zhang2024tackling}), the requirement of stopping the reverse SDE at $t_0>0$ is standard. As a result, even with perfect score estimation, exact initialization, and exact numerical sampling, the generated distribution is not $q_i$ itself but a Gaussian-smoothed version of $q_i$. The second source of error is the score estimation error, caused by finite data and function-approximation constraints. The third is the discretization error, introduced when the reverse SDE is solved numerically. The fourth is the initialization error: the reverse SDE is started from $\cN(0,\Id)$ rather than from $q_{i,T}$, which induces an error that decays exponentially in $T$.
There is a rich body of work providing non-asymptotic guarantees for diffusion-based models, accounting for score estimation errors under different discretization schemes \cite{sampling-is-as-easy-as-learning-the-score, benton2024nearly, convergence-manifold-hypotheses-vdb, chen-2023, minimax-optim-score-based-diff-models, wasserstein, confortiKLguarantees, samworth-shape-constraint}.

Our analysis first isolates truncation and shows that it alone induces a progressive drift away from $\pdata$ with recursive training. Errors due to imperfect score estimation, time discretization and initialization are then treated as perturbations of this idealized mechanism.

\paragraph{Contributions}
This paper develops an operator-theoretic framework for recursively trained diffusion models and characterizes their limiting behavior. Our main contributions are:
\begin{itemize}[leftmargin=2.5em]
  \item \textbf{Explicit collapse distribution.}
  We identify truncation as a primary driver of model collapse and show that, in the ideal case with no  discretization or score estimation errors, the recursive dynamics \eqref{eq:recursive-structure} converge to a unique limiting distribution $\spinf$. We derive a closed-form expression for $\spinf$ as an infinite mixture of progressively Gaussian-smoothed versions of $\pdata$, and establish a geometric rate of convergence  to that limit in the error-free regime (Theorem~\ref{thm:collapse-distrib-existence-uniqueness}). 
 
  \item \textbf{Spectral structure of collapse.}
  Decomposing the limiting distribution $\spinf$ in the Hermite polynomial basis, we provide a spectral characterization of model collapse (Proposition~\ref{prop:spectral-rpz-limit}), which shows that recursive training acts as  a low-pass filter on $\pdata$, suppressing high-order components more strongly than low-order ones, with explicit mode-dependent attenuation controlled by $\alpha$ and the stopping time $t_0$.
  \item \textbf{Annealed truncation schedules.} Having identified fixed truncation as a driver of model collapse, we introduce \emph{annealed} truncation schedules with generation-adaptive truncation times $t_0^{(N)}$. We prove that any schedule converging to zero asymptotically eliminates recursive compounding in the error-free regime; schedules converging to a positive limit attenuate but do not remove compounding (Theorem \ref{thm:annealed-truncation-corrected}). 
  
  \item \textbf{Robustness and rate of collapse under imperfect training.}
  We extend the analysis beyond the ideal setting to account for discretization and score estimation errors. We show that the recursive dynamics remain stable, converging at an exponential rate to a Wasserstein-2 ball around the ideal limiting distribution $\spinf$ (Theorem~\ref{thm:convergence-imperfect-regime} and Proposition \ref{prop:spectral-perturbation}). Thus, the limiting collapse regime becomes relevant even with a moderate number of rounds of recursive training.
\end{itemize}

\paragraph{Organization}
Section~\ref{sec:background} introduces the mathematical background and the assumptions underlying our analysis. Section~\ref{sec:fixedpoint} shows that, even in an idealized regime (perfect score estimation and exact sampling), model collapse occurs because of truncation of the reverse SDE as $t\downarrow0$. Section~\ref{sec:annealed-truncation} shows that allowing the truncation time to vary across generations changes the asymptotic behavior of the recursion. Section~\ref{sec:imperfect-sampling} provides robustness results for the idealized framework in the presence of errors. 
\paragraph{Notation}
Bold symbols (e.g.\ $\mathbf X_t$, $\mathbf Y_t$) denote $\mathbb R^d$-valued random variables or processes, and $\Probspace_2(\R^d)$ denotes the space of Borel probability measures on $\R^d$ with finite second moment. The space of continuous functions from $[0, T]$  to $\R^d$ is denoted by $C([0,T],\R^d)$. For random variables $\bX_1, \bX_2$ defined on the same probability space, we write $\bX_1 \indep \bX_2$ when $\bX_1$ and $\bX_2$ are independent. For a stochastic process $(\bZ_t)_{t\in[0, T]}$ defined on $C([0,T],\R^d)$, we write $\Law((\bZ_t)_t)$ for its law on path space and $\Law(\bZ_t)$ for the law of its marginal at time $t$. For two measures $\mu$ and $\nu$, we write $\mu \ll \nu$ to denote that $\mu$ is absolutely continuous with respect to $\nu$, meaning that every $\nu$-null set is also $\mu$-null. All probability measures $\mu \in \Probspace_2(\R^d)$ are assumed absolutely continuous with respect to the Lebesgue measure on $\R^d$ and, with some abuse of notation, we also use $\mu$ to refer to their density. Calligraphic symbols (e.g. $\cT$, $\cU$, $\cS$) refer to operators on probability distributions i.e., mapping $\Probspace_2(\R^d) \to \Probspace_2(\R^d)$, and $\I$ is the identity operator ($\I(\mu) = \mu$ for $\mu \in \Probspace_2(\R^d)$).  For $\mu, \nu \in \Probspace_2(\R^d)$, the Wasserstein-2 distance between $\mu$ and $\nu$, denoted $\cW_2(\mu, \nu)$ is defined as 
\begin{equation}
  \cW_2(\mu, \nu) \;:=\; \left(\inf_{\pi \in \Pi(\mu,\nu)}
  \int_{\R^d \times \R^d} \|\bx - \by\|^2 \,\pi(\md\bx, \md\by)\right)^{1/2},
\end{equation}
where $\Pi(\mu, \nu)$ denotes the set of all couplings of $\mu$ and $\nu$,
i.e., probability measures on $\R^d \times \R^d$ with marginals $\mu$ and $\nu$, respectively.

\section{Background} \label{sec:background}

\paragraph{Forward diffusion}
The $i$-th generation diffusion model is defined by a variance-preserving Ornstein-Uhlenbeck (OU) process \cite{sto-calculus-2,song2021scorebased} on the interval $[0,T]$, initialized at the mixture $q_i$:
\begin{equation} 
\md \bX_t^i
=
-\tfrac12 \bX_t^i\,\md t + \md \bB_t, 
\qquad 
\bX_{0}^i \sim q_i,
\label{eq:ou-forward}
\end{equation}
where $(\bB_t)_{t \in [0,T]}$ is a standard $d$-dimensional Brownian motion. For each time $t \in [0, T]$, the OU process \eqref{eq:ou-forward} defines a sampling operator $\cU_t : \Probspace_2(\R^d) \to \Probspace_2(\R^d)$ \cite{ethier-kurtz, oskendal-sde, Bakry2014Analysis, da2014-infinite-sde} that propagates the initial distribution $q_i$ through the diffusion as follows:
\begin{equation}
    \cU_t q_i:=\Law(e^{-t/2}\bX_0 + \sqrt{1 - e^{-t}}\bZ), \quad \bX_0 \sim q_i \indep \bZ \sim \cN(0, \bI_d).
    \label{eq:UtOU}
\end{equation}
Therefore, we denote the time-$t$ marginal of $q_i$ in \eqref{eq:ou-forward} by
$$
q_{i,t} := \Law(\bX_t^i) = \cU_t q_i.
$$
The sampling operator $\cU_t$ is central to our analysis, and we provide extended background on its properties in Appendix \ref{app:extended-background}.

\paragraph{Reverse-time generation and truncation}
Under standard regularity conditions \cite{anderson1982reverse,time-reversal-of-diffusions}, the time reversal of \eqref{eq:ou-forward} solves the reverse-time SDE integrated backward from $s=T$ to $s=0$,
\begin{equation}
\md\mathbf{Y}_s^{i, \star}
=
\big[
-\tfrac{1}{2}\mathbf{Y}_s^{i, \star} 
- 
\nabla_\bx \log q_{i,s}(\mathbf{Y}_s^{i, \star})\big]\, \md s
+ \md\bar{\mathbf{B}}_s,
\label{eq:reverse-sde-ideal}
\end{equation}
where $(\bar{\mathbf{B}}_s)_{s \in [0, T]}$ is another Brownian motion on the same probability space. If we run the reverse-time SDE starting from $\mathbf{Y}_T^{i, \star} \sim q_{i,T}$, then $\mathbf{Y}_{0}^{i, \star} \sim q_{i,0}$  \cite{anderson1982reverse, time-reversal-of-diffusions}.  The true scores $\nabla \log q_{i,t}$ are unknown and are approximated by a learned time-dependent network $\bs_{\theta_i}$ trained on samples from  $q_i$ in \eqref{eq:mixture}. However, since $q_i$ may be non-smooth or singular, the score $\nabla \log q_{i,t}$ can become unstable as $t\downarrow 0$ \cite{score-approx-chen-low-dim,zhang2024tackling}. Therefore, as is common in the literature \cite{khelifa2026errorpropagationmodelcollapse, benton2024nearly, sampling-is-as-easy-as-learning-the-score}, we work on a truncated interval $[t_0, T]$ with $t_0 >0$, where the diffusion has already regularized the law: under suitable nondegeneracy and regularity assumptions  \cite{aronson1967, aronson1968}, $q_{i,t}$ admits a smooth density for every $t>0$, so that the score is well defined and better behaved away from the singular endpoint. 
Substituting the learned score $\bs_{\theta_i}$ into \eqref{eq:reverse-sde-ideal} yields
\begin{equation}
\md \hat \bY_s^i
=
\big[-\tfrac12 \hat \bY_s^i - \bs_{\theta_i}(\hat \bY_s^i,s)\big]\md s
+ \md \bar \bB_s ,
\label{eq:reverse-sde-learned}
\end{equation}
with initialization $\hat \bY_T^{i}\sim \mathcal{N}(0,\bI_d)$.
The next generation model is therefore $\phat^{i+1} := \Law(\hat{\bY}_{t_0}^i)$.

\paragraph{Sampling operators} \label{paragraph:sampling-operators} The reverse SDEs   \eqref{eq:reverse-sde-ideal}  and \eqref{eq:reverse-sde-learned} define two sampling procedures. In the absence of score estimation and discretization errors, the reverse sampling operator for \eqref{eq:reverse-sde-ideal} is exactly $\cU_{t_0}$ and $\phat^{i+1} = \cU_{t_0}(q_i)$. Accounting for these two errors, the \textit{learned sampling operator} at generation $i$, denoted $\hat{\cS}_i$ is defined as  $\hat{\cS}_i(q_i) = \phat^{i+1}$. These operators formally define the one-step recursion in \eqref{eq:recursive-structure}.

\paragraph{Regularity conditions} 
Let $\Probspace_{\mathrm{reg}}(\R^d)$ be the set of probability distributions $\mu$ satisfying the following assumptions:
\begin{enumerate}[label=\textbf{(A\arabic*)}, ref=A\arabic*, leftmargin=*, resume=assumptions]
  \item\label{A1} $\mu$ admits a strictly positive smooth density with respect to the Lebesgue measure:
        $\mu(\bx) > 0$ for all $\bx \in \R^d$.
  \item\label{A2} Finite second moment:
        $\int_{\R^d} \|\bx\|^2 \mu(\bx)\,\md\bx < \infty$.
  \item\label{A3} The score $\nabla\log\mu(\bx)$ is
        square-integrable under $\mu$.
\end{enumerate}
Throughout, we assume that $\pdata \in \Probspace_{\mathrm{reg}}(\R^d)$, and observe that assumptions~\ref{A1}--\ref{A3} are preserved under the forward diffusion \eqref{eq:ou-forward} and under convex combinations.

\section{Limiting Collapse Distribution in the Error-free Regime}
\label{sec:fixedpoint}

In this section, we isolate the effect of truncation by studying an idealized error-free regime in which score estimation and discretization errors are ignored and the reverse SDE is initialized exactly at $q_{i,T}$. This setting shows that collapse does not require imperfect learning: truncation of the reverse SDE (at $t_0 >0$) alone is sufficient to drive the recursive dynamics toward a nontrivial limiting distribution.

\paragraph{Time reversal and one-step dynamics} In the error-free regime, the transition from one generation to the next is defined by the true reverse SDE \eqref{eq:reverse-sde-ideal} whose marginals, by time-reversal \cite{ anderson1982reverse, time-reversal-of-diffusions}, exactly match those of the forward SDE \eqref{eq:ou-forward}. For each $i \ge 0$, the model $\phat^{i+1}$ can therefore be expressed using the OU operator \eqref{eq:UtOU} applied at the truncation time $t_0$: $\phat^{i+1} = \cU_{t_0}(q_i) = \Law(e^{-t_0/2}\bX_0^i\, + \, \sqrt{1-e^{-t_0}} \bZ)$. Thus truncation alone introduces a mismatch between $\phat^{i+1}$ and $q_i$, the target distribution at generation $i$, and it corresponds to a small amount of residual Gaussian noise. Since this smoothing is repeated recursively, Gaussian-convolved residuals accumulate over generations, and truncation alone leads to progressive drift away from the true data distribution $\pdata$.

\paragraph{Collapse distribution}
Since $\phat^{i+1} = \cU_{t_0}(\alpha \pdata + (1-\alpha)\phat^i)$, the transition from the $i$-th to the $(i+1)$-th generation is generation-independent. Therefore, the limiting behavior of the recursion amounts to the fixed-point problem of
\begin{equation}
\label{eq:fixed-point-problem}
\mu \mapsto \cU_{t_0}\bigl(\alpha \pdata + (1-\alpha)\mu\bigr).
\end{equation}
The following theorem shows that this recursion admits a unique fixed point and expresses it as a Neumann series (proof in Appendix \ref{app:proof-thm-collapse}).

\begin{theorem}[Collapse distribution: existence, uniqueness, and explicit form]
\label{thm:collapse-distrib-existence-uniqueness} Under \ref{A1}--\ref{A3}, there exists a unique collapse distribution $\spinf \in \Probspace_2(\R^d)$ that is a fixed point of \eqref{eq:fixed-point-problem}. It is given by,
\begin{equation}\label{eq:neumann}
  \spinf \;=\; \alpha\sum_{k=0}^\infty (1-\alpha)^k\,\cU_{(k+1)t_0}(\pdata).
\end{equation}
Moreover, in the absence of score estimation, discretization and initialization errors, the sequence $(\phat^N)_{N\ge1}$ converges geometrically fast to $\spinf$: $\cW_2(\phat^N, \spinf)
\;\le\;
\kappa^N\,\cW_2(\pdata, \spinf)$ where $\kappa= \sqrt{1-\alpha}\,e^{-t_0/2}$.
\end{theorem}

Theorem \ref{thm:collapse-distrib-existence-uniqueness} shows that the collapse distribution is a mixture, with geometric weights, of increasingly smoothed copies of the true data distribution. The $k$th component of the mixture, $\cU_{(k+1)t_0}(\pdata)$, is the data distribution smoothed by $(k+1)t_0$ units of OU evolution, and it is weighted by $\alpha(1-\alpha)^k$. The rate of convergence depends on both the fresh-data proportion $\alpha$ and the truncation time $t_0$ through $\kappa=\sqrt{1-\alpha}\,e^{-t_0/2}$. The larger the $t_0$, the larger the Gaussian noise injection at each generation, and the faster the convergence; the larger the $\alpha$, the faster the convergence as well, since the fresh component of $q_i$ is identical across generations and only the synthetic component, of weight $(1-\alpha)$, carries memory of the past. The two parameters act differently on the \emph{limit}, however: a larger $\alpha$ brings $\spinf$ closer to the one-step smoothed distribution $\cU_{t_0}(\pdata)$, while a smaller $t_0$ brings it closer to $\pdata$ (Corollary~\ref{cor:limiting-behavior}, proof in Appendix \ref{app:proof-limiting-behavior}).

\begin{corollary}[Limiting behaviors]\label{cor:limiting-behavior}
The exact collapse distribution~\eqref{eq:neumann} satisfies the following:
\begin{enumerate}[label=(\roman*)]
  \item $\alpha \to 1^-$: $\cW_2(\spinf, \; \cU_{t_0}(\pdata)) \to 0$.
  \item $\alpha \to 0^+$: $\cW_2(\spinf,\;\Normal(0, \Id)) \to 0$.
  \item $t_0 \to 0^+$: $\cW_2(\spinf,\; \pdata) \to 0$.
  \item For any $\alpha \in (0,1)$, $t_0 > 0$, and $\pdata \not\equiv \Normal(0,\Id)$:
        $\cW_2(\spinf,\; \pdata) > 0$.
\end{enumerate}
\end{corollary}
These limits have simple interpretations. When $\alpha \to 1$, the recursion effectively disappears and only the one-step truncation bias remains. When $\alpha \to 0$, the fraction of fresh data injected in each retraining generation vanishes, letting Gaussian smoothing compound without attenuation; this ultimately drives the recursion towards the standard Gaussian (the invariant measure of the OU semigroup). When $t_0 \to 0^+$, truncation vanishes and the collapse distribution approaches $\pdata$. 

\begin{example}[Gaussian mixture]
\label{ex:gmm}
An interesting example to understand $\spinf$ is the case where $\pdata$ is a Gaussian mixture, i.e. $\pdata=\sum_{j=1}^J w_j \,  \Normal(\bmu_j,\bSigma_j)$, with non-negative weights $(w_j)_{j \in [J]}$ summing to $1$, means $\bmu_j \in \R^d$ and positive definite covariance matrices $\bSigma_j \in \R^{d\times d}$ for $j \in [J]$. Since
$\cU_t\Normal(\bmu,\bSigma)=\Normal\bigl(e^{-t/2}\bmu,\,e^{-t}\bSigma+(1-e^{-t})\Id\bigr)$, from 
\eqref{eq:neumann} the collapse distribution is
\begin{equation}\label{eq:gmm-collapse}
\spinf=\sum_{k=0}^\infty\alpha(1-\alpha)^k\sum_{j=1}^J w_j\,
\Normal\Bigl(\lambda_k^{1/2}\bmu_j,\;\lambda_k\bSigma_j+(1-\lambda_k)\Id\Bigr),
\qquad \lambda_k:=e^{-(k+1)t_0}.
\end{equation}
Each component of the outer sum is a Gaussian mixture, with the $k$th term representing the effect of training samples that have gone through $k+1$ rounds of truncated sampling. As $k$ increases the means are contracted toward $0$ and the covariances pulled toward $\Id$. Two consequences follow. First, modes disappear progressively over retraining rounds rather than all at once. Indeed, for a bimodal mixture
$\pdata=\tfrac12\Normal(-a,s^2)+\tfrac12\Normal(a,s^2)$ on $\R$ with $a>s$, $k$th block is bimodal if and only if $(k+1)t_0<\log(1+a^2-s^2)$, so only the most recent  blocks still resolve both modes. Second, $\spinf$ is not Gaussian even when $\pdata$ is:
for $\pdata=\Normal(0,s^2)$ with $s^2\neq1$, \eqref{eq:gmm-collapse} is a scale mixture of Gaussians with strictly positive excess kurtosis.
\end{example}

\subsection{Spectral decomposition of the collapse distribution}

While Theorem~\ref{thm:collapse-distrib-existence-uniqueness} gives a closed-form description of the collapse distribution, it does not directly identify how different features of the data distribution are affected by recursive training. In particular, the representation~\eqref{eq:neumann} does not
distinguish between low-order structure (e.g. mean, covariance) and finer non-Gaussian features (e.g. tails, oscillations, multimodality). We now refine this description by studying collapse mode by mode. Specifically, we analyze $\spinf$ in the Hermite polynomial basis, the spectral coordinates naturally associated with the Ornstein-Uhlenbeck dynamics. The functional analytic details for this section are collected in Appendix~\ref{app:extended-background}.

Let $\gamma = \cN(0, \bI_d)$, and let $L^2(\gamma)$ be the space of square-integrable functions with respect to $\gamma$. The inner product $\langle\cdot, \cdot\rangle_{L^2(\gamma)}$ on this space is defined as
$
\langle f, g \rangle_{L^2(\gamma)} = \int f(\bx) g(\bx) \gamma(\md \bx)
$.
    
\textbf{Decomposition of $\cU_t$ in Hermite Polynomial Basis}
For any multi-index $\bn=(n_1, \ldots, n_d) \in \N^d$, the multivariate normalized Hermite polynomial \cite{bogachev2015gaussian-measures} $H_\bn$ is defined as
$
H_{\bn}(\bx):=\tfrac{1}{\sqrt{\bn!}}\prod_{j=1}^d H_{n_j}(x_j),
$
where $H_k$ denotes the $k$-th univariate Hermite polynomial (defined in \eqref{eq:univariate_Hermite}). The degree of  $H_\bn$ is denoted by $|\bn| = n_1 + \cdots + n_d$. The family $\{H_{\bn}\}_{\bn\in\mathbb N^d}$ forms an orthonormal basis of $L^2(\gamma)$.  That is, any $g \in L^2(\gamma)$ can be decomposed as  \cite{bogachev2015gaussian-measures, janson1997gaussian-hilbert-spaces}:
\begin{equation*}
g = \sum_{\bn} \langle g, H_{\bn} \rangle_{L^2(\gamma)} H_{\bn},
\end{equation*}

Now, for any probability measure $\mu \ll \gamma$, let $f=\md\mu/\md\gamma$ be its Radon-Nikodym derivative (density ratio w.r.t. $\gamma$),  assumed to be in $L^2(\gamma)$.
Then, writing $f=\sum_{\bn} a_{\bn}H_{\bn}$, we have (Proposition \ref{prop:decomposition-Ut-hermite-basis}), 
\begin{equation}
\label{eq:sampling-OU-decomposition-Hn}
\frac{\md(\cU_t\mu)}{\md\gamma}
=
\sum_{\bn\in\mathbb N^d}
e^{-|\bn|t/2} a_{\bn}H_{\bn}.
\end{equation}
That is, OU smoothing attenuates higher-degree Hermite modes of the density ratio more strongly than lower-degree ones. Applying the decomposition \eqref{eq:sampling-OU-decomposition-Hn} to the representation \eqref{eq:neumann} of $\spinf$ requires that (i) $\pdata \ll  \gamma$, which is guaranteed by \ref{A1}--\ref{A3}, and (ii) $\frac{\md \pdata}{\md \gamma} \in L^2(\gamma)$, which we assume.

\begin{enumerate}[label=\textbf{(A\arabic*)}, ref=A\arabic*, leftmargin=*, resume=assumptions]
\item \label{ass:spectral-regularity-A4} The data distribution \(\pdata\) is such that $f_{\mathrm{data}} := \frac{\md \pdata}{\md \gamma} \in L^2(\gamma)$.
\end{enumerate}

Assumption~\ref{ass:spectral-regularity-A4} fails when $\pdata$ has heavy tails relative to $\gamma$, and when it has no density with respect to $\gamma$ (e.g.\ empirical measures or data supported on low-dimensional sets). For image datasets, pixel values are bounded, so $\pdata$ is compactly supported
and $\md\cU_t\pdata/\md \gamma\in L^2(\gamma)$ for every $t>0$. Under \ref{ass:spectral-regularity-A4}, $\spinf$ is also absolutely continuous with respect to $\gamma$, with density denoted by
$
f_\infty^\star := \frac{\md \spinf}{\md \gamma}
$.
The following result shows that the collapse dynamics act diagonally in the Hermite basis, with an explicit mode-wise attenuation factor (proof in Appendix \ref{app:proof-spectral-rpz-limit}).

\begin{proposition}[Spectral representation of the collapse distribution]
\label{prop:spectral-rpz-limit}
Under Assumptions \ref{A1}, \ref{A2}, \ref{A3} and \ref{ass:spectral-regularity-A4}, the density $f_\infty^\star =  \frac{\md \spinf}{\md \gamma}$ of the \textit{error-free} collapse distribution $\spinf$ admits the expansion
\begin{equation}\label{eq:spectral}
  f_\infty^\star
  =
  \sum_{\bn\in\mathbb N^d}
  m_{\bn}(\alpha,t_0)\,
  \big\langle f_{\mathrm{data}}, H_{\bn}\big\rangle_{L^2(\gamma)}\, H_{\bn},
  \qquad
  m_{\bn}(\alpha,t_0)
  =
  \frac{\alpha\,e^{-|\bn|t_0/2}}
       {1-(1-\alpha)e^{-|\bn|t_0/2}}.
\end{equation}
Equivalently, for every \(\bn\in\mathbb N^d\), we have 
$
\big\langle f_\infty^\star, H_{\bn}\big\rangle_{L^2(\gamma)}
=
m_{\bn}(\alpha,t_0)\,
\big\langle f_{\mathrm{data}}, H_{\bn}\big\rangle_{L^2(\gamma)}.
$
\end{proposition}

\begin{figure}[t]
    \centering
    \includegraphics[width=1.0\linewidth]{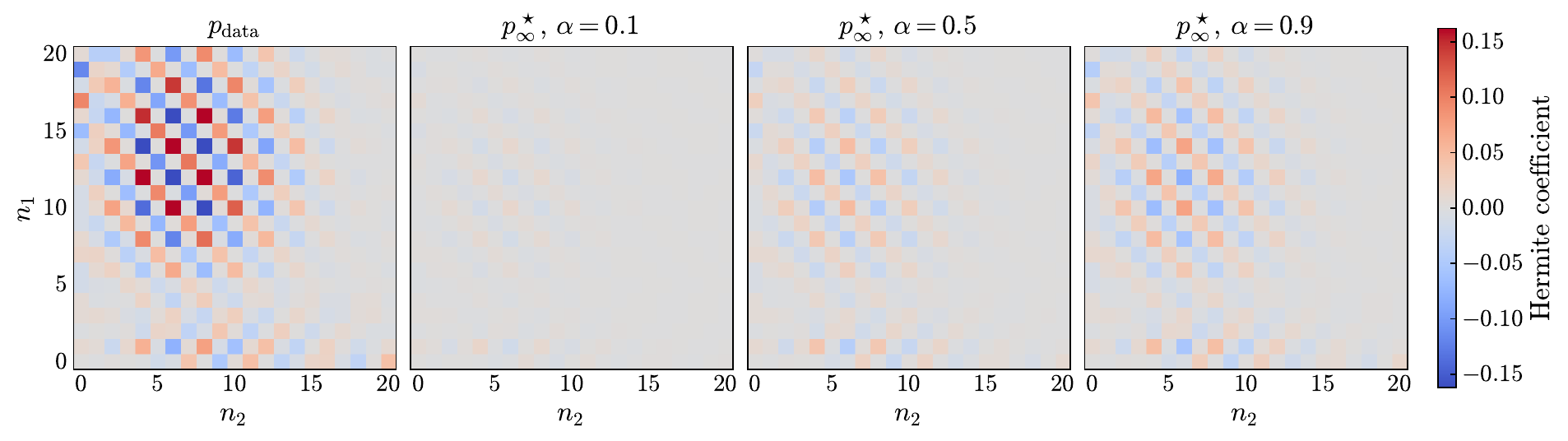}
    \caption{\textbf{Hermite-mode attenuation of $\spinf$ for varying $\alpha$.} Each heatmap displays the coefficient $\left\langle f, H_{n_1,n_2} \right\rangle_{L^2(\gamma)}$ of the density ratio \(f = \md p / \md\gamma\) in the two-dimensional Hermite basis, with \(n_1\) and \(n_2\) denoting the polynomial degrees along the two coordinates.
    \textit{Left}: decomposition of $\pdata$, chosen as an oscillatory distribution with substantial energy in medium- and high-degree Hermite modes. \textit{Remaining panels}: similar spectral decomposition of the fixed-point distribution $\spinf$ for $\alpha \in \{0.1, 0.5, 0.9\}$. As predicted by Proposition~\ref{prop:spectral-rpz-limit}, the attenuation is stronger for smaller $\alpha$. Code is available at \href{https://github.com/NailKhelifa/Recursively_Trained_Diffusion_Models}{GitHub}.}
    \label{fig:hermite-pdata-vs-spinf-vs-alpha}
\end{figure}

Equation~\eqref{eq:spectral} shows that the collapse distribution is obtained from $\pdata$ by a mode-wise attenuation in the Hermite basis relative to the Gaussian reference measure $\gamma$. The attenuation factor $m_\bn$ is monotone decreasing in the total degree $|\bn|$, meaning that high-order polynomials, which encode increasingly fine non-Gaussian features such as tails, oscillations, and multimodality, are suppressed more strongly than low-order ones, which encode Gaussian-like structure of a distribution relative to $\gamma$. We stress that this is a low-pass filter in the Hermite basis of $L^2(\gamma)$, not in spatial frequency: attenuating Hermite modes pulls the distribution toward $\gamma$, whose samples are white noise, so in pixel space truncation \emph{adds} high-frequency power (Figure~\ref{fig:truncation-affect-collapse}, right). This observation matches previous results in which model collapse has, for example, been associated with a degradation of tails \cite{mc_shumaylov, alemohammad2023selfconsuminggenerativemodelsmad, curse-schumailov}. Moreover, the larger the proportion of fresh data injected in each  generation of retraining, the weaker the attenuation of each modes of $\pdata$ affected by the recursion, consistent with similar observations in  \cite{bertrand2024stability, model-collapse-inevitable, khelifa2026errorpropagationmodelcollapse, mc-diffusion-sample-level}. Proposition \ref{prop:spectral-rpz-limit} thus completes Theorem \ref{thm:collapse-distrib-existence-uniqueness} by quantifying the mode-wise effect of \textit{truncation} on $\pdata$. Figure \ref{fig:hermite-pdata-vs-spinf-vs-alpha} illustrates the attenuation of high-degree Hermite modes of $\spinf$ on a two-dimensional distribution for varying $\alpha$, and how this evolves with the proportion of fresh data.

\begin{remark}
\label{rem:moment-attenuation}
Assumption~\ref{ass:spectral-regularity-A4} is used only to expand $\fdata$ in the Hermite basis of $L^2(\gamma)$ and is not required for the attenuation mechanism itself. For each  $\bn \in \N^d$, the corresponding Hermite coefficient is a moment $\langle \fdata,H_\bn\rangle_{L^2(\gamma)}=\E_{\pdata}[H_\bn(\bX)]$ which exists whenever $\E_{\pdata}[\|\bX\|^{|\bn|}]<\infty$. Under \ref{A1}--\ref{A3}, for every such $\bn$,
\begin{equation}\label{eq:moment-attenuation}
  \E_{\spinf}\bigl[H_\bn(\bX)\bigr] \;=\; m_\bn(\alpha,t_0)\,\E_{\pdata}\bigl[H_\bn(\bX)\bigr].
\end{equation}
(Appendix~\ref{app:proof-moment-attenuation}). Heavy tails thus shrink the set of finite Hermite moments (for which \eqref{eq:moment-attenuation} is defined), with the same attenuation factors as in Proposition \ref{prop:spectral-rpz-limit}. Assumption \ref{ass:spectral-regularity-A4} is only required for combining the
coefficients to define $\sfinf\in L^2(\gamma)$. 
\end{remark}

\section{Mitigating Model Collapse with Annealed Truncation Schedules}
\label{sec:annealed-truncation}

Prior studies of model collapse for recursively trained general generative models identified the fresh-data proportion $\alpha$ as a key parameter to stabilize or attenuate recursive degradation \cite{bertrand2024stability, model-collapse-inevitable, khelifa2026errorpropagationmodelcollapse, mc-diffusion-sample-level}. However, for diffusion models, Theorem \ref{thm:collapse-distrib-existence-uniqueness} and Proposition \ref{prop:spectral-rpz-limit} show that for any fresh-data proportion $\alpha<1$,  there is a smoothing bias created by the fixed truncation time $t_0>0$,  which implies $\cW_2(\spinf,\; \pdata) > 0$. Figure \ref{fig:truncation-affect-collapse} (Appendix \ref{app:experiments}) illustrates the effect of $t_0$ on collapse when varying $\alpha$ on CIFAR-10 \cite{cifar10}, showing that while a higher fresh-data injection rate reduces the severity of collapse (the limit $\spinf$ moves closer to $\cU_{t_0}\pdata$), it does not prevent it altogether. This suggests a complementary intervention on truncation: modifying the sampling operator across generations by progressively decreasing the truncation times. The resulting annealed schedule targets the specific source of bias identified above, namely the residual OU smoothing due to early stopping.

\begin{figure}[t]
    \centering
    \includegraphics[width=1.0\linewidth]{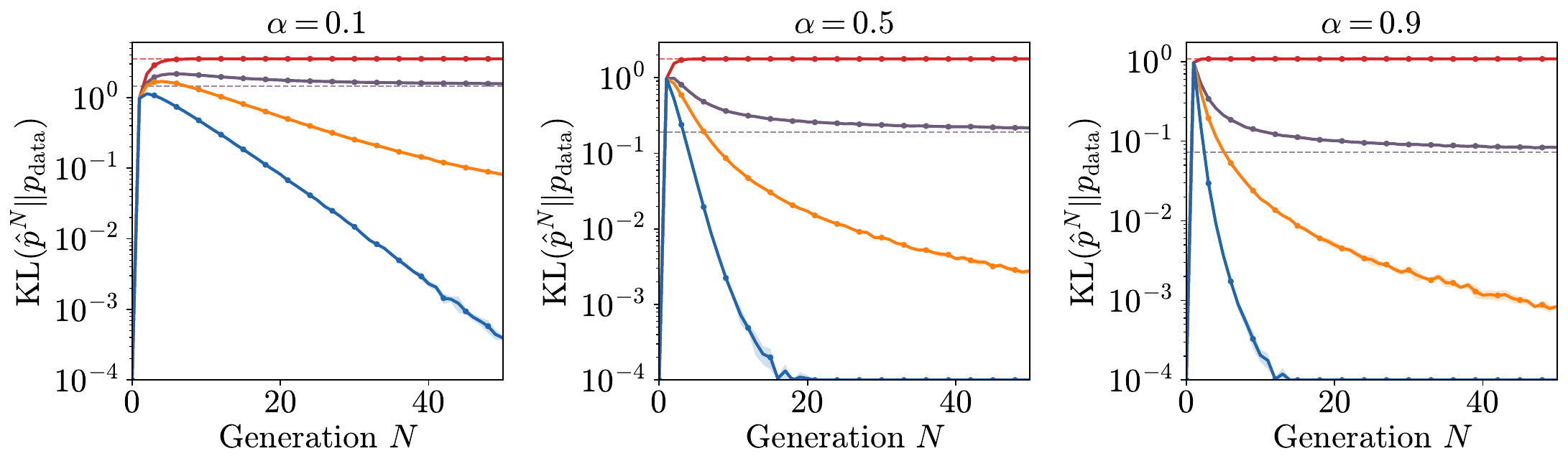}
    \caption{\textbf{Annealed truncation schedules eliminate recursive compounding (Theorem~\ref{thm:annealed-truncation-corrected}).} Error-free recursion on a 2D Gaussian mixture, $\mathrm{KL}(\phat^N \,\|\, \pdata)$ vs.\ generation $N$ for different fresh-data proportions $\alpha \in \{0.1, 0.5, 0.9\}$. Each plot compares four schedules: fixed at $t_0=0.5$ (red), decreasing and converging to $t_\infty = 0.2$ (purple), and annealed converging to $0$ given by $t_0/(1+i)^\beta$, with $\beta = 1$ (orange) and $\beta = 2$ (blue). The annealed schedules drive KL toward $0$ at rates increasing with $\beta$, while the fixed and positive-limit schedules plateau at  $\mathrm{KL}(\spinf, \pdata)$ and $\mathrm{KL}(p_\infty^{(t_\infty)}, \pdata)$, respectively. Experiment details in Appendix \ref{app:experiments}.}
    \label{fig:truncation-2d-gmm}
\end{figure}
\paragraph{Generation-dependent truncation} We consider a sequence of truncation times $(t_0^{(N)})_{N\ge0}$, so that, in the error-free regime, the recursion becomes
\begin{equation}
\label{eq:annealed-recursion}
\phat^{N+1}
=
\cU_{t_0^{(N)}}\bigl(\alpha\,\pdata + (1-\alpha)\phat^N\bigr),
\qquad N\ge 0.
\end{equation}
The following result (proof in Appendix \ref{app:proof-annealed-truncation-corrected}) gives an exact representation of the iterates $\phat^N$ obtained by the adaptive recursion \eqref{eq:annealed-recursion}, characterizes their limit when the schedule converges and shows that $\pdata$ can be exactly recovered. 

\begin{theorem}[Asymptotic effect of generation-dependent truncation]
\label{thm:annealed-truncation-corrected}
Let \((t_0^{(i)})_{i\ge 0}\) be a sequence of positive truncation times, and consider
the error-free recursion \eqref{eq:annealed-recursion}. Then, for every \(N\ge 1\),
\begin{equation}
\label{eq:annealed-unrolled}
\phat^N
=
\alpha \sum_{m=0}^{N-1} (1-\alpha)^m\,\cU_{\sigma_{m,N}}(\pdata)
+
(1-\alpha)^N \cU_{s_{0,N}}(\pdata),
\end{equation}
where $\sigma_{m,N}:=\sum_{\ell=N-1-m}^{N-1} t_0^{(\ell)}$ and $s_{0,N}:=\sum_{\ell=0}^{N-1} t_0^{(\ell)}$. Furthermore, if \(t_0^{(N)} \to \tinf \in [0,\infty)\) as $N \to \infty$,  then
$$
\phat^N \, \xrightarrow[N\to\infty]{\cW_2} \, 
p_\infty^{(\tinf)}
:=
\alpha \sum_{m=0}^{\infty} (1-\alpha)^m\,\cU_{(m+1)\tinf}\pdata.
$$
In particular, if $\tinf=0$, then \(p_\infty^{(\tinf)}=\pdata\), and the recursive compounding effect disappears asymptotically.
\end{theorem}

Theorem~\ref{thm:annealed-truncation-corrected} shows that the long-run effect of a varying truncation schedule is controlled by its \emph{asymptotic level} $\tinf$: because fresh data are re-injected at every generation, recent truncation times dominate the long-run dynamics. If the schedule converges to a positive limit, collapse persists as in the fixed-truncation regime with $t_0=\tinf$, but with a smaller bias than with fixed truncation at $t_0^{(0)}$ if the sequence $(t^{(i)}_0)_{i\ge0}$ is decreasing. By contrast, if $t_0^{(i)}\to 0$, then the error-free recursion converges back to $\pdata$; asymptotically, no nontrivial collapse distribution remains. This result holds for any sequence of positive truncation times, and in particular, allows  for adaptive truncation chosen online at each retraining round. For simplicity, our experiments use the predetermined $\beta$-annealed schedules defined below.

\paragraph{ $\beta$-annealed schedules}
For any $\beta>0$, the schedule $t_0^{(i)}(\beta)=t_0/(1+i)^\beta$ converges to $0$, implying $\cW_2(\phat^N,\pdata) \to 0$ as $N \to \infty$ in the error-free regime. However, different values of $\beta$ induce different \emph{finite-generation} distortions: larger $\beta$ drives the truncation time to zero faster, therefore suppressing recursive attenuation more aggressively over a fixed number of generations. Figure~\ref{fig:truncation-2d-gmm} compares four schedules over $50$ generations at $\alpha \in \{0.1, 0.5, 0.9\}$: fixed at $t_0=0.5$, decreasing to $t_\infty = 0.2$, and $\beta$-annealed schedules $t_0/(1+i)^\beta$ with $\beta \in \{1, 2\}$. The fixed and shifted schedules plateau at the predicted floors $\mathrm{KL}(\spinf, \pdata)$ and $\mathrm{KL}(p_\infty^{(t_\infty)}, \pdata)$, while the annealed schedules drive KL toward $0$ at rates increasing with $\beta$, confirming that recursive compounding is eliminated whenever $\tinf = 0$. Convergence slows with smaller $\alpha$, as expected from $\kappa = \sqrt{1-\alpha}\,e^{-t_0/2}$.

\paragraph{Comparing schedules at finite $N$} For a suitable non-decreasing map $h : \R_{\ge 0} \to \R_{\ge 0}$ (details in Appendix~\ref{app:proof-finite-N-schedule}), the representation \eqref{eq:annealed-unrolled}
yields, for every $N\ge1$,
$$
\cW_2^2(\phat^N,\pdata)\;\le\;\alpha\sum_{m=0}^{N-1}(1-\alpha)^m\,h(\sigma_{m,N})\;+\;(1-\alpha)^N\,h(s_{0,N}).
$$

Since $h$ is nondecreasing and $\sigma_{m,N}$, $s_{0,N}$ are nondecreasing in every $t_0^{(\ell)}$, a pointwise smaller schedule has a smaller bound at every $N$, consistent with the ordering in Figure~\ref{fig:truncation-2d-gmm}.

\paragraph{Self-regularization regime} The above finite time bound holds in the error-free regime. With errors, how fast the schedule may decay is limited, since the only data-independent control on the Fisher information of the training mixture, $J(q_{i,t})\le d/(1-e^{-t})$ (Proposition~\ref{prop:fisher-control-ou}(ii)), degenerates as $t\downarrow 0$. The recursive structure \eqref{eq:recursive-structure} provides a \emph{partial} self-regularization mechanism: while the part of the synthetic component of $q_i$ that was injected many generations ago has accumulated Gaussian smoothing $\sum_{j<i}t_0^{(j)}$ due to truncation, fresh data are re-injected at every generation, so the most recently injected components have been smoothed by only the last few truncation times, which vanish along the schedule. 

\begin{figure}[t]
    \centering
    \begin{subfigure}[t]{0.48\linewidth}
        \centering
        \includegraphics[width=\linewidth]{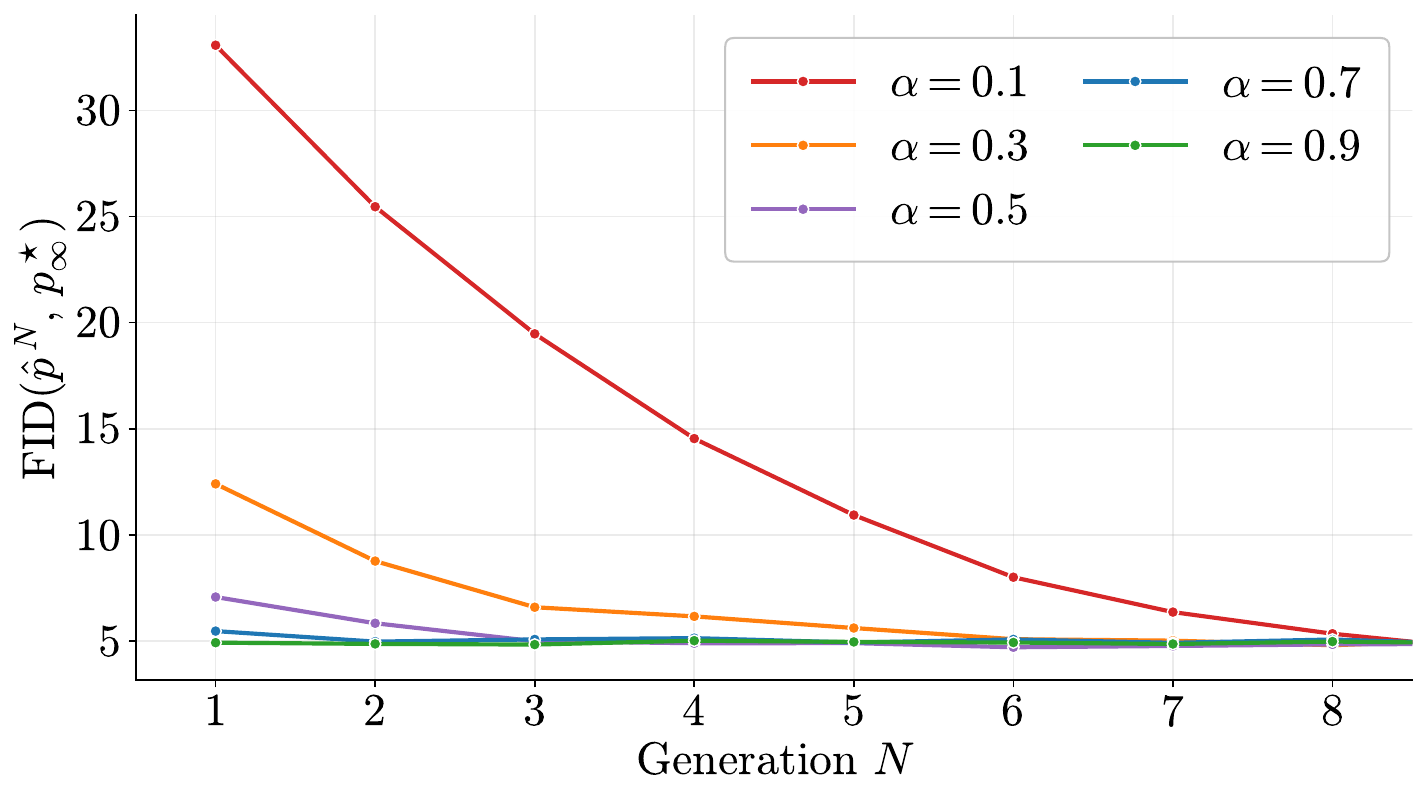}
        \centering
        \caption{$\mathrm{FID}(\phat^N, \spinf)$ vs.\ generation $N$.}
        \label{fig:cifar-fid-pinf}
    \end{subfigure}\hfill
    \begin{subfigure}[t]{0.48\linewidth}
        \centering
        \includegraphics[width=\linewidth]{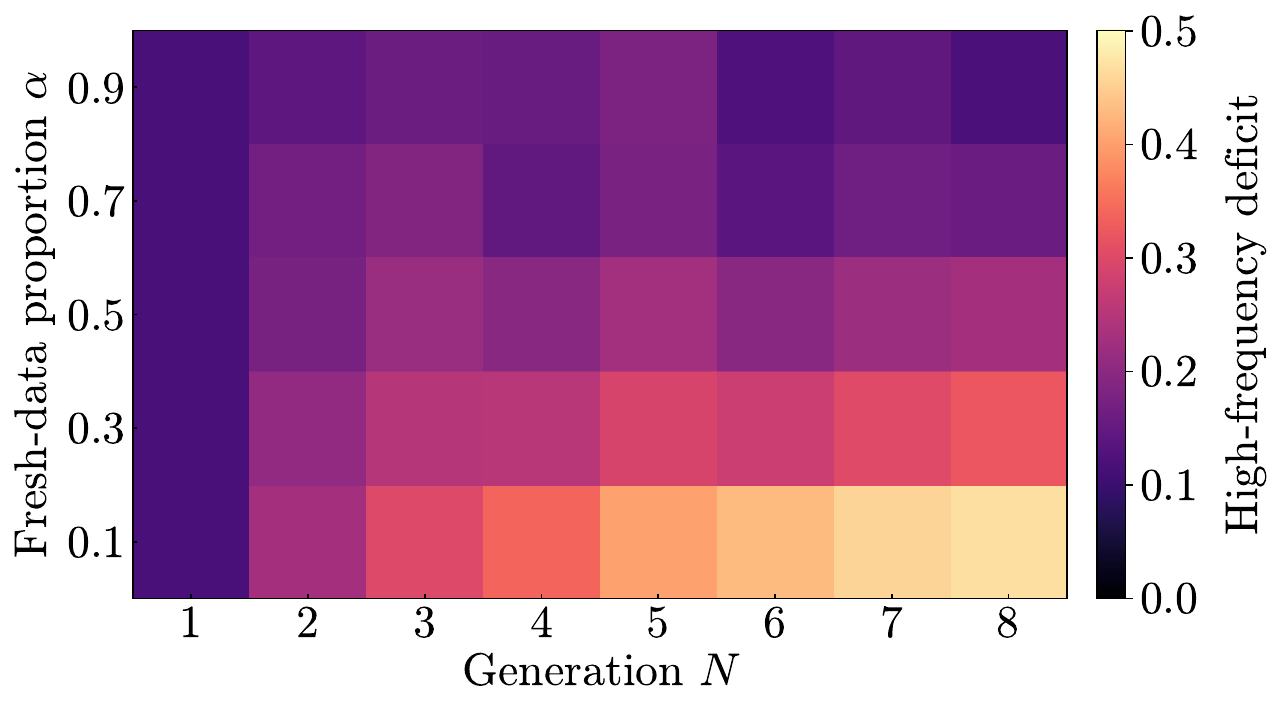}
        \caption{High-frequency energy ratio $\mathrm{HF}(\phat^N) / \mathrm{HF}(p_{\mathrm{data}})$.}
        \label{fig:cifar-hf-heatmap}
    \end{subfigure}
    \caption{\textbf{Convergence and spectral drift under recursive training (CIFAR-10).} Recursive training for $8$ generations at $\alpha \in \{0.1, 0.3, 0.5, 0.7, 0.9\}$. \textbf{(a)} The iterates $\phat^N$ contract toward $\spinf$ at a rate increasing with $\alpha$, consistent with the contraction factor $\kappa(\alpha) = \sqrt{1-\alpha}\,e^{-t_0/2}$, and settle at a similar floor consistent with stability around $\spinf$ (Theorem~\ref{thm:convergence-imperfect-regime}). 
    \textbf{(b)} High-frequency deficit $1-\mathrm{HF}(\phat^N)/\mathrm{HF}(p_{\mathrm{data}})$ over recursive generations and fresh-data proportions \(\alpha\). Larger values indicate a stronger loss of high-frequency content relative to the data distribution. The deficit grows fastest when $\alpha$ is small and $N$ increases. This loss is not caused by truncation, which adds high-frequency power in pixel space (Figure~\ref{fig:truncation-affect-collapse}, right); it reflects smoothing by the learned sampler (the one-step errors $\delta_i$), which the recursion accumulates. Details in Appendix \ref{app:experiments}.
    }
    \label{fig:cifar-validation}
\end{figure}

\section{Robustness of Limiting Distribution in the Presence of Errors} 
\label{sec:imperfect-sampling}
The idealized analysis in the last two sections isolated the effect of truncation, but in practice, recursive training is also affected by score estimation and discretization errors and by the initialization of the reverse SDE at $\cN(0, \bI_d)$ . In this section we provide robustness results quantifying how perturbations of the learned sampling operator, induced by errors in score estimation and discretization, affect the divergence between $\phat^N$ and the ideal limit $\spinf$. We return to the setting of a fixed truncation time $t_0 >0$. 

\paragraph{Score Estimation and Discretization Errors} Comparing the ideal and reverse SDEs \eqref{eq:reverse-sde-ideal} and \eqref{eq:reverse-sde-learned}, the score estimation error at generation $i$ is
$\be_i(\bx, s) = \bs_{\theta_i}(\bx, s) - \nabla_{\bx} \log q_{i,s}(\bx)$. For a given score network, $\be_i$ depends, through the target $q_i$, on the mixing factor $\alpha$, the ambient dimension $d$ and the $i$-th generation training sample size $n_i$. The learned reverse SDE~\eqref{eq:reverse-sde-learned} is further discretized via a numerical scheme on a partition $\{t_k\}_{k=1}^{K_i}$ of $[t_0, T]$. The $i$-th generation \textit{one-step sampling error} $\delta_i$ captures the combined effect of score estimation, discretization and initialization at $\cN(0, \bI_d)$ instead of $q_{i,T}$. We define it as the $\cW_2$ distance between the output of the imperfect sampler $\hat{\cS}_i$ and that of the ideal sampler $\cU_{t_0}$:
\begin{equation}\label{eq:one-step-error}
  \delta_i\equiv\delta_i(d, n_i, \{t_k\}_{k=1}^{K_i}, T, \alpha) \;=\; \cW_2\!\bigl(\hat{\cS}_i(q_i),\;\cU_{t_0}(q_i)\bigr).
\end{equation}
\begin{remark}[Concrete instantiations of $\delta_i$]\label{rem:concrete-delta}
Several works make the dependence between $\delta_i$ and $(d, n_i, \{t_k\}_{k=1}^{K_i}, \alpha)$ explicit via upper bounds. For the Euler--Maruyama discretization, Arsenyan et al. \cite{arsenyan-vardanyan-dalalyan} establish optimal $\cW_2$ bounds, and  for the exponential integrator,  Chen et al. \cite{sampling-is-as-easy-as-learning-the-score} provide TV bounds that can be transferred to $\cW_2$ under moment assumptions. Higher-order schemes such as midpoint randomization \cite{yu2025advancing}, stochastic Runge--Kutta methods \cite{wu2024stochastic}, and accelerated DDIM-type samplers \cite{li2024accelerating} have been shown to have improved dependence on the step size or on the number of function evaluations. Deterministic samplers based on the probability flow ODE \cite{gao2025convergence} admit analogous $\cW_2$ guarantees under log-concavity. Our results are general and applicable to any of these schemes. These bounds also contain an initialization term that decays exponentially in $T$.
\end{remark}

\paragraph{Concentration around $\spinf$} We now show that the imperfect iterates $\phat^N$ track the ideal collapse distribution $\spinf$ at a geometric rate, perturbed by the one-step errors $\delta_i$ (proof in Appendix \ref{app:proof-convergence-imperfect-regime}). 

\begin{theorem}[Geometric convergence  and stability]
\label{thm:convergence-imperfect-regime}
Assume \textup{\ref{A1}--\ref{A3}}. Then for each $N \ge 1$,
\begin{equation}
\label{eq:W2-recursion-main}
\cW_2(\phat^N, \spinf)
\;\le\;
\kappa^N\,\cW_2(\pdata, \spinf)
\;+\;
\sum_{i=0}^{N-1}\kappa^{N-1-i}\,\delta_i,
\end{equation}
where $\kappa = \sqrt{(1-\alpha)}e^{-t_0/2} < 1$. In particular:
\begin{enumerate}[label=(\roman*)]
    \item if $\sum_{i \ge 0} \delta_i < \infty$, then
          $\cW_2(\phat^N, \spinf) \to 0$;
    \item if $\sup_i\delta_i < \delta$, then
    $\displaystyle
    \limsup_{N\to\infty}\cW_2(\phat^N, \spinf)
    \le \frac{\delta}{1-\kappa}.
    $
\end{enumerate}
\end{theorem}

Theorem~\ref{thm:convergence-imperfect-regime} establishes that $\phat^N$ remains within a $\cW_2$-ball of radius $\delta/(1-\kappa)$ around $\spinf$, and Figure \ref{fig:cifar-validation} illustrates this robustness to error in a real image dataset (CIFAR-10) across varying levels of $\alpha$. Proposition \ref{prop:perturbation-decomposition} (Appendix \ref{app:proof-perturbation-decomp-prop}) complements Theorem \ref{thm:convergence-imperfect-regime} by providing an explicit characterization of the perturbation $(\phat^N - \spinf)$.

\paragraph{Spectral Decomposition in the Learned Regime} The spectral viewpoint on the error-free collapse distribution $\spinf$ can be extended to the general setting.

The following result (proof in Appendix \ref{app:proof-spectral-perturbation-cor}) complements Theorem~\ref{thm:convergence-imperfect-regime} with a mode-by-mode description of the error: mode $\bn$ contracts with geometric ratio $\kappa_\bn=(1-\alpha)e^{-|\bn|t_0/2}$, which decreases with $|\bn|$ (substantially so once $|\bn|t_0 \gtrsim 1$). Transient errors in high-order modes are therefore forgotten faster than errors in low-order modes, and persistent errors in mode $\bn$ are amplified by at most $1/(1-\kappa_\bn)\le 1/\alpha$, a factor that tends to $1$ as $|\bn|\to\infty$.

\begin{proposition}[Mode-dependent error propagation]
\label{prop:spectral-perturbation}
Under \ref{A1}--\ref{ass:spectral-regularity-A4}, further assume that for all $i \ge 0$, the learned distribution in the imperfect regime satisfies $\md\phat^i/\md\gamma\in L^2(\gamma)$. Then, for $N \ge 1$, $\md (\phat^N - \spinf)/\md \gamma \in L^2(\gamma)$ and, decomposing mode by mode, for each multi-index $\bn \in \N^d$,
\begin{equation}\label{eq:spectral-perturbation}
  \langle \phat^N - \spinf,\, H_\bn \rangle
  \;=\;
  \kappa_\bn^N\,\langle \pdata - \spinf,\, H_\bn\rangle
  \;+\;
  \sum_{i=0}^{N-1} \kappa_\bn^{N-1-i}\,\langle \phat^{i+1} - \cU_{t_0}(q_i),\, H_\bn\rangle,
\end{equation}
where $\kappa_\bn := (1-\alpha)\,e^{-|\bn|t_0/2}$ is the mode-dependent contraction rate. 
\end{proposition}

\section{Discussion and Future Work}
\label{sec:discussion}

This work provides an asymptotic characterization of model collapse in recursively trained diffusion models. First, in an idealized regime without discretization and score estimation errors, we show that collapse still occurs because of the truncation of the reverse diffusion near $t=0$, which is standard in practical sampling. Truncation induces a residual Gaussian smoothing in each generation, which compounds and drives the recursive dynamics to a unique limiting distribution $\spinf$, a geometric mixture of progressively smoothed copies of $\pdata$. The spectral representation of $\spinf$ further shows that this limiting distribution corresponds to a mode-wise attenuation of the data distribution in the Hermite basis, making precise the sense in which recursive diffusion training acts as a low-pass filter.

Our results also highlight that in diffusion models, model collapse is not solely driven by the proportion of fresh data $\alpha$. Increasing $\alpha$ attenuates recursive degradation, but it does not eliminate the smoothing effect induced by a fixed truncation time $t_0>0$. This observation motivates a diffusion-specific route to mitigating collapse, complementary to fresh data injection: by decreasing the truncation time across recursive generations (annealed truncation schedules), one can eliminate the asymptotic compounding effect in the error-free regime. 

\paragraph{Future directions and limitations} Several directions remain open. First, our analysis treats score estimation and discretization errors through abstract one-step perturbations. A natural next step is to combine the present fixed-point analysis with finite-sample score-learning bounds, thereby obtaining explicit end-to-end guarantees in terms of model class, sample size, dimension, and numerical solver. Second, the annealed truncation schedules studied here are prescribed, rather than learned or adapted to the current generation. Designing practical adaptive rules that balance numerical stability against recursive bias is an important direction. Third, the spectral characterization is developed for the variance-preserving OU diffusion and the Hermite basis associated with its invariant Gaussian measure. Extending this perspective to other diffusion parameterizations, alternative noising processes, or non-Gaussian reference measures may reveal more general spectral mechanisms underlying collapse.

\section*{Acknowledgements}
NBK is supported by a G-Research Trinity College Studentship, and  RET is supported by the EPSRC Probabilistic AI Hub (EP/Y028783/1).

\bibliographystyle{plain}
\bibliography{references}

@misc{gulcehre2023reinforcedselftrainingrestlanguage,
  author       = {Gulcehre, Caglar and Paine, Tom Le and Srinivasan, Srivatsan and Konyushkova, Ksenia and Weerts, Lotte and Sharma, Abhishek and Siddhant, Aditya and Ahern, Alex and Wang, Miaosen and Gu, Chenjie and Macherey, Wolfgang and Doucet, Arnaud and Firat, Orhan and de Freitas, Nando},
  title        = {Reinforced Self-Training ({ReST}) for Language Modeling},
  howpublished = {arXiv:2308.08998},
  year         = {2023}
}

@misc{alemohammad2024sims,
  author       = {Alemohammad, Sina and Humayun, Ahmed Imtiaz and Agarwal, Shruti and Collomosse, John and Baraniuk, Richard},
  title        = {Self-Improving Diffusion Models with Synthetic Data},
  howpublished = {arXiv:2408.16333},
  year         = {2024}
}

@inproceedings{zelikman2022star,
  author    = {Zelikman, Eric and Wu, Yuhuai and Mu, Jesse and Goodman, Noah},
  title     = {{STaR}: Bootstrapping Reasoning With Reasoning},
  booktitle = {Advances in Neural Information Processing Systems},
  volume    = {35},
  pages     = {15476--15488},
  year      = {2022}
}

@inproceedings{carlini2024datapoisoning,
  author    = {Carlini, Nicholas and Jagielski, Matthew and Choquette-Choo, Christopher A. and Paleka, Daniel and Pearce, Will and Anderson, Hyrum and Terzis, Andreas and Thomas, Kurt and Tram{\`e}r, Florian},
  title     = {Poisoning Web-Scale Training Datasets is Practical},
  booktitle = {IEEE Symposium on Security and Privacy (SP)},
  pages     = {407--425},
  year      = {2024}
}

@inproceedings{alemohammad2023selfconsuminggenerativemodelsmad,
  author    = {Alemohammad, Sina and Casco-Rodriguez, Josue and Luzi, Lorenzo and Humayun, Ahmed Imtiaz and Babaei, Hossein and LeJeune, Daniel and Siahkoohi, Ali and Baraniuk, Richard},
  title     = {Self-Consuming Generative Models Go {MAD}},
  booktitle = {International Conference on Learning Representations (ICLR)},
  year      = {2024}
}

@article{mc_shumaylov,
  author  = {Shumailov, Ilia and Shumaylov, Zakhar and Zhao, Yiren and Papernot, Nicolas and Anderson, Ross and Gal, Yarin},
  title   = {{AI} Models Collapse When Trained on Recursively Generated Data},
  journal = {Nature},
  volume  = {631},
  number  = {8022},
  pages   = {755--759},
  year    = {2024}
}

@inproceedings{dohmatob2024tale,
  author    = {Dohmatob, Elvis and Feng, Yunzhen and Yang, Pu and Charton, Fran{\c{c}}ois and Kempe, Julia},
  title     = {A Tale of Tails: Model Collapse as a Change of Scaling Laws},
  booktitle = {Proceedings of the 41st International Conference on Machine Learning},
  year      = {2024}
}

@inproceedings{shi2026a,
  author    = {Shi, Lianghe and Wu, Meng and Zhang, Huijie and Zhang, Zekai and Tao, Molei and Qu, Qing},
  title     = {A Closer Look at Model Collapse: From a Generalization-to-Memorization Perspective},
  booktitle = {Advances in Neural Information Processing Systems},
  volume    = {38},
  year      = {2025}
}

@inproceedings{dohmatob2024strong,
  author    = {Dohmatob, Elvis and Feng, Yunzhen and Subramonian, Arjun and Kempe, Julia},
  title     = {Strong Model Collapse},
  booktitle = {International Conference on Learning Representations (ICLR)},
  year      = {2025}
}

@misc{model-collapse-inevitable,
  author       = {Gerstgrasser, Matthias and Schaeffer, Rylan and Dey, Apratim and Rafailov, Rafael and Sleight, Henry and Hughes, John and Korbak, Tomasz and Agrawal, Rajashree and Pai, Dhruv and Gromov, Andrey and Roberts, Daniel A. and Yang, Diyi and Donoho, David L. and Koyejo, Sanmi},
  title        = {Is Model Collapse Inevitable? {B}reaking the Curse of Recursion by Accumulating Real and Synthetic Data},
  howpublished = {arXiv:2404.01413},
  year         = {2024}
}

@inproceedings{model-collapse-demystified,
  author    = {Dohmatob, Elvis and Feng, Yunzhen and Kempe, Julia},
  title     = {Model Collapse Demystified: The Case of Regression},
  booktitle = {Advances in Neural Information Processing Systems},
  volume    = {37},
  year      = {2024}
}

@misc{dey2025universality,
  author       = {Dey, Apratim and Donoho, David},
  title        = {Universality of the $\pi^2/6$ Pathway in Avoiding Model Collapse},
  howpublished = {arXiv preprint arXiv:2410.22812},
  year         = {2024}
}

@inproceedings{analyzing-mitigating-model-collapse,
  author    = {Zhu, Huminhao and Wang, Fangyikang and Ding, Tianyu and Qu, Qing and Zhu, Zhihui},
  title     = {Analyzing and Mitigating Model Collapse in Reflow Methods},
  booktitle = {Conference on Parsimony and Learning (CPAL)},
  year      = {2026}
}

@inproceedings{garg2026preventing,
  author    = {Garg, Anvit and Bhattacharya, Sohom and Sur, Pragya},
  title     = {Preventing Model Collapse Under Overparametrization: Optimal Mixing Ratios for Interpolation Learning and Ridge Regression},
  booktitle = {International Conference on Learning Representations (ICLR)},
  year      = {2026}
}

@inproceedings{Barzilai26when-models-dont,
  author    = {Barzilai, Daniel and Shamir, Ohad},
  title     = {When Models Don't Collapse: On the Consistency of Iterative {MLE}},
  booktitle = {Advances in Neural Information Processing Systems},
  volume    = {38},
  year      = {2025}
}

@inproceedings{jain2024scaling,
  author    = {Jain, Ayush and Montanari, Andrea and Sasoglu, Eren},
  title     = {Scaling Laws for Learning with Real and Surrogate Data},
  booktitle = {Advances in Neural Information Processing Systems},
  volume    = {37},
  pages     = {110246--110289},
  year      = {2024}
}

@misc{curse-schumailov,
  author       = {Shumailov, Ilia and Shumaylov, Zakhar and Zhao, Yiren and Gal, Yarin and Papernot, Nicolas and Anderson, Ross},
  title        = {The Curse of Recursion: Training on Generated Data Makes Models Forget},
  howpublished = {arXiv:2305.17493},
  year         = {2023}
}

@inproceedings{bertrand2024stability,
  author    = {Bertrand, Quentin and Bose, Avishek Joey and Duplessis, Alexandre and Jiralerspong, Marco and Gidel, Gauthier},
  title     = {On the Stability of Iterative Retraining of Generative Models on their own Data},
  booktitle = {International Conference on Learning Representations (ICLR)},
  year      = {2024}
}

@inproceedings{mc-diffusion-sample-level,
  author    = {Fu, Shi and Zhang, Sen and Wang, Yingjie and Tian, Xinmei and Tao, Dacheng},
  title     = {Towards Theoretical Understandings of Self-Consuming Generative Models},
  booktitle = {Proceedings of the 41st International Conference on Machine Learning},
  year      = {2024}
}

@inproceedings{mc-diffusion-solvable-model,
  author    = {Cui, Hugo and Pehlevan, Cengiz and Lu, Yue M.},
  title     = {A Solvable Model of Learning Generative Diffusion: Theory and Insights},
  booktitle = {Advances in Neural Information Processing Systems},
  volume    = {38},
  year      = {2025}
}

@inproceedings{khelifa2026errorpropagationmodelcollapse,
  author    = {Khelifa, Na{\"\i}l B. and Turner, Richard E. and Venkataramanan, Ramji},
  title     = {Quantifying Error Propagation and Model Collapse in Diffusion Models},
  booktitle = {Proceedings of the 43rd International Conference on Machine Learning (ICML)},
  year      = {2026}
}

@inproceedings{seminal-diffusion-models,
  author    = {Sohl-Dickstein, Jascha and Weiss, Eric and Maheswaranathan, Niru and Ganguli, Surya},
  title     = {Deep Unsupervised Learning using Nonequilibrium Thermodynamics},
  booktitle = {International Conference on Machine Learning},
  year      = {2015}
}

@inproceedings{song2019generative,
  author    = {Song, Yang and Ermon, Stefano},
  title     = {Generative Modeling by Estimating Gradients of the Data Distribution},
  booktitle = {Advances in Neural Information Processing Systems},
  volume    = {32},
  year      = {2019}
}

@inproceedings{song2021scorebased,
  author    = {Song, Yang and Sohl-Dickstein, Jascha and Kingma, Diederik P. and Kumar, Abhishek and Ermon, Stefano and Poole, Ben},
  title     = {Score-Based Generative Modeling through Stochastic Differential Equations},
  booktitle = {International Conference on Learning Representations (ICLR)},
  year      = {2021}
}

@inproceedings{DDPM,
  author    = {Ho, Jonathan and Jain, Ajay and Abbeel, Pieter},
  title     = {Denoising Diffusion Probabilistic Models},
  booktitle = {Advances in Neural Information Processing Systems},
  volume    = {33},
  year      = {2020}
}

@inproceedings{song2020improved_ncsn,
  author    = {Song, Yang and Ermon, Stefano},
  title     = {Improved Techniques for Training Score-Based Generative Models},
  booktitle = {Advances in Neural Information Processing Systems},
  volume    = {33},
  year      = {2020}
}

@inproceedings{kim2022soft,
  author    = {Kim, Dongjun and Shin, Seungjae and Song, Kyungwoo and Kang, Wanmo and Moon, Il-Chul},
  title     = {Soft Truncation: A Universal Training Technique of Score-based Diffusion Model for High Precision Score Estimation},
  booktitle = {In International Conference on Machine Learning},
  year      = {2022}
}

@article{Efron01122011,
  author  = {Efron, Bradley},
  title   = {{T}weedie's Formula and Selection Bias},
  journal = {Journal of the American Statistical Association},
  volume  = {106},
  number  = {496},
  pages   = {1602--1614},
  year    = {2011}
}

@inproceedings{karras2022elucidating,
  author    = {Karras, Tero and Aittala, Miika and Aila, Timo and Laine, Samuli},
  title     = {Elucidating the Design Space of Diffusion-Based Generative Models},
  booktitle = {Advances in Neural Information Processing Systems},
  volume    = {35},
  pages     = {26565--26577},
  year      = {2022}
}

@inproceedings{zhang2024tackling,
  author    = {Zhang, Pengze and Yin, Hubery and Li, Chen and Xie, Xiaohua},
  title     = {Tackling the Singularities at the Endpoints of Time Intervals in Diffusion Models},
  booktitle = {Proceedings of the IEEE/CVF Conference on Computer Vision and Pattern Recognition (CVPR)},
  pages     = {6945--6954},
  year      = {2024}
}

@inproceedings{max-likelihood-diffusion-models,
  author    = {Song, Yang and Durkan, Conor and Murray, Iain and Ermon, Stefano},
  title     = {Maximum Likelihood Training of Score-Based Diffusion Models},
  booktitle = {Advances in Neural Information Processing Systems},
  volume    = {34},
  year      = {2021}
}

@article{DPM-solver++,
  author  = {Lu, Cheng and Zhou, Yuhao and Bao, Fan and Chen, Jianfei and Li, Chongxuan and Zhu, Jun},
  title   = {{DPM-Solver++}: Fast Solver for Guided Sampling of Diffusion Probabilistic Models},
  journal = {Machine Intelligence Research},
  volume  = {22},
  number  = {4},
  pages   = {730--751},
  year    = {2025}
}

@inproceedings{sampling-is-as-easy-as-learning-the-score,
  author    = {Chen, Sitan and Chewi, Sinho and Li, Jerry and Li, Yuanzhi and Salim, Adil and Zhang, Anru R.},
  title     = {Sampling is as Easy as Learning the Score: Theory for Diffusion Models with Minimal Data Assumptions},
  booktitle = {International Conference on Learning Representations (ICLR)},
  year      = {2023}
}

@inproceedings{benton2024nearly,
  author    = {Benton, Joe and De Bortoli, Valentin and Doucet, Arnaud and Deligiannidis, George},
  title     = {Nearly $d$-Linear Convergence Bounds for Diffusion Models via Stochastic Localization},
  booktitle = {International Conference on Learning Representations (ICLR)},
  year      = {2024}
}

@article{convergence-manifold-hypotheses-vdb,
  author  = {De Bortoli, Valentin},
  title   = {Convergence of Denoising Diffusion Models under the Manifold Hypothesis},
  journal = {Transactions on Machine Learning Research},
  year    = {2022}
}

@inproceedings{chen-2023,
  author    = {Chen, Hongrui and Lee, Holden and Lu, Jianfeng},
  title     = {Improved Analysis of Score-based Generative Modeling: User-Friendly Bounds under Minimal Smoothness Assumptions},
  booktitle = {Proceedings of the 40th International Conference on Machine Learning},
  year      = {2023}
}

@inproceedings{minimax-optim-score-based-diff-models,
  author    = {Zhang, Kaihong and Yin, Heqi and Liang, Feng and Liu, Jingbo},
  title     = {Minimax Optimality of Score-based Diffusion Models: Beyond the Density Lower Bound Assumptions},
  booktitle = {In International Conference on Machine Learning},
  year      = {2024}
}

@article{wasserstein,
  author  = {Gao, Xuefeng and Nguyen, Hoang M. and Zhu, Lingjiong},
  title   = {{W}asserstein Convergence Guarantees for a General Class of Score-Based Generative Models},
  journal = {Journal of Machine Learning Research},
  volume  = {26},
  number  = {43},
  pages   = {1--54},
  year    = {2025}
}

@article{confortiKLguarantees,
  author  = {Conforti, Giovanni and Durmus, Alain and {Gentiloni Silveri}, Marta},
  title   = {{KL} Convergence Guarantees for Score Diffusion Models under Minimal Data Assumptions},
  journal = {SIAM Journal on Mathematics of Data Science},
  volume  = {7},
  number  = {1},
  pages   = {86--109},
  year    = {2025}
}

@misc{samworth-shape-constraint,
  author       = {Lewis, Rebecca M. and Feng, Oliver Y. and Reeve, Henry W. J. and Xu, Min and Samworth, Richard J.},
  title        = {Learning the Score under Shape Constraints},
  howpublished = {arXiv:2512.14624},
  year         = {2025}
}

@inproceedings{score-approx-chen-low-dim,
  author    = {Chen, Minshuo and Huang, Kaixuan and Zhao, Tuo and Wang, Mengdi},
  title     = {Score Approximation, Estimation and Distribution Recovery of Diffusion Models on Low-Dimensional Data},
  booktitle = {Proceedings of the 40th International Conference on Machine Learning},
  year      = {2023}
}

@inproceedings{arsenyan-vardanyan-dalalyan,
  author    = {Arsenyan, Vahan and Vardanyan, Elen and Dalalyan, Arnak S.},
  title     = {Assessing the Quality of Denoising Diffusion Models in {W}asserstein Distance: Noisy Score and Optimal Bounds},
  booktitle = {Advances in Neural Information Processing Systems},
  volume    = {38},
  pages     = {19548--19591},
  year      = {2025}
}

@inproceedings{yu2025advancing,
  author    = {Yu, Yifeng and Yu, Lu},
  title     = {Advancing {W}asserstein Convergence Analysis of Score-Based Models: Insights from Discretization and Second-Order Acceleration},
  booktitle = {Advances in Neural Information Processing Systems},
  year      = {2025}
}

@misc{wu2024stochastic,
  author       = {Wu, Yuchen and Chen, Yuxin and Wei, Yuting},
  title        = {Stochastic {R}unge--{K}utta Methods: Provable Acceleration of Diffusion Models},
  howpublished = {arXiv:2410.04760},
  year         = {2024}
}

@inproceedings{li2024accelerating,
  author    = {Li, Gen and Huang, Yu and Efimov, Timofey and Wei, Yuting and Chi, Yuejie and Chen, Yuxin},
  title     = {Accelerating Convergence of Score-Based Diffusion Models, Provably},
  booktitle = {Proceedings of the 41st International Conference on Machine Learning},
  series    = {Proceedings of Machine Learning Research},
  volume    = {235},
  pages     = {27942--27954},
  publisher = {PMLR},
  year      = {2024}
}

@inproceedings{gao2025convergence,
  author    = {Gao, Xuefeng and Zhu, Lingjiong},
  title     = {Convergence Analysis for General Probability Flow {ODE}s of Diffusion Models in {W}asserstein Distances},
  booktitle = {Proceedings of the 28th International Conference on Artificial Intelligence and Statistics},
  year      = {2025}
}

@book{sto-calculus-2,
  author    = {Karatzas, Ioannis and Shreve, Steven E.},
  title     = {Brownian Motion and Stochastic Calculus},
  edition   = {Second},
  series    = {Graduate Texts in Mathematics},
  volume    = {113},
  publisher = {Springer},
  address   = {New York},
  year      = {1991}
}

@book{ethier-kurtz,
  author    = {Ethier, Stewart N. and Kurtz, Thomas G.},
  title     = {Markov Processes: Characterization and Convergence},
  series    = {Wiley Series in Probability and Mathematical Statistics},
  publisher = {John Wiley \& Sons},
  address   = {New York},
  year      = {1986}
}

@book{oskendal-sde,
  author    = {{\O}ksendal, Bernt},
  title     = {Stochastic Differential Equations: An Introduction with Applications},
  edition   = {Sixth},
  series    = {Universitext},
  publisher = {Springer},
  address   = {Berlin},
  year      = {2003}
}

@book{Bakry2014Analysis,
  author    = {Bakry, Dominique and Gentil, Ivan and Ledoux, Michel},
  title     = {Analysis and Geometry of Markov Diffusion Operators},
  series    = {Grundlehren der mathematischen Wissenschaften},
  volume    = {348},
  publisher = {Springer},
  address   = {Cham},
  year      = {2014}
}

@book{da2014-infinite-sde,
  author    = {Da Prato, Giuseppe and Zabczyk, Jerzy},
  title     = {Stochastic Equations in Infinite Dimensions},
  edition   = {Second},
  series    = {Encyclopedia of Mathematics and its Applications},
  volume    = {152},
  publisher = {Cambridge University Press},
  address   = {Cambridge},
  year      = {2014}
}

@article{anderson1982reverse,
  author  = {Anderson, Brian D. O.},
  title   = {Reverse-Time Diffusion Equation Models},
  journal = {Stochastic Processes and their Applications},
  volume  = {12},
  number  = {3},
  pages   = {313--326},
  year    = {1982}
}

@article{time-reversal-of-diffusions,
  author  = {Haussmann, U. G. and Pardoux, E.},
  title   = {Time Reversal of Diffusions},
  journal = {The Annals of Probability},
  volume  = {14},
  number  = {4},
  pages   = {1188--1205},
  year    = {1986}
}

@article{aronson1967,
  author  = {Aronson, Donald G.},
  title   = {Bounds for the Fundamental Solution of a Parabolic Equation},
  journal = {Bulletin of the American Mathematical Society},
  volume  = {73},
  number  = {6},
  pages   = {890--896},
  year    = {1967}
}

@article{aronson1968,
  author  = {Aronson, Donald G.},
  title   = {Non-Negative Solutions of Linear Parabolic Equations},
  journal = {Annali della Scuola Normale Superiore di Pisa, Classe di Scienze (3)},
  volume  = {22},
  number  = {4},
  pages   = {607--694},
  year    = {1968}
}

@book{bogachev2015gaussian-measures,
  author    = {Bogachev, Vladimir I.},
  title     = {Gaussian Measures},
  series    = {Mathematical Surveys and Monographs},
  volume    = {62},
  publisher = {American Mathematical Society},
  address   = {Providence, RI},
  year      = {1998}
}

@book{janson1997gaussian-hilbert-spaces,
  author    = {Janson, Svante},
  title     = {Gaussian Hilbert Spaces},
  series    = {Cambridge Tracts in Mathematics},
  volume    = {129},
  publisher = {Cambridge University Press},
  address   = {Cambridge},
  year      = {1997}
}

@book{ReedSimonIV,
  author    = {Reed, Michael and Simon, Barry},
  title     = {Methods of Modern Mathematical Physics. {IV}: Analysis of Operators},
  publisher = {Academic Press},
  address   = {New York},
  year      = {1978}
}

@article{Banach1922,
  author  = {Banach, Stefan},
  title   = {Sur les op{\'e}rations dans les ensembles abstraits et leur application aux {\'e}quations int{\'e}grales},
  journal = {Fundamenta Mathematicae},
  volume  = {3},
  number  = {1},
  pages   = {133--181},
  year    = {1922}
}

@book{santambrogio-ot,
  author    = {Santambrogio, Filippo},
  title     = {Optimal Transport for Applied Mathematicians: Calculus of Variations, PDEs, and Modeling},
  series    = {Progress in Nonlinear Differential Equations and Their Applications},
  volume    = {87},
  publisher = {Birkh{\"a}user},
  address   = {Cham},
  year      = {2015}
}

@article{cuturi-peyre,
  author  = {Peyr{\'e}, Gabriel and Cuturi, Marco},
  title   = {Computational Optimal Transport: With Applications to Data Science},
  journal = {Foundations and Trends in Machine Learning},
  volume  = {11},
  number  = {5--6},
  pages   = {355--607},
  year    = {2019}
}

@book{villani2008optimal,
  author    = {Villani, C{\'e}dric},
  title     = {Optimal Transport: Old and New},
  series    = {Grundlehren der mathematischen Wissenschaften},
  volume    = {338},
  publisher = {Springer},
  address   = {Berlin},
  year      = {2009}
}

@article{hyvarinen2005estimation,
  author  = {Hyv{\"a}rinen, Aapo},
  title   = {Estimation of Non-Normalized Statistical Models by Score Matching},
  journal = {Journal of Machine Learning Research},
  volume  = {6},
  pages   = {695--709},
  year    = {2005}
}

@article{hyvarinen2007_extension,
  author  = {Hyv{\"a}rinen, Aapo},
  title   = {Some Extensions of Score Matching},
  journal = {Computational Statistics \& Data Analysis},
  volume  = {51},
  number  = {5},
  pages   = {2499--2512},
  year    = {2007}
}

@article{denoising-score-matching-vincent,
  author  = {Vincent, Pascal},
  title   = {A Connection Between Score Matching and Denoising Autoencoders},
  journal = {Neural Computation},
  volume  = {23},
  number  = {7},
  pages   = {1661--1674},
  year    = {2011}
}

@article{STAM1959101,
  author  = {Stam, A. J.},
  title   = {Some Inequalities Satisfied by the Quantities of Information of {F}isher and {S}hannon},
  journal = {Information and Control},
  volume  = {2},
  number  = {2},
  pages   = {101--112},
  year    = {1959}
}

@article{johnson-barron-fisher-information-inequalities,
  author  = {Johnson, Oliver and Barron, Andrew},
  title   = {{F}isher Information Inequalities and the Central Limit Theorem},
  journal = {Probability Theory and Related Fields},
  volume  = {129},
  number  = {3},
  pages   = {391--409},
  year    = {2004}
}

@article{information-theoretic-inequalities,
  author  = {Dembo, Amir and Cover, Thomas M. and Thomas, Joy A.},
  title   = {Information Theoretic Inequalities},
  journal = {IEEE Transactions on Information Theory},
  volume  = {37},
  number  = {6},
  pages   = {1501--1518},
  year    = {1991}
}

@techreport{cifar10,
  author      = {Krizhevsky, Alex},
  title       = {Learning Multiple Layers of Features from Tiny Images},
  institution = {University of Toronto},
  year        = {2009}
}

@inproceedings{unet-seminal,
  author    = {Ronneberger, Olaf and Fischer, Philipp and Brox, Thomas},
  title     = {{U-Net}: Convolutional Networks for Biomedical Image Segmentation},
  booktitle = {Medical Image Computing and Computer-Assisted Intervention (MICCAI)},
  volume    = {9351},
  pages     = {234--241},
  year      = {2015}
}

@inproceedings{vaswani2017attention,
  author    = {Vaswani, Ashish and Shazeer, Noam and Parmar, Niki and Uszkoreit, Jakob and Jones, Llion and Gomez, Aidan N. and Kaiser, {\L}ukasz and Polosukhin, Illia},
  title     = {Attention is All you Need},
  booktitle = {Advances in Neural Information Processing Systems},
  volume    = {30},
  year      = {2017}
}

@inproceedings{minSNR,
  author    = {Hang, Tiankai and Gu, Shuyang and Li, Chen and Bao, Jianmin and Chen, Dong and Hu, Han and Geng, Xin and Guo, Baining},
  title     = {Efficient Diffusion Training via {Min-SNR} Weighting Strategy},
  booktitle = {Proceedings of the IEEE/CVF International Conference on Computer Vision (ICCV)},
  pages     = {7441--7451},
  year      = {2023}
}

@inproceedings{adamw,
  author    = {Loshchilov, Ilya and Hutter, Frank},
  title     = {Decoupled Weight Decay Regularization},
  booktitle = {International Conference on Learning Representations (ICLR)},
  year      = {2019}
}

@inproceedings{Heuseletal2017,
  author    = {Heusel, Martin and Ramsauer, Hubert and Unterthiner, Thomas and Nessler, Bernhard and Hochreiter, Sepp},
  title     = {{GANs} Trained by a Two Time-Scale Update Rule Converge to a Local {N}ash Equilibrium},
  booktitle = {Advances in Neural Information Processing Systems},
  volume    = {30},
  year      = {2017}
}

\newpage

\appendix 

\section{Notation}
In the rest of this appendix, we use the following notation: 
\begin{itemize}
    \item We denote by $\gamma = \cN(0, \bI_d)$ the standard centered and unit variance Gaussian measure on $\R^d$. Explicitly, for any $\bx \in \R^d$,  
    $$
    \gamma(\bx)
    =
    (2\pi)^{-d/2} \exp\!\left(-\tfrac12 \|\bx\|^2\right).
    $$
    \item The space $L^2(\gamma)$ is defined as,
    $$
    L^2(\gamma) 
    := 
    \left\{
    f:\R^d \to \R \;\middle|\;
    \int_{\R^d} \|f(\bx)\|^2_2 \,\gamma(\md\bx) < \infty
    \right\},
    $$ 
    where functions are identified up to equality $\gamma$-almost everywhere.
    \item We equip $L^2(\gamma)$ with an inner product defined as, 
    \begin{equation}
    \label{eq:L2gamma-inner}
    \langle f, g \rangle_{L^2(\gamma)}
    :=
    \int_{\R^d} f(\bx)\,g(\bx)\,\gamma(\md\bx),
    \end{equation}
    making it a Hilbert space $\big(L^2(\gamma), \langle \cdot, \cdot \rangle_{L^2(\gamma)}\big)$. The $L^2(\gamma)$ inner product induces the following $L^2(\gamma)$-norm, 
    \begin{equation}
    \|f\|_{L^2(\gamma)}
    =
    \left(
    \int_{\R^d} \|f(\bx)\|^2_2 \,\gamma(\md\bx)
    \right)^{1/2}.
    \end{equation}
    \item For any sequence of distributions $(\mu_t)_{t \in \R_+} \in (\Probspace_2(\R^d))^{\R_+}$ and any fixed $\mu^\star \in \Probspace_2(\R^d)$, we say that $(\mu_t)_t$ converges to $\mu^\star$ in $\cW_2$ and denote $\mu_t \underset{t \to \infty}{\overset{\cW_2}{\longrightarrow}}\mu^\star$ if, 
    $$
    \cW_2(\mu_t, \mu^\star) \underset{t \to \infty}{\longrightarrow} 0.
    $$
\end{itemize}

\section{Extended Background and Preliminary Results}
\label{app:extended-background}

In this appendix, we provide a detailed  background of the mathematical tools and the classical results underlying our proofs.

\subsection{Diffusion, Ornstein--Uhlenbeck Process and Semigroup}

Recall that, for each generation $i$, the forward process $(\bX_t^i)_{t \in [0, T]}$ is defined as a solution of the following variance-preserving Ornstein-Uhlenbeck diffusion initialized at the mixture distribution $q_i$:
\begin{equation}
\label{eq:OU-forward-appendix}
\md \bX_t^i
=
-\tfrac12 \bX_t^i\,\md t + \md \bB_t, 
\qquad 
\bX_{0}^i \sim q_i.
\end{equation}
Then, integrating \eqref{eq:OU-forward-appendix} yields
$$
\bX^i_t = e^{-t/2}\bX^i_0 + \sqrt{1-e^{-t}}\bZ,
\qquad \bZ\sim\mathcal N(0,\bI_d).
$$
The corresponding transition kernel can thus be written as,
\begin{equation}
\label{eq:transition-kernel-OU}
\cK_t(\bx,\md\by)
=
\cN\!\left(e^{-t/2}\bx,(1-e^{-t})\bI_d\right)(\md\by).
\end{equation}

\paragraph{Markov semigroup.}
The transition kernel $\cK_t$ defines a Markov semigroup $(P_t)_{t\ge0}$ acting on test functions $\varphi:\R^d \to \R$ by
\begin{equation}
\label{eq:Pt-def}
(P_t \varphi)(\bx)
:=
\int_{\R^d} \varphi(\by)\,\cK_t(\bx,\md\by)
=
\E\big[\varphi(\bX_t^i)\mid \bX_0^i=\bx\big].
\end{equation}
The family $(P_t)_{t\ge0}$ satisfies the semigroup property $P_{t+s}=P_t P_s$ and $P_0=\mathrm{Id}$ \cite{ethier-kurtz}.

\paragraph{Sampling operator on probability measures.}
The adjoint of $P_t$ acts on test functions. On the other hand, the OU diffusion \eqref{eq:OU-forward-appendix} defines sampling operators $(\cU_t)_{t \in [t_0, T]}$, acting on probability measures $\mu \in \Probspace_2(\R^d)$ and defined by,
\begin{equation}
\label{eq:Ut-def}
\cU_t \mu (A)
:=
\int_{\R^d} \cK_t(\bx,A)\,\mu(\md\bx),
\qquad A \subset \R^d.
\end{equation}
Equivalently,
\begin{equation}
\label{eq:Ut-explicit}
\cU_t\mu
=
\Law\!\big(e^{-t/2}\bX_0 + \sqrt{1-e^{-t}}\,\bZ\big),
\qquad \bX_0 \sim \mu,\;\bZ \sim \cN(0,\bI_d),\;\bX_0 \indep \bZ.
\end{equation}

The operator $\cU_t$ is the sampling operator used throughout the paper. Given its central role in this work, we collect its properties in the following proposition. 

\begin{proposition}[Basic properties of the OU sampling operator]
\label{prop:OU-sampling-operator-properties}
The Ornstein--Uhlenbeck sampling operator $(\cU_t)_{t\ge 0}$ as defined in \eqref{eq:Ut-def},
satisfies the following  properties.
\begin{enumerate}[label=(\roman*)]
    \item \textbf{Linearity.}
    For any two probability measures $\mu,\nu$ and scalars $a,b\in\R$,
    $$
    \cU_t(a\mu+b\nu)
    =
    a\,\cU_t\mu+b\,\cU_t\nu.
    $$

    \item \textbf{Semigroup property.}
    For all $s,t\ge 0$,
    $$
    \cU_{t+s}=\cU_t\cU_s,
    \qquad
    \cU_0=\Id.
    $$

    \item \textbf{Weak duality identity.}
    For any bounded measurable test function $\varphi$ and every finite measure $\mu$,
    $$
    \int_{\R^d}\varphi(\by)\,(\cU_t\mu)(\md\by)
    =
    \int_{\R^d}(P_t\varphi)(\bx)\,\mu(\md\bx),
    $$

    \item \textbf{Moment preservation.}
    Let $k\ge1$, $\bZ\sim\cN(0,\bI_d)$ and $\mu\in\Probspace(\R^d)$ with
    $M_k(\mu):=\int_{\R^d}\|\bx\|^k\,\mu(\md\bx)<\infty$. Then, for every $t\ge0$,
    $$
    \Big(\int_{\R^d}\|\bx\|^k\,\cU_t\mu(\md\bx)\Big)^{1/k}
    \;\le\;
    M_k(\mu)^{1/k}+\bigl(\E\|\bZ\|^k\bigr)^{1/k},
    $$
    so $\sup_{t\ge0}M_k(\cU_t\mu)<\infty$. In particular $\cU_t$ maps $\Probspace_2(\R^d)$ into itself.

    \item \textbf{Uniqueness and Invariance of the Gaussian measure.}
    Let $\gamma=\cN(0,\bI_d)$. Then $\gamma$ is the unique measure satisfying $\cU_t\gamma=\gamma$.

    \item \textbf{Action in $L^2(\gamma)$.} For any $\mu \in \Probspace_2(\R^d)$, writing $f = \md \mu/\md \gamma$,
    $$
    \frac{\md(\cU_t\mu)}{\md\gamma}
    =
    P_t f.
    $$
    \item \textbf{Continuity in time.} For a fixed $\mu \in \Probspace_2(\R^d)$ and a fixed $s \in [0, t]$, one has that, 
    $$
    \cU_t\mu \underset{t \to s}{\overset{\cW_2}{\longrightarrow}} \cU_s\mu.
    $$
    In particular, 
    $$
    \cU_t\mu \underset{t \to 0}{\overset{\cW_2}{\longrightarrow}} \mu.
    $$
\end{enumerate}
\end{proposition}

\begin{proof}

\medskip
\emph{(i) Linearity.} Let $\mu,\nu \in \Probspace(\R^d)$ be two probability distributions on $\R^d$ and $a,b\in\R$ two scalars. Let $A\subseteq \R^d$ be any measurable set in $\R^d$. Then using \eqref{eq:Ut-def},
$$
\cU_t(a\mu+b\nu)(A)
=
\int_{\R^d}\cK_t(\bx,A)\,(a\mu+b\nu)(\md\bx).
$$
Looking at \eqref{eq:transition-kernel-OU}, and by linearity of the integral in the measure,
$$
\cU_t(a\mu+b\nu)(A)
=
a\int_{\R^d}\cK_t(\bx,A)\,\mu(\md\bx)
+
b\int_{\R^d}\cK_t(\bx,A)\,\nu(\md\bx)
=
a\,\cU_t\mu(A) +b\,\cU_t\nu(A).
$$

\medskip
\emph{(ii) Semigroup property.}
Let $\mu \in \Probspace(\R^d)$ and consider $\bX \sim \mu$. In addition, fix two independently distributed Gaussian random variables $Z_1,Z_2\sim\cN(0,\bI_d)$ and fix $s, t \in [0, T]$. Applying successively $\cU_s$ and $\cU_t$,
\begin{align*}
\cU_t \cU_s\mu 
&=
e^{-t/2}
\left(e^{-s/2}\bX+\sqrt{1-e^{-s}}\,\bZ_1\right)
+
\sqrt{1-e^{-t}}\,\bZ_2\\
&=
e^{-(t+s)/2}\bX
+
\left(e^{-t/2}\sqrt{1-e^{-s}}\,\bZ_1
+
\sqrt{1-e^{-t}}\,\bZ_2\right).
\end{align*}
Since $Z_1,Z_2$ are independent standard Gaussians, the variable $ \left(e^{-t/2}\sqrt{1-e^{-s}}\,\bZ_1 + \sqrt{1-e^{-t}}\,\bZ_2\right)$ is a centered Gaussian with covariance
$$
e^{-t}(1-e^{-s})\bI_d+(1-e^{-t})\bI_d
=
(1-e^{-(t+s)})\bI_d.
$$
Denoting $\bZ \sim \cN(0, \bI_d)$, the resulting law is
$$
\cU_t \cU_s\mu = \Law\!\left(
e^{-(t+s)/2}\bX+\sqrt{1-e^{-(t+s)}}\,\bZ
\right)
=
\cU_{t+s}\mu.
$$
This being true for any $\mu$, one has that $\cU_t\cU_s=\cU_{t+s}$. The identity $\cU_0=\Id$ follows directly from
the definition.

\medskip
\emph{(iii) Weak duality identity.}
By the definitions of $\cU_t$ in \eqref{eq:Ut-def} and $P_t$ in  \eqref{eq:Pt-def},
$$
\int_{\R^d}\varphi(\by)\,(\cU_t\mu)(\md\by)
=
\int_{\R^d}
\underbrace{\left(
\int_{\R^d}\varphi(\by)\,\cK_t(\bx,\md\by)
\right)}_{:=P_t\varphi(\bx)}
\mu(\md\bx)
=
\int_{\R^d}(P_t\varphi)(\bx)\,\mu(\md\bx).
$$
Thus $\cU_t=P_t^\ast$ in the usual measure-function duality.

\emph{(iv) Moment preservation.}
Fix $\mu$ with $M_k(\mu)<\infty$ and let $\bX\sim\mu\indep\bZ\sim\cN(0,\bI_d)$. By \eqref{eq:Ut-explicit},
$\bY=e^{-t/2}\bX+\sqrt{1-e^{-t}}\,\bZ\sim\cU_t\mu$. Since
$\|\bY\|\le e^{-t/2}\|\bX\|+\sqrt{1-e^{-t}}\,\|\bZ\|$, Minkowski's inequality in $L^k$ gives
$$
\bigl(\E\|\bY\|^k\bigr)^{1/k}\le e^{-t/2}\bigl(\E\|\bX\|^k\bigr)^{1/k}+\sqrt{1-e^{-t}}\,\bigl(\E\|\bZ\|^k\bigr)^{1/k},
$$
which is finite because Gaussian vectors have moments of all orders. The second inequality uses
$e^{-t/2},\sqrt{1-e^{-t}}\le1$. For $k=2$, expanding the square and using that $\bZ$ is independent
of $\bX$ with $\E\bZ=0$ and $\E\|\bZ\|^2=d$,
$$
\E\|\bY\|^2=e^{-t}\,\E\|\bX\|^2+d\,(1-e^{-t}).
$$
\medskip
\emph{(v) Gaussian invariance.}
Let $\bX\sim\gamma=\cN(0,\bI_d)$ and $\bZ\sim\cN(0,\bI_d)$ be independent Gaussians, then $\cU_t\gamma = \Law(e^{-t/2}\bX+\sqrt{1-e^{-t}}\,\bZ)$ is a sum of centered independent Gaussians, with covariance $e^{-t}\bI_d+(1-e^{-t})\bI_d=\bI_d$. Hence, $\cU_t \gamma = \gamma$. Let $p_0$ be another distribution such that $\cU_t p_0 = p_0$ for all $t$. Then $p_0 = \cU_{nt}p_0$ and, by Proposition~\ref{prop:OU-W2-contraction}, $\cW_2(p_0, \gamma) = \cW_2(\cU_{nt}p_0, \cU_{nt}\gamma) \le e^{-nt/2}\cW_2(p_0, \gamma) \to 0$ as $n \to \infty$, so $p_0 = \gamma$. Hence the uniqueness.

\medskip
\emph{(vi) Action on densities relative to $\gamma$.}
Let $\mu \in L^2(\gamma)$ and let $f=\md\mu/\md\gamma$, then for any bounded measurable test function $\varphi$, by duality,
\begin{align*}
\int_{\R^d}\varphi\,d(\cU_t\mu)
=
\int_{\R^d}P_t\varphi\,d\mu
=
\int_{\R^d}P_t\varphi(\bx) f(\bx)\,\gamma(\md\bx).
\end{align*}
The OU semigroup is reversible with respect to $\gamma$, hence $P_t$ is
self-adjoint in $L^2(\gamma)$ \cite{Bakry2014Analysis}:
$$
\int_{\R^d}P_t\varphi(\bx) f(\bx)\,\gamma(\md\bx)
=
\int_{\R^d}\varphi(\bx) P_t f(\bx)\,\gamma(\md\bx).
$$
Therefore,
$$
\int_{\R^d}\varphi\,d(\cU_t\mu)
=
\int_{\R^d}\varphi(\bx) P_t f(\bx)\,\gamma(\md\bx).
$$
Since this holds for all bounded measurable $\varphi$, we conclude that
$
\frac{d(\cU_t\mu)}{d\gamma}
=
P_t f
$.

\medskip
\emph{(vii) Continuity in time.} 
 Fix $\mu \in \Probspace_2(\R^d)$ and $s\in [0, T]$. Let $\bX \sim \mu, \,  \indep \, \bZ \sim \cN(0, \bI_d)$ and $\bY_t = e^{-t/2}\bX + \sqrt{1 - e^{-t}}\bZ$. Then, $\bY_t \sim \cU_t$ and coupling $\bY_t$ with $\bX$ gives, 
\begin{align*}
\cW_2^2(\cU_t\mu, \mu) 
&\le \E\big[\|\bY_t - \bX\|^2_2\big]\\
&= (e^{-t/2}-1)^2\E\big[\|\bX\|_2^2\big] + (1-e^{-t})\E\big[\|\bZ\|_2^2\big]\\
&= (e^{-t/2}-1)^2\E\big[\|\bX\|_2^2\big] + d(1-e^{-t}). 
\end{align*}
Hence, 
\begin{align*}
\cW_2^2(\cU_t\mu, \mu) 
&\le (e^{-t/2}-1)^2\E\big[\|\bX\|_2^2\big] + d(1-e^{-t}) \underset{t \to 0}{\longrightarrow} 0.
\end{align*}
For the more general case, observe that for $s, t \ge 0$, using the semigroup property (proved in \textit{(ii)}) and the contraction of the OU semigroup in $\cW_2$ (Proposition \ref{prop:OU-W2-contraction}), one has that, 
$$
\cW_2(\cU_t \mu, \cU_s \mu) = \cW_2(\cU_{\min(s, t)}\cU_{|t-s|}\mu, \cU_{\min(s,t)}\mu) \le \cW_2(\cU_{|t-s|}\mu, \mu)
$$
which tends to zero as $|t-s| \to 0$ by the previous argument.
\end{proof}

\subsection{Role of the Gaussian distribution and Decomposition in the Hermite Polynomial Basis}

We refer the reader to \cite{bogachev2015gaussian-measures, janson1997gaussian-hilbert-spaces} for a detailed review of the properties that we recall in the following. 
In our context, $\gamma = \mathcal N(0,\mathbf I_d)$ plays a special role because, as established in Proposition \ref{prop:OU-sampling-operator-properties}, it is the \textit{unique invariant measure} of the Ornstein--Uhlenbeck process, i.e. the unique distribution such that
$$
\cU_t(\gamma) = \gamma.
$$

\paragraph{Hermite polynomials.} 
Hermite polynomials are a family of polynomials that play a key role in the analysis of Gaussian Hilbert spaces, and hence, because of the central role of Gaussians in diffusion processes, in the theory of Markov operators. For $k\in\N$, the degree-$k$ (probabilists') univariate Hermite polynomial is defined for all $x \in \R$ by
\begin{equation}
H_k(x) =
(-1)^k e^{x^2/2}
\frac{\md^k}{\md x^k}
e^{-x^2/2},
\label{eq:univariate_Hermite}
\end{equation}
and satisfies $\int_\R H_k(x)H_\ell(x)\,\gamma_1(\md x) = k!\,\delta_{k\ell}$ with $\gamma_1=\cN(0,1)$. For a multi-index $\bn=(n_1,\ldots,n_d)\in\N^d$ we write $\bn! := \prod_{j=1}^d n_j!$ and $|\bn| := n_1+\cdots+n_d$, and define the \emph{normalized} multivariate Hermite polynomials, as in Section~\ref{sec:fixedpoint}, by
\begin{equation}
H_{\bn}(\bx)
=
\frac{1}{\sqrt{\bn!}}\prod_{j=1}^d H_{n_j}(x_j),
\qquad \bx=(x_1,\ldots,x_d)\in\R^d .
\label{eq:multivariate_Hermite}
\end{equation}
With this normalization, $\langle H_\bn, H_{\mathbf m}\rangle_{L^2(\gamma)} = \delta_{\bn\mathbf m}$.

We now state two key properties that make these polynomials important analytic tools, the proofs of which can be find in \cite{bogachev2015gaussian-measures, janson1997gaussian-hilbert-spaces}.  The first one is that they form an orthogonal basis of $L^2(\gamma)$.

\begin{proposition}[Orthonormal basis of $L^2(\gamma)$]
\label{prop:orthogonal-basis-L2gamma}
    The family $\{H_\bn\}_{\bn\in\N^d}$ defined in \eqref{eq:multivariate_Hermite} is an orthonormal basis of $L^2(\gamma)$: for any $f \in L^2(\gamma)$,
    \begin{equation}
    f = \sum_{\bn\in\N^d} \langle f, H_{\bn} \rangle_{L^2(\gamma)} H_{\bn}
    \qquad\text{and}\qquad
    \|f\|_{L^2(\gamma)}^2 = \sum_{\bn\in\N^d} \langle f, H_{\bn} \rangle_{L^2(\gamma)}^2,
    \end{equation}
    with convergence of the series in $L^2(\gamma)$.
\end{proposition}

The second key property is that Hermite polynomials are eigenfunctions of the OU semigroup $(P_t)_t$, defined in \eqref{eq:Pt-def}.

\begin{proposition}[Diagonalization of the OU semigroup]
\label{prop:hermite-poly-diagonalize-OU-generator}
    For each multi-index $\bn \in \N^d$, denote the degree of a multivariate Hermite polynomial $H_\bn$ as $|\bn| = n_1 + \cdots + n_d$.  Then, for any $t \in [0, T]$,
    $$
    P_t H_\bn = e^{-|\bn|t/2}H_\bn.
    $$
    In other words, multivariate Hermite polynomials are the eigenfunctions of the semigroup $(P_t)_t$, associated with the eigenvalue $e^{-|\bn|t/2}$ for polynomials of degree $|\bn|$.
\end{proposition}

In the context of this work, we are mainly interested in the effect of Hermite polynomials on the sampling operators $(\cU_t)_t$. By combining Proposition \ref{prop:hermite-poly-diagonalize-OU-generator} and Proposition \ref{prop:OU-sampling-operator-properties}, one has that, for any $f \in L^2(\gamma)$,
\begin{equation}
f = \sum_{\bn} \langle f, H_{\bn} \rangle_{L^2(\gamma)} H_{\bn}
\quad \Longrightarrow \quad
P_t f
=
\sum_{\bn}
e^{-|\bn|t/2}
\langle f, H_{\bn} \rangle_{L^2(\gamma)} H_{\bn}.
\end{equation}
Thus, combining this result with \textit{(vi)} in Proposition \ref{prop:OU-sampling-operator-properties}, we obtain a spectral representation of the action of the sampling operator $\cU_t$ on probability measures. 
\begin{proposition}
\label{prop:decomposition-Ut-hermite-basis}
    Let $\mu \in \Probspace_2(\R^d)$ such that $\mu \ll \gamma$ and $f = \tfrac{\md \mu}{\md \gamma} \in L^2(\gamma)$, then,
    \begin{equation}
    \frac{\md(\cU_t \mu)}{\md\gamma}
    =
    \sum_{\bn \in \N^d}
    e^{-|\bn|t/2}
    \langle f, H_{\bn} \rangle_{L^2(\gamma)} H_{\bn}.
    \end{equation}
\end{proposition}

This provides a complete spectral characterization of the OU dynamics.

\subsection{Wasserstein-2 Geometry}

The space of probability measures on $\R^d$ with finite second moment, denoted $\Probspace_2(\R^d)$, is equipped with the Wasserstein-2 distance defined for all $\mu,\nu \in \Probspace_2(\R^d)$ as,
\begin{equation}
\cW_2^2(\mu,\nu)
=
\inf_{\pi \in \Pi(\mu,\nu)}
\int_{\R^d \times \R^d} \|\bx-\by\|^2 \, \pi(\md\bx,\md\by),
\end{equation}
where $\Pi(\mu,\nu)$ is the set of couplings of $\mu$ and $\nu$.
Altogether, the space $(\Probspace_2(\R^d), \cW_2)$ is a complete separable metric space \cite{santambrogio-ot, cuturi-peyre, villani2008optimal}.

A key property used in this work is the contraction of the OU transition operator $(\cU_t)_{t \in [0, T]}$ in the above Wasserstein-2 geometry. 

\begin{proposition}[Wasserstein contraction of OU semigroup \cite{villani2008optimal, Bakry2014Analysis}]
\label{prop:OU-W2-contraction}
Let $(\cU_t)_{t\ge 0}$ denote the Ornstein--Uhlenbeck sampling operator defined in \eqref{eq:Ut-def}. Then for all $\mu,\nu \in \Probspace_2(\R^d)$,
\begin{equation}
\cW_2\bigl(\cU_t \mu, \cU_t \nu\bigr)
\le
e^{-t/2} \cW_2(\mu,\nu).
\end{equation}
\end{proposition}

\begin{proof}
Let $(\bX, \bY)$ be an optimal coupling of $(\mu, \nu)$, i.e. a joint distribution on $\R^d \times \R^d$ such that the marginals satisfy $\bX \sim \mu$ and $\bY \sim \nu$. Additionally, let
$\bZ\sim\mathcal N(0,I_d)$ be an independent Gaussian noise. Furthermore, define the time-$t$ diffused marginals of $\bX$ and $\bY$ as,
$$
\bX_t=e^{-t/2}\bX+\sqrt{1-e^{-t}}\bZ,
\qquad
\bY_t=e^{-t/2}\bY+\sqrt{1-e^{-t}}\bZ.
$$
Then, by definition of $\cU_t$, $\bX_t\sim \cU_t\mu$ and $\bY_t\sim \cU_t\nu$, hence
$$
\cW_2^2(\cU_t\mu,\cU_t\nu)
\le
\E_{(\bX, \bY)}\big[\|\bX_t-\bY_t\|^2\big]
=
e^{-t}\E_{(\bX, \bY)}\big[\|\bX-\bY\|^2\big]
=
e^{-t}\cW_2^2(\mu,\nu).
$$
Taking the square-root yields the result.
\end{proof}

We now state a property of $\cW_2$ that will prove very useful when working on the limiting distribution $\spinf$ which can be expressed as a Neumann series. 

\begin{proposition}[Convexity of $\cW_2^2$ under mixtures]
\label{prop:mixture-convexity-W2}
Let $(\lambda_k)_{k\ge0} \in (\R_+)^\N$ be nonnegative weights such that $\sum_{k=0}^\infty \lambda_k=1$. Let
$(\mu_k)_{k\ge0}$ and $(\nu_k)_{k\ge0}$ be probability measures in
$\Probspace_2(\R^d)$ such that the mixtures
$$
\mu := \sum_{k=0}^\infty \lambda_k \mu_k,
\qquad
\nu := \sum_{k=0}^\infty \lambda_k \nu_k
$$
belong to $\Probspace_2(\R^d)$. Then
\begin{equation}
\label{eq:mixture-convexity-W2}
\cW_2^2(\mu,\nu)
\le
\sum_{k=0}^\infty
\lambda_k
\cW_2^2(\mu_k,\nu_k).
\end{equation}
\end{proposition}

\begin{proof}
For each $k\ge0$, let $\pi_k\in\Gamma(\mu_k,\nu_k)$ be an optimal coupling, so that
$$
\int_{\R^d\times\R^d}\|\bx-\by\|^2\,\pi_k(\md\bx,\md\by)
=
\cW_2^2(\mu_k,\nu_k).
$$
Define $\pi := \sum_{k=0}^\infty \lambda_k \pi_k$, then $\pi$ is a probability measure on $\R^d\times\R^d$. Moreover, 
$$
\sum_{k=0}^\infty \lambda_k \mu_k = \mu, \qquad \sum_{k=0}^\infty \lambda_k \nu_k = \nu,
$$
meaning $\pi\in\Gamma(\mu,\nu)$. Therefore,
$$
\cW_2^2(\mu,\nu)
\le
\int_{\R^d\times\R^d}\|\bx-\by\|^2\,\pi(\md\bx,\md\by).
$$
Using the definition of $\pi$,
$$
\int \|\bx-\by\|^2\,\pi(\md\bx,\md\by)
=
\sum_{k=0}^\infty
\lambda_k
\int \|\bx-\by\|^2\,\pi_k(\md\bx,\md\by).
$$
Thus,
$$
\cW_2^2(\mu,\nu)
\le
\sum_{k=0}^\infty
\lambda_k
\cW_2^2(\mu_k,\nu_k),
$$
which proves the claim.
\end{proof}

\subsection{Neumann Series and Fixed-Point Operators}

This section introduces the following fundamental proposition linking Neumann Series, and linear operators on a Banach space.
\begin{proposition}[Neumann Series \cite{ReedSimonIV}]
\label{prop:neumann-series-convergence}
Let $\mathcal T$ be a linear operator on a Banach space with $\|\mathcal T\| < 1$. Then the Neumann series
\begin{equation}
(I - \mathcal T)^{-1}
=
\sum_{k=0}^{\infty} \mathcal T^k
\end{equation}
converges in operator norm.
\end{proposition}

In the setting described in Section \ref{sec:fixedpoint}, the recursion
\begin{equation}
\mu \mapsto \cU_{t_0}\bigl(\alpha \pdata + (1-\alpha)\mu\bigr)
\end{equation}
can be rewritten as a linear fixed-point equation whose solution is expressed as a Neumann series in the operator $(1-\alpha)\cU_{t_0}$, viewed as a bounded linear operator on the Banach space of finite signed measures on $\R^d$ with the total-variation norm. Since $\cU_{t_0}$ is given by a Markov kernel, $\|\cU_{t_0}\mu\|_{\mathrm{TV}}\le\|\mu\|_{\mathrm{TV}}$, so that $\|(1-\alpha)\cU_{t_0}\|\le1-\alpha<1$ and Proposition~\ref{prop:neumann-series-convergence} applies.

\section{Intermediate Results}

In this section, we summarize two key results that are consistently invoked in our proofs. These are specific to our setting, but fundamentally rely on the classical results stated in Section \ref{app:extended-background}.

The first central result is the contraction property of the one-step ideal sampling operator in the error-free regime described in Section \ref{sec:fixedpoint}.

\begin{proposition}[Contraction of the one-step ideal sampling operator]
\label{prop:contraction-one-step-sampling}
    Let $\cT:\Probspace_2(\R^d) \to \Probspace_2(\R^d)$ denote the operator that maps, to each $\mu \in \Probspace_2(\R^d)$ the end-point marginal of the ideal reverse diffusion (truncated at $t_0>0$) , i.e., 
    $$
    \cT \mu = \cU_{t_0}\big(\alpha \pdata + (1-\alpha)\mu\big).
    $$
    Then, $\cT$ is a contraction on $(\Probspace_2(\R^d), \cW_2)$, with constant $\kappa = \sqrt{1-\alpha} e^{-t_0/2}$, i.e., 
\begin{equation}
    \cW_2(\cT\mu, \cT\nu) \le \sqrt{1-\alpha}e^{-t_0/2}\cW_2(\mu, \nu).
\label{eq:T_contraction}
\end{equation}
\end{proposition}

\begin{proof}
Consider two probability distributions $\mu, \nu \in \Probspace_2(\R^d)$, and denote
$$
q_\mu := \alpha \pdata+(1-\alpha)\mu,
\qquad
q_\nu := \alpha \pdata+(1-\alpha)\nu.
$$
Then, by contractivity of $\cW_2$ on $\cU_{t_0}$ (Proposition \ref{prop:OU-W2-contraction}),
\begin{align}
\cW_2(\cT\mu, \cT\nu) 
&= 
\cW_2\Big(\cU_{t_0}\big(\alpha \pdata + (1-\alpha)\mu\big), \cU_{t_0}\big(\alpha \pdata + (1-\alpha)\nu\big)\Big) \nonumber \\
&\;\leq\; 
e^{-t_0/2} \cW_2(\alpha\pdata + (1-\alpha)\mu,\, \alpha\pdata + (1-\alpha)\nu) \nonumber \\
&= e^{-t_0/2} \cW_2(q_\mu, q_\nu). \label{eq:OU_contraction0}
\end{align}
It remains to bound $\cW_2(q_\mu, q_\nu)$. To that end, consider an optimal coupling $\pi$ of $\mu$ and $\nu$ and let $\lambda$ be the diagonal coupling of $\pdata$ with itself, namely the law of $(\bX, \bX)$ with $\bX \sim \pdata$. Define the coupling,
$$
\tilde{\pi} = \alpha \lambda+(1-\alpha)\pi.
$$
Then $\tilde{\pi}$ has first marginal $q_\mu$ and second marginal $q_\nu$. Therefore, 
$$
\cW_2^2(q_\mu, q_\nu) \le \int_{\R^d \times \R^d}\|\bx - \by\|_2^2 \tilde{\pi}(\md \bx, \md \by).
$$
But, by definition of $\tilde{\pi}$, 
$$
\int_{\R^d \times \R^d}\|\bx - \by\|_2^2 \tilde{\pi}(\md \bx, \md \by) = \alpha \int \|\bx-\by\|_2^2\lambda(\md \bx, \md \by) + (1-\alpha)\int \|\bx-\by\|_2^2 \pi(\md \bx, \md \by).
$$
The first term is zero, because $\lambda$ is supported on the diagonal, while the second term equals $(1-\alpha)\cW_2^2(\mu, \nu)$ since $\pi$ is an optimal coupling between $\mu$ and $\nu$. Hence, 
$$
\cW_2^2(q_\mu, q_\nu) \le (1-\alpha)\cW_2^2(\mu, \nu),
$$
and, equivalently, taking the square root, 
$$
\cW_2(q_\mu, q_\nu) \le \sqrt{1-\alpha}\cW_2(\mu, \nu),
$$
Combining with the OU contraction \eqref{eq:OU_contraction0} gives the final result in \eqref{eq:T_contraction}.
\end{proof}

Since $\cT$ is a contraction, the next natural question is to ask for a characterization of its fixed points. 

\begin{proposition}[Fixed point of the one-step ideal sampling operator]
\label{prop:fixed-point-carac-ideal-sampling-one-step}
    Let $\cT:\Probspace_2(\R^d) \to \Probspace_2(\R^d)$ denote the operator that maps, to each $\mu \in \Probspace_2(\R^d)$ the end-point marginal of the ideal reverse diffusion (truncated at $t_0>0$), i.e., 
    $$
    \cT \mu = \cU_{t_0}\big(\alpha \pdata + (1-\alpha)\mu\big).
    $$
    Then, any fixed point $\mu^\star$ of $\cT$ satisfies, 
    $$
    (\I - (1-\alpha)\cU_{t_0})\mu^\star = \alpha \cU_{t_0}(\pdata).
    $$
\end{proposition}

\begin{proof}
Assume $\mu^\star$ is a fixed point of $\cT$. Then, 
\begin{align*}
\cT \mu^\star = \mu^\star 
&\iff 
\cU_{t_0}\big(\alpha \pdata + (1-\alpha)\mu^\star\big) = \mu^\star\\
&\iff 
\alpha\cU_{t_0}\big(\pdata) + (1-\alpha)\cU_{t_0}(\mu^\star\big) = \mu^\star\\
&\iff (\I - (1-\alpha)\cU_{t_0})\mu^\star = \alpha \cU_{t_0}(\pdata),
\end{align*}
where the second and third lines use the linearity of the sampling operator $\cU_{t_0}$ (Proposition \ref{prop:OU-sampling-operator-properties}).
\end{proof}

\section{Proofs}
\subsection{Proof of Theorem \ref{thm:collapse-distrib-existence-uniqueness}}
\label{app:proof-thm-collapse}
 
\begin{proof}
\medskip
\textit{(i) Limiting distribution expression}

\medskip
Recall that $\cT: \mu \mapsto \cU_{t_0}(\alpha\pdata + (1-\alpha)\mu)$, where the operator $\cU_{t_0}$ is defined in \eqref{eq:UtOU}, is the one-step ideal sampling operator in the error-free regime. By Proposition \ref{prop:contraction-one-step-sampling}, $\cT$ is a contraction on $(\Probspace_2(\R^d), \cW_2)$ with constant $\kappa = \sqrt{(1-\alpha)}\,e^{-t_0/2} \;<\; 1$, and Banach's fixed point theorem \cite{Banach1922} gives existence and uniqueness.

\medskip
\noindent
\medskip
\noindent
It remains to identify $\spinf$ with the series \eqref{eq:neumann}. Define
$$
\nu := \alpha\sum_{k=0}^\infty (1-\alpha)^k\,\cU_{(k+1)t_0}\pdata .
$$
Since the weights $\alpha(1-\alpha)^k$ are nonnegative and sum to one, $\nu$ is a probability measure; moreover, by the identity $\E\|\bY\|^2 = e^{-t}\E\|\bX\|^2 + d(1-e^{-t})$ established in the proof of Proposition \ref{prop:OU-sampling-operator-properties}\textit{(iv)}, each $\cU_{(k+1)t_0}\pdata$ has second moment at most $\int\|\bx\|^2\pdata(\md\bx)+d$ uniformly in $k$, so $\nu\in\Probspace_2(\R^d)$. By linearity of $\cU_{t_0}$ (Proposition \ref{prop:OU-sampling-operator-properties}\textit{(i)}) and the semigroup property \textit{(ii)},
\begin{align*}
\cT\nu
&= \alpha\,\cU_{t_0}\pdata + (1-\alpha)\,\alpha\sum_{k=0}^\infty (1-\alpha)^k\,\cU_{t_0}\cU_{(k+1)t_0}\pdata\\
&= \alpha\,\cU_{t_0}\pdata + \alpha\sum_{k=0}^\infty (1-\alpha)^{k+1}\,\cU_{(k+2)t_0}\pdata
= \alpha\sum_{j=0}^\infty (1-\alpha)^{j}\,\cU_{(j+1)t_0}\pdata
= \nu .
\end{align*}
Hence $\nu$ is a fixed point of $\cT$ and, by uniqueness, $\spinf=\nu$, which is \eqref{eq:neumann}. Equivalently, $\spinf$ is the unique solution of the linear equation $(\I-(1-\alpha)\cU_{t_0})\mu = \alpha\,\cU_{t_0}\pdata$ of Proposition \ref{prop:fixed-point-carac-ideal-sampling-one-step}, and \eqref{eq:neumann} is its Neumann-series solution; this can also be justified via Proposition \ref{prop:neumann-series-convergence} on the Banach space of finite signed measures equipped with the total-variation norm, on which $\|\cU_{t_0}\|\le1$ (it is given by a Markov kernel) and therefore $\|(1-\alpha)\cU_{t_0}\|\le 1-\alpha<1$.

\medskip
\textit{(ii) Geometric rate of convergence}

Define $D_N := \cW_2(\phat^N, \spinf)$, for $N \ge 0$. We claim that $D_{N+1} \leq \kappa\,D_N$. By the triangle inequality,
$$
D_{N+1} 
= \cW_2(\phat^{N+1}, \spinf) 
\leq \cW_2(\phat^{N+1},\, \cT\phat^N) 
+ \cW_2(\cT\phat^N,\, \spinf)
$$

\medskip
\textit{First term.} Since we are working in the error-free regime $\phat^{N+1} = \cT\phat^N$ (perfect sampler) thus $\cW_2(\phat^{N+1},\, \cT\phat^N) = 0$.

\medskip
\textit{Second term.} Since $\spinf = \cT\spinf$, we need to bound $\cW_2(\cT\phat^N,\, \cT\spinf)$. Once again, the contraction property of $\cT$ (Proposition \ref{prop:contraction-one-step-sampling}) yields, 
$$
\cW_2(\cT\phat^N,\, \spinf) \le \kappa \cW_2(\phat^{N}, \spinf) 
$$
where $\kappa := \sqrt{1-\alpha}\,e^{-t_0/2} < 1$. Combining the two bounds:
\begin{equation}\label{eq:one-step-recursion-v1}
D_{N+1} \leq \kappa\,D_N \qquad N\ge 0,
\end{equation}
and iterating \eqref{eq:one-step-recursion-v1} yields the geometric decay rate, 
$$
D_{N} \leq \kappa^N\,D_0 \qquad N\ge 0.
$$
\end{proof}

\subsection{Proof of Corollary \ref{cor:limiting-behavior}}
\label{app:proof-limiting-behavior}

\begin{proof}
In this proof, we use the Neumann series representation (Theorem \ref{thm:collapse-distrib-existence-uniqueness})
\begin{equation}
\label{eq:pstar-limits-proof}
p_\infty^\star
=
\sum_{k=0}^\infty \pi_k(\alpha)
\cU_{(k+1)t_0}(\pdata),
\end{equation}
with geometric weights $\pi_k(\alpha) = \alpha(1-\alpha)^k$. We also denote $D:=\cW_2(\pdata,\gamma)<\infty$. The proofs mainly rely on the two following facts: 
\begin{enumerate}
    \item \textit{Convexity of $\cW_2$ under mixtures.} By Proposition \ref{prop:mixture-convexity-W2}, for any target distribution $\nu \in \Probspace_2(\R^d)$, since $\spinf$ given in \eqref{eq:pstar-limits-proof} is a convex combination with $\sum_k \pi_k(\alpha)=1$,
    $$
    \cW_2^2(\spinf, \nu) \le \sum_{k=0}^\infty \pi_k(\alpha) \cW_2^2(\cU_{(k+1)t_0}\pdata, \nu).
    $$
    \item \textit{Contraction of $\cW_2$ on the OU sampling operator}. By Proposition \ref{prop:OU-W2-contraction}, for any $\mu \in \Probspace_2(\R^d)$, the OU sampling operator satisfies
    \begin{equation}
    \label{eq:OU-contract-gamma-proof}
    \cW_2(\cU_t\mu,\gamma)
    =
    \cW_2(\cU_t\mu,\cU_t\gamma)
    \le
    e^{-t/2}\cW_2(\mu,\gamma),
    \end{equation}
    since $\gamma$ is invariant under $\cU_t$.
\end{enumerate}

\noindent
\textit{(i) Limit $\alpha\to1^-$.}
Fix $t_0>0$ and let $\mu_k:=\cU_{(k+1)t_0}(\pdata)$. Then $\mu_0=\cU_{t_0}(\pdata)$ and
$$
p_\infty^\star
=
\alpha \mu_0+\sum_{k=1}^\infty \pi_k(\alpha)\mu_k.
$$
We can use $\{ \pi_k(\alpha) \}_{k \ge 0}$ to define a natural coupling between $\bX \sim \mu_0$ and $\bY \sim \spinf$, where $\bX=\bY$ with probability $\pi_0(\alpha)=\alpha$. Then, by Proposition \ref{prop:mixture-convexity-W2} with target $\mu_0$,
$$
\cW_2^2(p_\infty^\star,\mu_0)
\le
\sum_{k=1}^\infty \pi_k(\alpha)
\cW_2^2(\mu_k,\mu_0).
$$
Moreover, by the triangle inequality and \eqref{eq:OU-contract-gamma-proof},
$$
\cW_2(\mu_k,\mu_0)
\le
\cW_2(\mu_k,\gamma)+\cW_2(\mu_0,\gamma)
\le
e^{-(k+1)t_0/2}D+e^{-t_0/2}D
\le
2D.
$$
Therefore
$$
\cW_2^2(p_\infty^\star,\cU_{t_0}\pdata)
\le
4D^2\sum_{k=1}^\infty \pi_k(\alpha)
=
4D^2\sum_{k=1}^\infty \alpha(1-\alpha)^k
=
4D^2(1-\alpha),
$$
which converges to zero as $\alpha\to1^-$.

\medskip

\noindent
\textit{(ii) Limit $\alpha\to0^+$.}
Using \eqref{eq:mixture-convexity-W2} with target $\gamma$ and then
\eqref{eq:OU-contract-gamma-proof},
\begin{align}
\cW_2^2(p_\infty^\star,\gamma)
&\le
\sum_{k=0}^\infty
\pi_k(\alpha)
\cW_2^2\bigl(\cU_{(k+1)t_0}\pdata,\gamma\bigr) \\
&\le
D^2
\sum_{k=0}^\infty
\pi_k(\alpha) e^{-(k+1)t_0}.
\end{align}
The geometric sum is
$$
\sum_{k=0}^\infty
\pi_k(\alpha) e^{-(k+1)t_0}
=
\sum_{k=0}^\infty
\alpha(1-\alpha)^k e^{-(k+1)t_0}
=
\frac{\alpha e^{-t_0}}
{1-(1-\alpha)e^{-t_0}}.
$$
Since $t_0>0$, the denominator converges to $1-e^{-t_0}>0$ as
$\alpha\to0^+$. Hence
$$
\cW_2^2(p_\infty^\star,\gamma)
\le
D^2
\frac{\alpha e^{-t_0}}
{1-(1-\alpha)e^{-t_0}}
\longrightarrow 0.
$$

\medskip

\noindent
\textit{(iii) Limit $t_0\to0^+$.}
Fix $\alpha\in(0,1]$. By Proposition \ref{prop:OU-sampling-operator-properties} \textit{(vii)}, for each fixed $k$,
\begin{equation}
\label{eq:continuity-Ukplus1-proof}
\cU_{(k+1)t_0}(\pdata)
\xrightarrow[t_0\to0^+]{\cW_2}
\pdata.
\end{equation}

Now applying Proposition \ref{prop:mixture-convexity-W2} with $\pdata$ yields,
\begin{equation}
    \cW_2^2(p_\infty^\star,\pdata)
\le
\sum_{k=0}^\infty
\alpha(1-\alpha)^k
\cW_2^2\bigl(
\cU_{(k+1)t_0}\pdata,
\pdata
\bigr).
\label{eq:W2_t0_bound}
\end{equation}
But by \eqref{eq:continuity-Ukplus1-proof}, for each fixed $k$, $\cW_2^2\bigl(
\cU_{(k+1)t_0}\pdata,\pdata\bigr) \xrightarrow[t_0\to0^+]{}0$. Moreover,
uniformly in $k$ and $t_0\ge0$,
$$
\cW_2\bigl(\cU_{(k+1)t_0}\pdata,\pdata\bigr)
\le
\cW_2\bigl(\cU_{(k+1)t_0}\pdata,\gamma\bigr)
+
\cW_2(\gamma,\pdata)
\le
2D.
$$
Thus each summand in \eqref{eq:W2_t0_bound} is bounded by
$
4D^2\alpha(1-\alpha)^k,
$
which is summable in $k$. Therefore, by dominated convergence for series, we have 
\begin{align*}
    \lim_{t_0 \to 0} \cW_2^2(p_\infty^\star,\pdata) & \le  \lim_{t_0 \to 0} 
\sum_{k=0}^\infty
\alpha(1-\alpha)^k
\cW_2^2\bigl(
\cU_{(k+1)t_0}\pdata,
\pdata
\bigr) \\ 
& = \sum_{k=0}^\infty
\alpha(1-\alpha)^k
\lim_{t_0 \to 0} \cW_2^2\bigl( \cU_{(k+1)t_0}\pdata,
\pdata
\bigr)
=  0.
\end{align*}

\medskip

\noindent
\textit{(iv) Strict positivity when $\pdata\neq\gamma$.}
If one assumes by contradiction that $\cW_2(p_\infty^\star,\pdata)=0$, then $p_\infty^\star=\pdata$ but since $p_\infty^\star$ is characterized as a fixed
point of $\mu\mapsto \cU_{t_0}\bigl(\alpha \pdata+(1-\alpha)\mu\bigr)$, we obtain
$$
\pdata
=
\cU_{t_0}\bigl(\alpha \pdata+(1-\alpha)\pdata\bigr)
=
\cU_{t_0}\pdata.
$$

Thus $\pdata$ is invariant under the OU transition at time $t_0$. But, the OU semigroup has the unique invariant probability measure $\gamma=\cN(0,\bI_d)$ (Proposition \ref{prop:OU-sampling-operator-properties} \textit{(v)}). Therefore $\pdata = \gamma$, and this contradicts the initial assumption. Thus, $\pdata \neq \gamma$ and $\cW_2(p_\infty^\star,\pdata)>0$.
\end{proof}

\subsection{Moments of limiting distribution}
\label{app:moments-collapse-proof}

Looking at the first two moments of the collapse distribution $\spinf$ hints at what collapse implies in terms of concentration.

\begin{proposition}[Moments]\label{prop:moments-collapse}
Let $\bmu := \E_{\pdata}[\bX]$ and define
$$
c_1 := \frac{\alpha\,e^{-t_0/2}}{1-(1-\alpha)e^{-t_0/2}},
\qquad
c_2 := \frac{\alpha\,e^{-t_0}}{1-(1-\alpha)e^{-t_0}} .
$$
The mean, second-moment matrix and covariance of $\spinf$ are
\begin{align}
  \E_{\spinf}[\bX]
  &= c_1\,\bmu,
  \label{eq:mean-collapse}\\[4pt]
  \E_{\spinf}[\bX\bX^\top]
  &= c_2\,\E_{\pdata}[\bX\bX^\top] + (1-c_2)\,\bI_d,
  \label{eq:second-moment-collapse}\\[4pt]
  \Cov_{\spinf}(\bX)
  &= c_2\,\Cov_{\pdata}(\bX) + (1-c_2)\,\bI_d + (c_2 - c_1^2)\,\bmu\bmu^\top .
  \label{eq:cov-collapse}
\end{align}
\end{proposition}

\begin{remark}
Here $0<c_1,c_2<1$ and $c_2\ge c_1^2$, with equality only if $t_0=0$. Thus the mean is attenuated relative to $\pdata$ for any $\alpha<1$, and when $\pdata$ is centered the covariance is pulled toward $\bI_d$; for non-centered $\pdata$ an additional rank-one term along $\bmu$ appears, reflecting the spread of the block means $e^{-(k+1)t_0/2}\bmu$ in \eqref{eq:neumann}. Note that $c_1 = m_\bn(\alpha,t_0)$ for $|\bn|=1$ and $c_2 = m_\bn(\alpha,t_0)$ for $|\bn|=2$, consistently with Proposition~\ref{prop:spectral-rpz-limit}.
\end{remark}

\begin{proof}
Write $\lambda_k := e^{-(k+1)t_0}$ and $\pi_k := \alpha(1-\alpha)^k$, so that $\spinf=\sum_{k\ge0}\pi_k\,\cU_{(k+1)t_0}\pdata$ by \eqref{eq:neumann}. From \eqref{eq:UtOU}, if $\bY\sim\cU_t\mu$ with $\bX\sim\mu$ and $\bZ\sim\Normal(0,\Id)$ independent, then $\bY = e^{-t/2}\bX+\sqrt{1-e^{-t}}\,\bZ$ and hence
$$
\E[\bY]=e^{-t/2}\,\E[\bX],
\qquad
\E[\bY\bY^\top]=e^{-t}\,\E[\bX\bX^\top]+(1-e^{-t})\,\bI_d .
$$
The mean and the second-moment matrix are linear in the underlying measure, so they commute with the mixture:
\begin{align*}
\E_{\spinf}[\bX]
&=\sum_{k\ge0}\pi_k\,\lambda_k^{1/2}\,\bmu
= \alpha e^{-t_0/2}\sum_{k\ge0}\bigl[(1-\alpha)e^{-t_0/2}\bigr]^k\,\bmu
= c_1\,\bmu,\\
\E_{\spinf}[\bX\bX^\top]
&=\sum_{k\ge0}\pi_k\bigl(\lambda_k\,\E_{\pdata}[\bX\bX^\top]+(1-\lambda_k)\,\bI_d\bigr)
= c_2\,\E_{\pdata}[\bX\bX^\top]+(1-c_2)\,\bI_d,
\end{align*}
using $\sum_k\pi_k\lambda_k = c_2$ and $\sum_k\pi_k=1$. For the covariance, we use the above expressions, 
\begin{align*}
\Cov_{\spinf}(\bX)
&=\E_{\spinf}[\bX\bX^\top]-\E_{\spinf}[\bX]\,\E_{\spinf}[\bX]^\top\\
&= c_2\bigl(\Cov_{\pdata}(\bX)+\bmu\bmu^\top\bigr)+(1-c_2)\,\bI_d - c_1^2\,\bmu\bmu^\top\\
&= c_2\,\Cov_{\pdata}(\bX)+(1-c_2)\,\bI_d+(c_2-c_1^2)\,\bmu\bmu^\top .
\end{align*}
Finally, $c_1=\sum_k\pi_k\lambda_k^{1/2}$ and $c_2=\sum_k\pi_k\lambda_k$ are the first and second moments of $\lambda_K^{1/2}$ for $K$ geometric with $\Prob(K=k)=\pi_k$, so $c_2-c_1^2=\mathrm{Var}(\lambda_K^{1/2})\ge0$, with equality iff $\lambda_k$ does not depend on $k$, i.e.\ iff $t_0=0$.
\end{proof}

\subsection{Proof of Proposition~\ref{prop:spectral-rpz-limit}}
\label{app:proof-spectral-rpz-limit}

\begin{proof}
Under Assumption~\ref{ass:spectral-regularity-A4}, the data distribution admits
a density $\fdata := \frac{\md \pdata}{\md \gamma}
\in L^2(\gamma)$, where $\gamma=\cN(0,\bI_d)$. Since the OU sampling operator
$\cU_t$ preserves absolute continuity with respect to $\gamma$, the collapse
distribution $\spinf$ also admits a density
$
\sfinf := \frac{\md \spinf}{\md \gamma}.
$

From Theorem~\ref{thm:collapse-distrib-existence-uniqueness}, we have the
Neumann-series representation
\begin{equation}
\label{eq:neumann-measure-proof}
\spinf
=
\alpha \sum_{k=0}^\infty (1-\alpha)^k
\cU_{(k+1)t_0}\pdata.
\end{equation}

Recall $(P_t)_{t\ge0}$ denotes the OU semigroup acting on $L^2(\gamma)$ (see Section \ref{app:extended-background}). By Proposition \ref{prop:OU-sampling-operator-properties} \textit{(vi)}),
$$
\frac{\md(\cU_t\mu)}{\md\gamma}
=
P_t\!\left(\frac{\md\mu}{\md\gamma}\right).
$$
Applying this to \eqref{eq:neumann-measure-proof} and using the linearity of $\cU_t$, we obtain
\begin{equation}
\label{eq:neumann-density-proof}
\sfinf
=
\alpha \sum_{k=0}^\infty (1-\alpha)^k
P_{(k+1)t_0} \fdata,
\end{equation}
with convergence in $L^2(\gamma)$.

Since $\fdata \in L^2(\gamma)$, by Proposition \ref{prop:orthogonal-basis-L2gamma}
\begin{equation}
\label{eq:ortho-decomp-mode-proof}
\fdata
=
\sum_{\bn\in\mathbb N^d}
\langle \fdata, H_{\bn}\rangle_{L^2(\gamma)} H_{\bn},
\end{equation}
with convergence in $L^2(\gamma)$. But, by Proposition \ref{prop:hermite-poly-diagonalize-OU-generator}, for each $\bn \in \N^d$ and $k\ge0$,
\begin{equation}
\label{eq:hermite-diagonalize-kplus1-proof}
P_{(k+1)t_0}H_\bn = e^{-|\bn|(k+1)t_0/2}H_\bn.
\end{equation}
Thus, applying $P_{(k+1)t_0}$, and combining equations \eqref{eq:hermite-diagonalize-kplus1-proof} and \eqref{eq:ortho-decomp-mode-proof} yields,
\begin{equation}
\label{eq:P-kplus1-on-hermite}
P_{(k+1)t_0} \fdata
=
\sum_{\bn\in\mathbb N^d}
e^{-|\bn|(k+1)t_0/2}
\langle \fdata, H_{\bn}\rangle_{L^2(\gamma)} H_{\bn}.
\end{equation}

Substituting \eqref{eq:P-kplus1-on-hermite} into \eqref{eq:neumann-density-proof},
\begin{align}
f_\infty^\star
&=
\alpha \sum_{k=0}^\infty (1-\alpha)^k
\sum_{\bn\in\mathbb N^d}
e^{-|\bn|(k+1)t_0/2}
\langle f_{\mathrm{data}}, H_{\bn}\rangle H_{\bn}.
\end{align}
Since $P_t$ is a contraction on $L^2(\gamma)$ and the coefficients are
square-summable, the series converges absolutely in $L^2(\gamma)$, which
justifies exchanging the sums. Hence,
$$
f_\infty^\star
=
\sum_{\bn\in\mathbb N^d}
\left[
\alpha \sum_{k=0}^\infty (1-\alpha)^k
e^{-|\bn|(k+1)t_0/2}
\right]
\langle f_{\mathrm{data}}, H_{\bn}\rangle H_{\bn}.
$$

Now looking at each term inside the series on $\N^d$, we observe that, for each $\bn \in \N^d$,
\begin{align*}
\alpha \sum_{k=0}^\infty (1-\alpha)^k e^{-|\bn|(k+1)t_0/2}
&=
\alpha e^{-|\bn|t_0/2}
\sum_{k=0}^\infty \left[(1-\alpha)e^{-|\bn|t_0/2}\right]^k \\
&= 
\alpha e^{-|\bn|t_0/2}
\frac{1}
{1-(1-\alpha)e^{-|\bn|t_0/2}} \qquad ((1-\alpha)e^{-|\bn|t_0/2}<1).
\end{align*}
Denoting $m_\bn(\alpha, t_0) = \frac{\alpha e^{-|\bn|t_0/2}}
{1-(1-\alpha)e^{-|\bn|t_0/2}}$, yields \eqref{eq:spectral}:
$$
\langle f_\infty^\star, H_{\bn}\rangle_{L^2(\gamma)}
=
m_{\bn}(\alpha,t_0)
\langle f_{\mathrm{data}}, H_{\bn}\rangle_{L^2(\gamma)},
$$

Therefore, each Hermite mode of the data distribution is attenuated by the explicit factor \(m_{\bn}(\alpha,t_0)\). This shows that the collapse distribution is a spectrally filtered version of \(\pdata\), with preferential suppression of high-frequency modes.
\end{proof}

\subsection{Proof of the finite-moment attenuation identity (Remark~\ref{rem:moment-attenuation})}
\label{app:proof-moment-attenuation}

\begin{proof}
Fix $\bn\in\N^d$. Assume $\int_{\R^d}\|\bx\|^{|\bn|}\,\pdata(\md\bx)<\infty$, let $\bX\sim\pdata$ and $\bZ\sim\Normal(0,\Id)$ be
independent, and write $\sigma_t:=\sqrt{1-e^{-t}}$. $H_\bn$ is a polynomial of total degree $|\bn|$, i.e.\ a finite sum
$\sum_{|\bbeta|\le |\bn|}c_{\bbeta}\,\bx^{\bbeta}$. Since $|\bx^{\bbeta}|\le\|\bx\|^{|\bbeta|}\le1+\|\bx\|^{|\bn|}$,
we have $|H_\bn(\bx)|\le C_\bn(1+\|\bx\|^{|\bn|})$ with $C_\bn:=\sum_\beta|c_\beta|$. By
Proposition~\ref{prop:OU-sampling-operator-properties}(iv), 
\begin{equation}\label{eq:hermite-L1-uniform}
\sup_{t\ge0}\big\{\E_{\cU_t\pdata}[|H_\bn|]\big\}
\;\le\;C_\bn\Bigl(1+\bigl(M_{|\bn|}(\pdata)^{1/|\bn|}+(\E\|\bZ\|^{|\bn|})^{1/|\bn|}\bigr)^{|\bn|}\Bigr)=:B_\bn<\infty .
\end{equation}
Fubini's theorem therefore applies, and conditioning on $\bX$ gives
\begin{equation}
\label{eq:proof-attenuation-form-without-A4}
\E_{\cU_t\pdata}[H_\bn]=\E\bigl[H_\bn(e^{-t/2}\bX+\sigma_t\bZ)\bigr]=\E\bigl[(P_tH_\bn)(\bX)\bigr]\overset{\mathrm{Prop.}~\ref{prop:hermite-poly-diagonalize-OU-generator}}{=} e^{-|\bn|t/2}\E_{\pdata}\bigl[H_\bn\bigr].
\end{equation}
Recall that $\spinf$ is defined in \eqref{eq:neumann} as, 
$$
\spinf = \sum_{j\ge0} \alpha(1-\alpha)^j\,\cU_{(j+1)t_0}\pdata.
$$
so that, by monotone convergence, for every measurable
$g\ge0$, 
$$
\int g\,\md\spinf=\sum_j \alpha(1-\alpha)^j\int g\,\md(\cU_{(j+1)t_0}\pdata).
$$ 
Taking $g=|H_\bn|$, \eqref{eq:hermite-L1-uniform} gives $\E_{\spinf}|H_\bn|\le B_\bn<\infty$ and using \eqref{eq:proof-attenuation-form-without-A4}, 
\begin{align*}
\E_{\spinf}[H_\bn]
& =\sum_{j\ge0}\alpha(1-\alpha)^j\,\E_{\cU_{(j+1)t_0}\pdata}[H_\bn] \\ 
& =\alpha\,\E_{\pdata}[H_\bn]\sum_{j\ge0}(1-\alpha)^j\,e^{-|\bn|(j+1)t_0/2}
=m_\bn(\alpha,t_0)\,\E_{\pdata}[H_\bn].
\end{align*}
\end{proof}
 
\subsection{Proof of Theorem~\ref{thm:annealed-truncation-corrected}}
\label{app:proof-annealed-truncation-corrected}
\begin{proof}
We consider the error-free recursion
\begin{equation}
\label{eq:proof-annealed-recursion}
\phat^{i+1}
=
\cU_{t_0^{(i)}}\bigl(\alpha\,\pdata + (1-\alpha)\phat^i\bigr),
\qquad i\ge 0.
\end{equation}

We first prove \eqref{eq:annealed-unrolled} by induction on \(N\).
For \(N=1\), by linearity of $\cU_t$ (Proposition \ref{prop:OU-sampling-operator-properties} \textit{(i)}),
$$
\phat^1
=
\cU_{t_0^{(0)}}\bigl(\alpha\,\pdata + (1-\alpha)\phat^0\bigr)
=
\alpha\,\cU_{t_0^{(0)}}\pdata
+
(1-\alpha)\,\cU_{t_0^{(0)}}\pdata,
$$
which agrees with \eqref{eq:annealed-unrolled}, since
$$
\sigma_{0,1}=t_0^{(0)},
\qquad
s_{0,1}=t_0^{(0)}.
$$

Assume now that \eqref{eq:annealed-unrolled} holds for some \(N\ge 1\), that is,
$$
\phat^N
=
\alpha \sum_{m=0}^{N-1}(1-\alpha)^m\,\cU_{\sigma_{m,N}}\pdata
+
(1-\alpha)^N \cU_{s_{0,N}}\pdata,
$$
where $\sigma_{m,N}=\sum_{\ell=N-1-m}^{N-1} t_0^{(\ell)}$ and $s_{0,N}=\sum_{\ell=0}^{N-1} t_0^{(\ell)}$. Once again, by linearity of $\cU_t$  and leveraging the induction assumption,
\begin{align*}
\phat^{N+1}
&=
\cU_{t_0^{(N)}}\bigl(\alpha\,\pdata + (1-\alpha)\phat^N\bigr) \\
&=
\alpha\,\cU_{t_0^{(N)}}\pdata
+
(1-\alpha)\,\cU_{t_0^{(N)}}\phat^N \\
&=
\alpha\,\cU_{t_0^{(N)}}\pdata
+
(1-\alpha)\,\cU_{t_0^{(N)}}\!\left[
\alpha \sum_{m=0}^{N-1}(1-\alpha)^m\,\cU_{\sigma_{m,N}}\pdata
+
(1-\alpha)^N \cU_{s_{0,N}}\phat^0
\right]\\
&=\alpha\,\cU_{t_0^{(N)}}\pdata
+
\!\left[
\alpha \sum_{m=0}^{N-1}(1-\alpha)^{m+1}\,\cU_{t_0^{(N)}}\cU_{\sigma_{m,N}}\pdata
+
(1-\alpha)^{N+1} \cU_{t_0^{(N)}}\cU_{s_{0,N}}\phat^0
\right]
\end{align*}
But, by semigroup property of $\cU_t$ (Proposition \ref{prop:OU-sampling-operator-properties}, \textit{(ii)}), $,\cU_{t_0^{(N)}}\cU_{\sigma_{m,N}} = \cU_{t_0^{(N)} + \sigma_{m,N}}$ and similarly $\cU_{t_0^{(N)}}\cU_{s_{0,N}} = \cU_{t_0^{(N)}+s_{0,N}}$. Observing that $t_0^{(N)}+\sigma_{m,N}
= \sum_{\ell=N-m-1}^{N} t_0^{(\ell)} = \sigma_{m+1,N+1}$ and $t_0^{(N)}+s_{0,N}=s_{0,N+1}$ finally yields,
\begin{equation}
\phat^{N+1}
=
\alpha\sum_{m=0}^{N}(1-\alpha)^m\,\cU_{\sigma_{m,N+1}}\pdata
+
(1-\alpha)^{N+1}\cU_{s_{0,N+1}}\pdata,
\label{eq:phat-expression-proof-trunc}
\end{equation}
which is exactly \eqref{eq:annealed-unrolled} with $N$ replaced by \(N+1\). This proves the
representation formula for all \(N\ge 1\).

\medskip
\noindent
\textit{Asymptotic behavior when \(t_0^{(i)}\to \tinf\).}
Assume that \(t_0^{(i)}\to \tinf\in[0,\infty)\). We define the geometric weights $\pi_m(\alpha) = \alpha(1-\alpha)^m$ and the \textit{annealed truncation limiting distribution} $\truncpinf$ as,
$$
\truncpinf
:=
\sum_{m=0}^{\infty}\pi_m(\alpha)\,
\cU_{(m+1)\tinf}\pdata.
$$
The objective of this part is to show that $\cW_2(\phat^N, \truncpinf) \underset{N \to \infty}{\longrightarrow}0$. Observe that, for any fixed \(m\ge 0\),
\begin{equation}
    \label{eq:proof-sigma-limit}
\sigma_{m,N} : =\sum_{\ell=N-1-m}^{N-1} t_0^{(\ell)} \, \longrightarrow  \, (m+1)\tinf
\quad\text{as }N\to\infty.
\end{equation}
By continuity of \(t\mapsto \cU_t(\pdata)\) in \(\cW_2\) (Proposition \ref{prop:OU-sampling-operator-properties}, \textit{(vii)}),
\begin{equation}
\label{eq:continuity-tinf-proof}
\cU_{\sigma_{m,N}}(\pdata)
\;\xrightarrow[N\to\infty]{\cW_2}\;
\cU_{(m+1)\tinf}(\pdata)
\qquad\text{for each fixed }m.
\end{equation}
We now compare $\phat^N$ and $\truncpinf$ as mixtures with the
same total weights. To that end, we decompose $\truncpinf$ as follows,
$$
p_\infty^{(\tinf)}
=
\sum_{m=0}^{N-1}\pi_m(\alpha)\,\cU_{(m+1)\tinf}(\pdata)
+
(1-\alpha)^N r_N,
$$
where
$$
r_N
:=
\frac{1}{(1-\alpha)^N}
\sum_{m=N}^{\infty}\pi_m(\alpha)
\cU_{(m+1)\tinf}\pdata
=
\alpha\sum_{j=0}^{\infty}(1-\alpha)^j
\cU_{(N+j+1)\tinf}\pdata.
$$
The measure \(r_N\) is a probability measure. By convexity of \(\cW_2^2\) under mixtures (Proposition~\ref{prop:mixture-convexity-W2}), we have
\begin{align}
\cW_2^2\!\left(\phat^N,p_\infty^{(\tinf)}\right)
&\le
\sum_{m=0}^{N-1}\pi_m(\alpha)
\cW_2^2\!\left(
\cU_{\sigma_{m,N}}\pdata,
\cU_{(m+1)\tinf}\pdata
\right)
+
(1-\alpha)^N
\cW_2^2\!\left(
\cU_{s_{0,N}}\pdata,
r_N
\right).
\label{eq:annealed-asymptotic-bound}
\end{align}

We treat the two terms separately. 

\medskip
\textit{First term.} Fix $M \ge 0$, then
\begin{align*}
\sum_{m=0}^{N-1}\pi_m(\alpha)
\cW_2^2\!\left(
\cU_{\sigma_{m,N}}\pdata,
\cU_{(m+1)\tinf}\pdata 
\right)
&\le 
\sum_{m=0}^{M}\pi_m(\alpha)
\cW_2^2\!\left(
\cU_{\sigma_{m,N}}\pdata,
\cU_{(m+1)\tinf}\pdata
\right)\\
&+
\sum_{m=M+1}^{N-1}\pi_m(\alpha) B,
\end{align*}
where \(B<\infty\) is a uniform second-moment bound. Proposition \ref{prop:OU-sampling-operator-properties} \textit{(iv)} ensures that, for any $t$, $\cU_t \pdata$ has uniformly bounded second moment,
\begin{equation}
\label{eq:uniform-second-bound-moment-proof-trunc}
\sup_{t\ge0}\int \|\bx\|^2_2\,\cU_t\pdata(\md\bx)
\le
\int\|\bx\|^2_2\,\pdata(\md\bx)+d,
\end{equation}
Therefore, all the above pairwise $\cW_2$-distances are bounded uniformly.

For fixed $M$, the finite sum converges to zero by the continuity invoked in \eqref{eq:continuity-tinf-proof}. The tail satisfies
\begin{align*}
\sum_{m=M+1}^{N-1}\pi_m(\alpha) B
= 
B\sum_{m=M+1}^{N-1}\alpha(1-\alpha)^m  
&\le
B\sum_{m=M+1}^{\infty}\alpha(1-\alpha)^m \\[0.3em]
&=
B(1-\alpha)^{M+1}.
\end{align*}
which can be made arbitrarily small by choosing \(M\) large. Hence the first term in \eqref{eq:annealed-asymptotic-bound} converges to zero.

\medskip
\textit{Second term.} For the second term in \eqref{eq:annealed-asymptotic-bound}, as stated in  \eqref{eq:uniform-second-bound-moment-proof-trunc}, $\cU_{s_{0,N}}\pdata$  has uniformly bounded second moments by Proposition \ref{prop:OU-sampling-operator-properties} \textit{(iv)}. As $r_N$ is  a convex combination of measures of the form $\cU_{t_j}\pdata$ for $j \ge 0$, its second moment is also bounded. Therefore, there exists a constant $B_0 < \infty$ such that,
$$
\cW_2^2\!\left(
\cU_{s_{0,N}}(\phat^0),
r_N
\right)
\le B_0
$$
for all \(N\). Since \((1-\alpha)^N\to0\), the second term in
\eqref{eq:annealed-asymptotic-bound} also converges to zero.

Combining the two estimates yields the desired convergence
$$
\cW_2\!\left(\phat^N,p_\infty^{(\tinf)}\right)\to0.
$$

\medskip
\textit{Case $\tinf = 0$.}
If \(\tinf=0\), then \(\cU_{(m+1)\tinf}=\cU_0=\I\) for every \(m\), and thus
$$
p_\infty^{(0)}
=
\alpha \sum_{m=0}^{\infty}(1-\alpha)^m\,\pdata
=
\pdata,
$$
since \(\sum_{m=0}^{\infty}\alpha(1-\alpha)^m=1\). This proves that the recursive
compounding effect disappears asymptotically when \(t_0^{(i)}\to 0\).
\end{proof}

\subsection{Proof of the finite-$N$ schedule bound}
\label{app:proof-finite-N-schedule}

\begin{proof}
Let $h(t):=\min\bigl\{d\,t+\tfrac{t^2}{4}\,\E_{\pdata}\|\bX\|^2,\;4\cW_2^2(\pdata,\gamma)\bigr\}$, where the
two expressions inside the minimum come from the following bounds, valid for every $t\ge0$:
\begin{enumerate}
    \item By the proof of Proposition~\ref{prop:OU-sampling-operator-properties}(vii) together with $1-e^{-x}\le x$
    $$
    \cW_2^2(\cU_t\pdata,\pdata)\le(1-e^{-t/2})^2\,\E_{\pdata}\|\bX\|^2+d\,(1-e^{-t});
    $$
    \item By the proof of Corollary~\ref{cor:limiting-behavior}, $$
    \cW_2(\cU_t\pdata,\pdata)\le\cW_2(\cU_t\pdata,\gamma)+\cW_2(\gamma,\pdata)\le(e^{-t/2}+1)\,\cW_2(\pdata,\gamma).
    $$ 
\end{enumerate}
Hence $\cW_2^2(\cU_t\pdata,\pdata)\le h(t)$.
Decomposing $\pdata=\alpha\sum_{m=0}^{N-1}(1-\alpha)^m\,\pdata+(1-\alpha)^N\pdata$ and applying the
convexity of $\cW_2^2$ under mixtures (Proposition~\ref{prop:mixture-convexity-W2}) term by term against
the representation \eqref{eq:annealed-unrolled} of Theorem~\ref{thm:annealed-truncation-corrected} yields
$$
\cW_2^2(\phat^N,\pdata)\;\le\;\alpha\sum_{m=0}^{N-1}(1-\alpha)^m\,h(\sigma_{m,N})\;+\;(1-\alpha)^N\, h(s_{0,N}).
$$
\end{proof}

\subsection{Proof of Theorem~\ref{thm:convergence-imperfect-regime}}
\label{app:proof-convergence-imperfect-regime}

\begin{proof}
The proof proceeds in three steps: a one-step recursion, its iteration by induction, and verification of the two regimes.

\textit{(a) One-step Recursion.}

By the triangle inequality,
$$
D_{N+1} 
= \cW_2(\phat^{N+1}, \spinf) 
\leq \cW_2(\phat^{N+1},\, \cT(\phat^N)) 
+ \cW_2(\cT(\phat^N),\, \spinf)
$$
For the first term: since $\phat^{N+1} = \hat{\cS}_N(q_N)$ (imperfect sampler) and $\cT(\phat^N) = \cU_{t_0}(q_N)$ (ideal sampler applied to the same training mixture $q_N$), the definition of $\delta_N$ in~\eqref{eq:one-step-error} gives
$$
\cW_2(\phat^{N+1},\, \cT(\phat^N)) 
= \cW_2(\hat{\cS}_N(q_N),\, \cU_{t_0}(q_N)) 
\leq \delta_N.
$$
For the second term: since $\spinf = \cT(\spinf)$, we need to bound $\cW_2(\cT(\phat^N),\, \cT(\spinf))$. But $\cT$ is a contraction (Proposition \ref{prop:contraction-one-step-sampling}) thus, 
$$
\cW_2(\cT(\phat^N),\, \spinf) 
= \cW_2(\cT(\phat^N),\, \cT(\spinf))
\leq \sqrt{1-\alpha}\,e^{-t_0/2}\;D_N 
= \kappa\,D_N,
$$
where $\kappa := \sqrt{1-\alpha}\,e^{-t_0/2} < 1$. Combining the two bounds:
\begin{equation}\label{eq:one-step-recursion}
D_{N+1} \leq \kappa\,D_N + \delta_N, \qquad N\ge 0.
\end{equation}

\medskip
\textit{(b) Full Induction.}

We prove~\eqref{eq:W2-recursion-main} by induction on $N$. The base case $N = 1$ follows directly from~\eqref{eq:one-step-recursion}: $D_1 \leq \kappa\,D_0 + \delta_0$. Assume the bound holds at generation $N$:
$$
D_N \leq \kappa^N\,D_0 + \sum_{i=0}^{N-1} \kappa^{N-1-i}\,\delta_i.
$$
Then, substituting into~\eqref{eq:one-step-recursion},
\begin{align*}
D_{N+1} 
\leq \kappa\,D_N + \delta_N 
\leq \kappa\Bigl(\kappa^N\,D_0 + \sum_{i=0}^{N-1} \kappa^{N-1-i}\,\delta_i\Bigr) + \delta_N 
= \kappa^{N+1}\,D_0 + \sum_{i=0}^{N} \kappa^{N-i}\,\delta_i,
\end{align*}
which completes the induction and establishes~\eqref{eq:W2-recursion-main}.

We now prove the two asymptotic regimes of Theorem \ref{thm:convergence-imperfect-regime}.

\noindent\textit{(i) Summable perturbations.}
Assume $\sum_{i \geq 0} \delta_i < \infty$, which implies $\delta_N \to 0$ as $N \to \infty$. The first term $\kappa^N\,D_0 \to 0$. For the convolution term, fix $\varepsilon > 0$ and choose $M$ such that $\sum_{i=M}^{\infty} \delta_i < \varepsilon$. For $N > M$, split:
$$
\sum_{i=0}^{N-1} \kappa^{N-1-i}\,\delta_i
= \underbrace{\sum_{i=0}^{M-1} \kappa^{N-1-i}\,\delta_i}_{\to\,0 \text{ as } N \to \infty \text{ ($M$ fixed, $\kappa < 1$)}}
+ \underbrace{\sum_{i=M}^{N-1} \kappa^{N-1-i}\,\delta_i}_{\leq\, \sum_{i=M}^{N-1} \delta_i \;\leq\; \sum_{i=M}^{\infty} \delta_i \,<\, \varepsilon}.
$$
Since $\varepsilon$ is arbitrary, $D_N \to 0$.

\noindent\textit{(ii) Constant perturbation floor.}
If $\delta_i <\delta$, the geometric sum evaluates to
$\sum_{i=0}^{N-1} \kappa^{N-1-i} = \sum_{j=0}^{N-1} \kappa^j = (1-\kappa^N)/(1-\kappa)$,
so
$$
D_N 
\leq \kappa^N\,D_0 + \frac{\delta\,(1-\kappa^N)}{1-\kappa}
\;\xrightarrow{N \to \infty}\; \frac{\delta}{1-\kappa}. 
$$
\end{proof}

\subsection{Explicit characterization of perturbation}
\label{app:proof-perturbation-decomp-prop}

\begin{proposition}[Linear perturbation recursion]
\label{prop:perturbation-decomposition}
Define the perturbation $\xi_N := \phat^N - \spinf$ and, for each generation $i$, let $\eta_i \;:=\; \phat^{i+1} - \cU_{t_0}(q_i)$ be the \emph{one-step perturbation} at generation $i$. Then, the imperfect distribution $ \phat^N$ satisfies:
\begin{equation}\label{eq:explicit-decomposition}
  \phat^N
  \;=\;
  \spinf
  \;+\;
  \underbrace{(1-\alpha)^N\,\cU_{Nt_0}(\xi_0)}_{\text{initialization transient}}
  \;+\;
  \underbrace{\sum_{i=0}^{N-1}(1-\alpha)^{N-1-i}\,
    \cU_{(N-1-i)t_0}(\eta_i)}_{\text{accumulated perturbations}}.
\end{equation}
\end{proposition}

\begin{proof}
Define the perturbation $\xi_N := \phat^N - \spinf$. We first prove it satisfies the following linear recursion: 
\begin{equation}\label{eq:xi-recursion}
  \xi_{N+1} \;=\; (1-\alpha)\,\cU_{t_0}(\xi_N) \;+\; \eta_N.
\end{equation}
Indeed, by definition of $\eta_N$, we have $\phat^{N+1} = \cU_{t_0}(q_N) + \eta_N$. Since $\spinf = \cU_{t_0}(\alpha\,\pdata + (1-\alpha)\spinf)$, subtracting and using the linearity of $\cU_{t_0}$ on signed measures, we obtain the following:
\begin{align*}
  \xi_{N+1} 
  &= \phat^{N+1} - \spinf \\
  &= \cU_{t_0}(q_N) - \cU_{t_0}(\alpha\,\pdata + (1-\alpha)\spinf) + \eta_N \\
  &= \cU_{t_0}(\alpha \pdata + (1-\alpha)\phat^N) - \cU_{t_0}(\alpha\,\pdata + (1-\alpha)\spinf) + \eta_N\\
  &= \alpha\cU_{t_0}(\pdata) + (1-\alpha)\cU_{t_0}(\phat^N) - \alpha\cU_{t_0}(\pdata) - (1-\alpha)\cU_{t_0}(\spinf) + \eta_N
  \qquad \text{ by linearity of } \cU_{t_0}\\
  &= \cU_{t_0}\!\bigl((1-\alpha)(\phat^N - \spinf)\bigr) + \eta_N \qquad \text{ by linearity of } \cU_{t_0} \\
  &= (1-\alpha)\,\cU_{t_0}(\xi_N) + \eta_N.
\end{align*}
Unrolling the recursion and applying the semigroup property $\cU_{t_0}^k = \cU_{kt_0}$ yields the end result:
$$
  \phat^N
  \;=\;
  \spinf
  \;+\;
  (1-\alpha)^N\,\cU_{Nt_0}(\xi_0)
  \;+\;
  \sum_{i=0}^{N-1}(1-\alpha)^{N-1-i}\,
  \cU_{(N-1-i)t_0}(\eta_i).
$$
\end{proof}

Equation~\eqref{eq:explicit-decomposition} reveals that each iterate
$\phat^N$ decomposes as $\spinf$ plus two correction terms. The
initialization transient decays geometrically \emph{and} is progressively smoothed by $Nt_0$ units of OU evolution. Each perturbation $\eta_i$ introduced at generation $i$ is subsequently
smoothed by $(N{-}1{-}i)t_0$ units of OU evolution before being
observed at generation $N$, with amplitude damped by $(1{-}\alpha)^{N-1-i}$.

\subsection{Proof of Proposition \ref{prop:spectral-perturbation}}
\label{app:proof-spectral-perturbation-cor}
\begin{proof}
In Proposition \ref{prop:perturbation-decomposition}, we explicitly wrote $\phat^N$ in presence of errors as, 
$$
\phat^N
\;=\;
\spinf
\;+\;
(1-\alpha)^N\,\cU_{Nt_0}(\xi_0)
\;+\;
\sum_{i=0}^{N-1}(1-\alpha)^{N-1-i}\,
\cU_{(N-1-i)t_0}(\eta_i),
$$
where, 
$$
\xi_N = \phat^N - \spinf, \quad \text{ and } \quad \eta_N = \phat^{N+1} - \cU_{t_0}(q_N).
$$
Now, projecting perturbation $\xi_N$ onto the Hermite polynomials basis $H_\bn$, since $\cU_{t_0}$ acts on mode $\bn$ by multiplication by $e^{-|\bn|t_0/2}$ (Proposition Prop. \ref{prop:decomposition-Ut-hermite-basis}):
\begin{align*}
\langle \xi_{N+1}, H_\bn \rangle
&= \langle (1-\alpha)\cU_{t_0}(\xi_N) + \eta_N, H_\bn \rangle\\
&= (1-\alpha)\,e^{-|\bn|t_0/2}\,\langle \xi_N, H_\bn \rangle
+ \langle \eta_N, H_\bn \rangle\\
&= \kappa_\bn\,\langle \xi_N, H_\bn \rangle + \langle \eta_N, H_\bn \rangle.
\end{align*}
Unrolling gives~\eqref{eq:spectral-perturbation}.
\end{proof}

\section{Self-Regularization and Annealed Truncation}
\label{app:self-regularization}

In this section, we investigate the well-posedness of annealed truncation schedules, which consists in considering adaptive truncation times $t_0^{(i)}$ that converge to $0$ as  the generation $i \to \infty$.

\subsection{Fisher Information Instability and Properties}
\label{app:subsec-fisher-information-and-prop}

\paragraph{Instability of the reverse diffusion around $t \downarrow 0$} We first formally state the instability problem that arises near the reverse diffusion endpoint $t=0$. Recall that, in order to sample from the reverse diffusion, one needs to estimate the score of the reverse time-$t$ marginal at all times $t$. For a target distribution $q_i$ at generation $i$, we denote the score of the reverse time-$t$ marginal by $\nabla \log q_{i,t}$. The key quantity that captures how hard these are to estimate is the Fisher information of $q_{i,t}$ \cite{max-likelihood-diffusion-models}, defined as the expected squared magnitude of the score:
\begin{equation}
\label{eq:Fisher-information}
J(q_{i,t})=\int \| \nabla \log q_{i,t}(\bx)\|_2^2 \,  q_{i,t}(\bx)\md \bx.
\end{equation}
In diffusion training, the score matching objective \cite{hyvarinen2005estimation, hyvarinen2007_extension, denoising-score-matching-vincent} at generation $i$ contains terms of the form,
\begin{equation}
\label{eq:target-dsm}
\E_{\bx_t \sim q_{i,t}}\big[\|\bs_\theta(\bX_t, t)- \nabla \log q_{i,t}(\bX_t)\|_2^2\big],
\end{equation}
suggesting that if $J(q_{i,t})$ defined in \eqref{eq:Fisher-information} is large, the target score itself \eqref{eq:target-dsm} has large $L^2$-norm. 

Concretely, for OU diffusion, one has 
$$
\cU_t q_{i} = \Law(\bX_t^i) = \Law(e^{-t/2}\bX_0^i + \sqrt{1-e^{-t}}\bZ)
$$
where $\bZ \sim \cN(0, \bI_d)$ and $\bX_0^i \sim q_i$.
Therefore $\cU_t q_{i}$  approaches the raw data distribution $q_i$ as $t \to 0$. If the training distribution is empirical, singular, sharp, multimodal or concentrated near a low-dimensional set, then its density develops steep gradients at small $t$. The generic Fisher bound captures exactly this ((ii) in Proposition \ref{prop:fisher-control-ou}) :
$$
  J(\cU_t q_{i})
  \leq
  \frac{d}{1-e^{-t}}
  \underset{t \to 0}{\sim}
  \frac{d}{t},
  \qquad \text{ since } 1-e^{-t} \underset{t \to 0}{\sim} t.
$$

\begin{proposition}[Fisher control for OU smoothing]
\label{prop:fisher-control-ou}
Let $\mu \in \Probspace(\R^d)$ be absolutely continuous with respect to the
Lebesgue measure on $\R^d$. Denote by $J(\mu)$ its Fisher information,
\begin{equation}
\label{eq:def-Fisher-information}
J(\mu)
=
\int_{\mathbb R^d}
\|\nabla \log \mu(\bx)\|^2 \mu(\bx)\,\md \bx
=
\int_{\mathbb R^d}
\frac{\|\nabla \mu(\bx)\|^2}{\mu(\bx)}\,\md \bx .
\end{equation}
We equivalently write $J(\bX)$ to denote $J(\mu)$ when $\bX\sim\mu$.
Let $(\cU_t)_{t\ge0}$ denote the OU sampling operator (defined in \eqref{eq:Ut-def})
Then the following properties hold.
\begin{enumerate}
    \item[(i)] \textbf{Quadratic scaling under dilation.}
    For any $a\neq0$,
    $$
    J(a\bX)=a^{-2}J(\bX).
    $$

    \item[(ii)] \textbf{Universal OU smoothing bound.}
    For every $t>0$,
    $$
    J(\cU_t\mu)
    \le
    \frac{d}{1-e^{-t}}.
    $$

    \item[(iii)] \textbf{Fisher control under OU evolution.}
    If $J(\mu)<\infty$, then for every $t\ge0$,
    $$
    J(\cU_t\mu)
    \le
    e^t J(\mu).
    $$
    In particular, for $t$ in a bounded interval $[0,T]$,
    $$
    \sup_{0\le t\le T}J(\cU_t\mu)
    \le
    e^T J(\mu).
    $$

    \item[(iv)] \textbf{Convexity under mixtures.}
    Let $\mu_1,\ldots,\mu_N$ have smooth positive densities, and let
    $w_1,\ldots,w_N\ge0$ satisfy $\sum_{k=1}^Nw_k=1$. Then
    $$
    J\left(\sum_{k=1}^Nw_k\mu_k\right)
    \le
    \sum_{k=1}^Nw_kJ(\mu_k).
    $$
\end{enumerate}
\end{proposition}

These results are standard, and detailed proofs can be found in \cite{STAM1959101, johnson-barron-fisher-information-inequalities}. We provide proofs for completeness.

\begin{proof}
\textit{(i)  Quadratic scaling under dilation.}
Let $\bY=a\bX$ with $a\neq0$. If $\bX$ has density $\mu$, then $\bY$ has
density
$$
\mu_a(\by)
=
|a|^{-d}\mu(\by/a).
$$
Therefore,
$$
\nabla_{\by}\log\mu_a(\by)
=
\nabla_{\by}\log\mu(\by/a)
=
a^{-1}\nabla\log\mu(\by/a).
$$
Hence,
\begin{equation}
\label{eq:change-of-variables-a}
J(\bY)= \int_{\R^d}
\|\nabla\log\mu_a(\by)\|^2\mu_a(\by)\,\md\by=
\int_{\R^d}
\left\|a^{-1}\nabla\log\mu(\by/a)\right\|^2
|a|^{-d}\mu(\by/a)\,\md\by .
\end{equation}
By change of variables $\by=a\bx$, \eqref{eq:change-of-variables-a} yields the desired results:
$$
J(\bY)
=
a^{-2}
\int_{\R^d}
\|\nabla\log\mu(\bx)\|^2\mu(\bx)\,\md\bx
=
a^{-2}J(\bX).
$$

\medskip
\textit{(ii)  Universal OU smoothing bound.}
Fix $\mu \in \Probspace(\R^d)$ and let $\bY_t = e^{-t/2}\bX+\sqrt{1-e^{-t}}\bZ$ where $\bX \sim \mu \indep \bZ \sim \cN(0, \bI_d)$. Also, let $\sigma_t^2 = 1-e^{-t}$. Using the standard Fisher-information inequality \cite{information-theoretic-inequalities, johnson-barron-fisher-information-inequalities} $J(\bX_1 + \bX_2) \le J(\bX_2)$ for any two independent random variables $\bX_1$ and $\bX_2$, to $\bX_1 = e^{-t/2}\bX$ and $\bX_2 = \sigma_t \bZ$ gives,
$$
J(\bY_t)
\le
J(\sigma_t\bZ).
$$
Since $\sigma_t\bZ\sim\Normal(0,\sigma_t^2\bI_d)$, its Fisher information is
$$
J(\sigma_t\bZ)
=
\frac{d}{\sigma_t^2}
=
\frac{d}{1-e^{-t}}.
$$
Therefore,
$$
J(\cU_t\mu)
=
J(\bY_t)
\le
\frac{d}{1-e^{-t}}.
$$

\medskip
\textit{(iii)  Fisher control under OU evolution.} Once again, let $\bY_t$ be defined as 
$\bY_t=e^{-t/2}\bX+\sqrt{1-e^{-t}}\bZ$. By the  same monotonicity  property Fisher information used in \textit{(ii)}, 
$$
J(\bY_t)
\le
J(e^{-t/2}\bX).
$$
Using the scaling property proved in \textit{(i)} with $a=e^{-t/2}$, 
$$
J(e^{-t/2}\bX)
=
e^tJ(\bX).
$$
Therefore,
$$
J(\cU_t\mu)
=
J(\bY_t)
\le
e^tJ(\mu).
$$
If $t\in[0,T]$, then $e^t\le e^T$, and hence
$$
\sup_{0\le t\le T}J(\cU_t\mu)
\le
e^TJ(\mu).
$$

\medskip
\textit{(iv) Convexity under mixtures.}
Let $\mu_1,\ldots,\mu_N$ denote both the probability measures and their smooth
positive densities and define the convex mixture distribution $\mu = \sum_{k=1}^Nw_k\mu_k$. Then, by taking the gradient of $\mu$,
$$
\nabla\mu
=
\sum_{k=1}^Nw_k\nabla\mu_k .
$$
Noticing that the map $(u,\bv)\mapsto \frac{\|\bv\|^2}{u}$ for $u>0$ and $\bv \in \R^d$ is convex, one has pointwise in $\bx$ that
$$
\frac{\|\nabla\mu(\bx)\|^2}{\mu(\bx)}
=
\frac{
\left\|
\sum_{k=1}^Nw_k\nabla\mu_k(\bx)
\right\|^2
}{
\sum_{k=1}^Nw_k\mu_k(\bx)
}
\le
\sum_{k=1}^Nw_k
\frac{\|\nabla\mu_k(\bx)\|^2}{\mu_k(\bx)}.
$$
Integrating over $\R^d$ gives the final result,
$$
J(\mu)
=
\int_{\R^d}
\frac{\|\nabla\mu(\bx)\|^2}{\mu(\bx)}
\,\md\bx
\le
\sum_{k=1}^Nw_k
\int_{\R^d}
\frac{\|\nabla\mu_k(\bx)\|^2}{\mu_k(\bx)}
\,\md\bx  = \sum_{k=1}^Nw_kJ(\mu_k).
$$
\end{proof}

\subsection{Annealed Truncation}

The  instability of score estimation near $t=0$ highlights the limits of considering truncation times $t_0^{(i)}$ that tend to zero. The recursive structure partially mitigates this problem, because synthetic samples have already passed through previous OU smoothing steps. However, it does not imply that the whole synthetic distribution has accumulated the full cumulative smoothing \(\sum_{j<i}t_0^{(j)}\) since, in our setting, fresh data are reintroduced at each generation, and recently injected components have only been smoothed for the most recent truncation times. The following decomposition makes this precise.

\begin{proposition}[Fisher decomposition under annealed truncation]
\label{prop:fisher-decomposition-annealed}
Consider the error-free annealed recursion with schedules $(t_0^{(i)})_{i \geq 0}$,
\begin{equation}
\label{eq:expression-phat-i-annealed-recursion}
\phat^{i+1}
=
\cU_{t_0^{(i)}}\bigl(
\alpha \pdata+(1-\alpha)\phat^i
\bigr),
\qquad
\phat^0=\pdata,
\end{equation}
where $t_0^{(i)}>0$. Then, for each generation $N\geq1$, 
\begin{equation}
\label{eq:fisher-info-annealed-truncation-upper-bound}
\begin{split}
  J(q_{N,t_0^{(N)}})
  & \leq
  \alpha J(\cU_{t_0^{(N)}} \pdata) 
  +
  (1-\alpha)\alpha
  \sum_{m=0}^{N-1}
  (1-\alpha)^{m}
  J(\cU_{t_0^{(N)}+\sigma_{m,N}}\pdata) \\
  & \qquad 
  +  (1-\alpha)^{N +1}
  J(\cU_{t_0^{(N)}+s_{0,N}}\pdata),
  \end{split}
\end{equation}
where, as in Theorem~\ref{thm:annealed-truncation-corrected}, $\sigma_{m,N}:=\sum_{\ell=N-1-m}^{N-1} t_0^{(\ell)}$ is the smoothing accumulated by the fresh data injected at generation $N-1-m$, and $s_{0,N}:=\sum_{\ell=0}^{N-1} t_0^{(\ell)} = \sigma_{N-1,N}$.
\end{proposition}

\begin{proof}
By Theorem \ref{thm:annealed-truncation-corrected}, for any finite horizon $N \ge 1$, $\phat^N$ can be written as 
\begin{equation}
\label{eq:expression-phat-i-annealed-recursion-proof}
\phat^N
=
\alpha \sum_{m=0}^{N-1} (1-\alpha)^m\,\cU_{\sigma_{m,N}}\pdata
+
(1-\alpha)^N \cU_{s_{0,N}}\pdata,
\end{equation}
For any $t \ge 0$, $q_{N,t} := \cU_t q_N$ can be written as
\begin{align*}
  q_{N,t}
  &=
  \cU_t(\alpha \pdata + (1-\alpha)\phat^N) = 
  \alpha \cU_t \pdata
  +
  (1-\alpha)\cU_t\phat^N, 
\end{align*}
where we have used the linearity of $\cU_t$ (Proposition \ref{prop:OU-sampling-operator-properties}).
Thus, plugging the expression of $\phat^N$ given in \eqref{eq:expression-phat-i-annealed-recursion-proof} and using linearity again, we obtain
\begin{align}
\label{eq:expression-qNt-fisher}
  q_{N,t} 
  &= 
  \alpha \cU_t \pdata
  +
  (1-\alpha)\cU_t\phat^N \nonumber\\
  &=
  \alpha \cU_t \pdata 
  +
  \alpha
  \sum_{m=0}^{N-1}
  (1-\alpha)^{m+1}
  \cU_{t+\sigma_{m,N}}\pdata
  +
  (1-\alpha)^{N+1}
  \cU_{t+s_{0,N}} \pdata.
\end{align}
The Fisher bound \eqref{eq:fisher-info-annealed-truncation-upper-bound} follows from \eqref{eq:expression-qNt-fisher} using the convexity of Fisher information under mixtures (Proposition~\ref{prop:fisher-control-ou}\textit{(iv)}) and taking $t = t_0^{(N)}$.
\end{proof}

Proposition~\ref{prop:fisher-decomposition-annealed} clearly differentiates between two terms. On the one hand, components injected many generations in the past have accumulated several steps of smoothing and are regularized by the recursion. On the other hand, components injected more recently have less smoothing, and the fresh component 
$\alpha\cU_t \pdata$ has only been smoothed by the current truncation time $t_0^{(N)}$. Therefore, recursive self-regularization does not by itself give a uniform Fisher bound as $t \downarrow 0$ for arbitrary raw data. We distinguish the settings in which these annealed truncations are well-defined.

\subsubsection{Case 1: Under \texorpdfstring{\ref{A3}}{(A3)} -- finite Fisher information}
Under the standing assumption \ref{A3}, annealed truncation schedules are well-defined since the unknown target distribution $\pdata$ has finite Fisher information.

\begin{corollary}[Annealing under finite-Fisher regularity]
\label{cor:annealing-finite-fisher}
If $\pdata$ has finite Fisher information, i.e. if $J(\pdata)<\infty$, then, for a bounded positive truncation schedule $(t_0^{(N)})_{N\ge0}$,
$$
\sup_N J\bigl(q_{N,t_0^{(N)}}\bigr)<\infty .
$$
\end{corollary}

\begin{proof}
Let \(T_0:=\sup_N t_0^{(N)}<\infty\). By
\eqref{eq:expression-qNt-fisher}, evaluated at \(t=t_0^{(N)}\), we have
\begin{align*}
q_{N,t_0^{(N)}}
&=
\alpha \cU_{t_0^{(N)}}\pdata
+
\alpha
\sum_{m=0}^{N-1}
(1-\alpha)^{m+1}
\cU_{t_0^{(N)}+\sigma_{m,N}}\pdata
+
(1-\alpha)^{N+1}
\cU_{t_0^{(N)}+s_{0,N}}\pdata .
\end{align*}
The coefficients in this decomposition are non-negative and sum to one. Hence, by
convexity of Fisher information under mixtures,
\begin{align}
\label{eq:convex-decompositio-case1-proof}
J(q_{N,t_0^{(N)}})
&\le
\alpha J(\cU_{t_0^{(N)}}\pdata)
+
\alpha
\sum_{m=0}^{N-1}
(1-\alpha)^{m+1}
J(\cU_{t_0^{(N)}+\sigma_{m,N}}\pdata) \nonumber \\
&\qquad
+
(1-\alpha)^{N+1}
J(\cU_{t_0^{(N)}+s_{0,N}}\pdata).
\end{align}

We now show that every term appearing on the right-hand side is uniformly
bounded. Since $J(\pdata)<\infty$, Proposition~\ref{prop:fisher-control-ou}(iii) implies that, for $0\le s\le T_0$,
$$
J(\cU_s\pdata)\le e^{T_0}J(\pdata).
$$
On the other hand, for $s\ge T_0$, Proposition~\ref{prop:fisher-control-ou}(ii) gives the universal smoothing bound
$$
J(\cU_s\pdata)
\le
\frac{d}{1-e^{-s}}
\le
\frac{d}{1-e^{-T_0}} .
$$
Consequently, for all $s\ge0$,
$$
J(\cU_s\pdata)
\le
C_{T_0,\pdata}
:=
\max\left\{
e^{T_0}J(\pdata),
\frac{d}{1-e^{-T_0}}
\right\}.
$$
Applying this bound to each component in the convex decomposition \eqref{eq:convex-decompositio-case1-proof} yields
$$
J(q_{N,t_0^{(N)}})\le C_{T_0,\pdata}
$$
for every $N$, which yields the result.  
\end{proof}

\begin{remark}[Interpretation of the finite Fisher information assumption ($J(\pdata)<\infty$)]
The condition $J(\pdata)<\infty$ is the regularity assumption \ref{A3} of the main text. It fails on raw empirical measures which have no Lebesgue densities and hence no classical scores, and may also not be satisfied for idealized distributions supported on low-dimensional sets. Hence, in such cases, annealing the truncation all the way to zero is not justified by Fisher
information control alone.
\end{remark}

\subsubsection{Case 2: Beyond \texorpdfstring{\ref{A3}}{(A3)} -- target distribution without finite Fisher information}
\label{app:case2}

We distinguish between two alternatives in the case where $J(\pdata) < \infty$ fails.

    1. \textit{Strictly positive truncation floor.} A first option is to consider truncation schedules $(t_0^{(N)})_{N\ge0}$ that are decreasing down to a small positive floor $\tinf>0$ for which Fisher scores remain bounded independently of the regularity of $\pdata$, by ((ii)-Proposition \ref{prop:fisher-control-ou}),
    $$
      J(\cU_{\tinf} \pdata)
      \leq
      \frac{d}{1-e^{-\tinf}}
    $$

    2. \textit{Coupled annealing of fresh data and truncation.}
    A way to remove the finite-Fisher-information requirement on
    $\pdata$ while still letting the truncation time vanish is to
    anneal the fresh-data proportion together with the truncation schedule. Consider
    the generation-dependent recursion
    $$
        \phat^{N+1}
        =
        \cU_{t_0^{(N)}}
        \Bigl(
            \alpha_N \pdata
            +
            (1-\alpha_N)\phat^N
        \Bigr),
        \qquad
        t_0^{(N)}\downarrow 0,
        \qquad
        \alpha_N\downarrow 0 .
    $$
    For any generation $N\ge1$, linearity of $\cU_{t_0^{(N)}}$ (Proposition \ref{prop:OU-sampling-operator-properties}) and convexity of Fisher information (Proposition \ref{prop:fisher-control-ou}), 
    \begin{align*}
        J(\phat^{N+1})
        &=
        J\!\left(
            \alpha_N \mathcal U_{t_0^{(N)}}p_{\mathrm{data}}
            +
            (1-\alpha_N)\mathcal U_{t_0^{(N)}}\phat^N
        \right) \\
        &\le
        \alpha_N J(\mathcal U_{t_0^{(N)}}p_{\mathrm{data}})
        +
        (1-\alpha_N)J(\mathcal U_{t_0^{(N)}}\phat^N).
    \end{align*}
    The first term is the only term involving freshly injected, potentially
    irregular data. By the universal OU smoothing bound ((ii) in Proposition \ref{prop:fisher-control-ou}),
    $$
        \alpha_N J(\mathcal U_{t_N}p_{\mathrm{data}})
        \le
        \alpha_N\frac{d}{1-e^{-t_N}}.
    $$
    Consequently, the fresh-data contribution remains uniformly bounded provided
    $$
        \sup_{N\ge0}
        \frac{\alpha_N}{1-e^{-t_N}}
        <\infty,
    $$
    or equivalently, since $1-e^{-t_N}\sim t_N$ as $t_N \to 0$,
    $$
        \alpha_N=\mathcal O(t_N)
        \qquad\text{as }t_N\downarrow0.
    $$
    
    This condition controls the newly injected raw-data component. To obtain a uniform Fisher bound for the whole training distribution, one must additionally control the synthetic term
    $$
        (1-\alpha_N)J(\mathcal U_{t_N}\phat^N),
    $$
    either by an induction argument, by an explicit unrolling of the recursion, or by imposing a stronger global summability condition on the whole sequence of weights and accumulated smoothing times. Thus, coupled annealing of $\alpha_N$ and $t_N$ provides an alternative to $J(p_{\mathrm{data}})<\infty$ for the fresh component, but full uniform Fisher control requires the same convexity argument to be applied to all mixture components.

\section{Experiments}
\label{app:experiments}

We validate our theoretical results on (i) synthetic 2D Gaussian mixtures, where the limiting distribution $\spinf$ (Theorem~\ref{thm:collapse-distrib-existence-uniqueness}) can be computed in closed form  and serves as a controlled environment for Corollary~\ref{cor:limiting-behavior}; (ii) CIFAR-10 \cite{cifar10} with fixed truncation, where we test the geometric rate of convergence to $\spinf$ (Theorems~\ref{thm:collapse-distrib-existence-uniqueness} and \ref{thm:convergence-imperfect-regime}) and the spectral structure of collapse (Proposition~\ref{prop:spectral-rpz-limit}); and (iii) CIFAR-10 with annealed truncation schedules, where we test the recovery of $\pdata$ predicted by Theorem~\ref{thm:annealed-truncation-corrected}. All synthetic two-dimensional experiments were run on CPU workers, whereas image experiments (on CIFAR-10) were run on a single NVIDIA RTX 6000 Blackwell GPU with 96GB memory; a full 8-generation CIFAR-10 recursive training run required approximately 12 hours of wall-clock time.

\subsection{Synthetic 2D Gaussian mixtures}
\label{sec:exp-synthetic}

We set $\pdata$ to a three-component isotropic mixture in $\R^2$ and run the error-free recursion \eqref{eq:fixed-point-problem} for $50$ generations at $\alpha \in \{0.1, 0.5, 0.9\}$, using a Variance Preserving SDE \cite{song2021scorebased} with $\beta\equiv 1$ so that the OU semigroup matches our analysis exactly and the contraction constant $\kappa = \sqrt{1-\alpha}\,e^{-t_0/2}$ holds without approximation. For $\pdata$ a Gaussian mixture, $\spinf$ in  \eqref{eq:neumann} is itself a countably-infinite mixture of Gaussians  that can be sampled from KL divergences and admit accurate Monte-Carlo estimators.

\textit{Limiting behaviors (Corollary~\ref{cor:limiting-behavior}).} Figure~\ref{fig:limiting-behaviours-2d-gmm} sweeps $\alpha \in [0.02, 0.98]$ at fixed $t_0$ (left and center panels) and $t_0 \in [0, 5]$ at fixed $\alpha$ (right panel), computing $\spinf$ analytically at each grid point. The three predicted limits are recovered: $\mathrm{KL}(\spinf \,\|\, \cU_{t_0}\pdata) \to 0$ as $\alpha \to 1$, $\mathrm{KL}(\spinf \,\|\, \cN(0,\bI)) \to 0$ as $\alpha \to 0$, and $\mathrm{KL}(\spinf \,\|\, \pdata) \to 0$ as $t_0 \to 0$, with strict positivity at every interior $(\alpha, t_0)$.

\textit{Truncation schedules (Theorem~\ref{thm:annealed-truncation-corrected}).} Figure~\ref{fig:truncation-2d-gmm} compares four truncation schedules over $50$ generations at $\alpha \in \{0.1, 0.5, 0.9\}$: a fixed schedule $t_0^{(i)} \equiv t_0$, a shifted schedule with $t_0^{(i)} \to t_\infty = 0.2 > 0$,and two annealed schedules $t_0^{(i)} = t_0/(1+i)^\beta$ with $\beta \in \{1, 2\}$, both converging to $0$ (only the latter is summable; Theorem~\ref{thm:annealed-truncation-corrected} requires only $t_0^{(i)}\to0$). The fixed schedule plateaus at the predicted irreducible floor $\mathrm{KL}(\spinf, \pdata)$, and the shifted schedule plateaus at the lower predicted limit $\mathrm{KL}(p_\infty^{(t_\infty)}, \pdata)$, both consistent with the asymptotic statement of Theorem~\ref{thm:annealed-truncation-corrected}. The two annealed schedules drive the recursion back toward $\pdata$, with the decay rate increasing in $\beta$, confirming that recursive compounding is asymptotically eliminated whenever $t_0^{(i)} \to 0$. The qualitative behavior is consistent across $\alpha$, with smaller $\alpha$ producing slower convergence, as expected from the contraction constant $\kappa = \sqrt{1-\alpha}\,e^{-t_0/2}$.

\begin{figure}[t]
    \centering
    \includegraphics[width=1.0\linewidth]{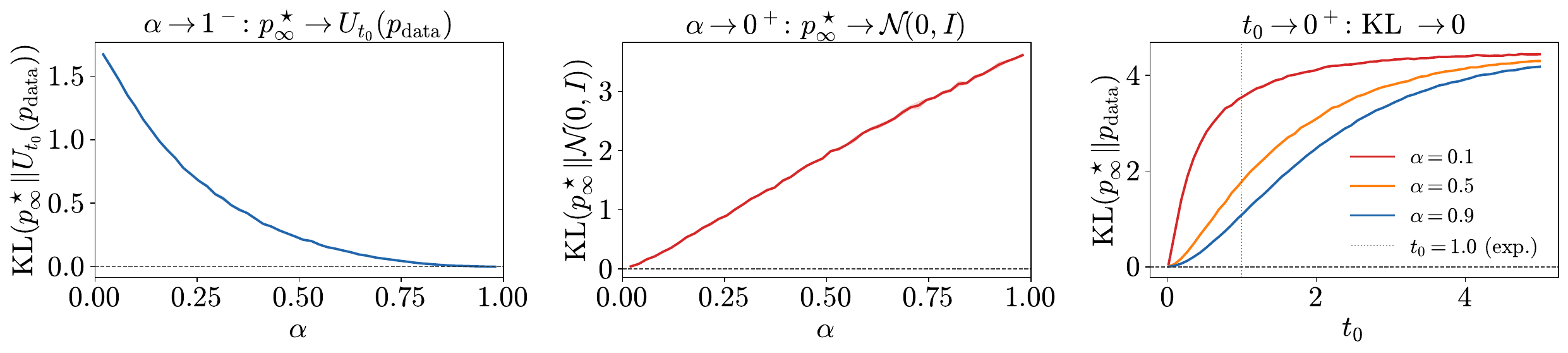}
    \caption{\textbf{Limiting behaviours of $\spinf$ on a 2D Gaussian mixture.} The closed-form Neumann-series representation of $\spinf$ (Theorem~\ref{thm:collapse-distrib-existence-uniqueness}) recovers the three predicted limits of Corollary~\ref{cor:limiting-behavior}: convergence to a single OU-smoothed copy of $\pdata$ as $\alpha \to 1$ (i), convergence to the standard Gaussian as $\alpha \to 0$ (ii), and recovery of $\pdata$ as $t_0 \to 0$ (iii); strict positivity at all interior $(\alpha, t_0)$ is visible in panel (iii).}
    \label{fig:limiting-behaviours-2d-gmm}
\end{figure}

\begin{figure}[t]
    \centering
    \includegraphics[width=1.0\linewidth]{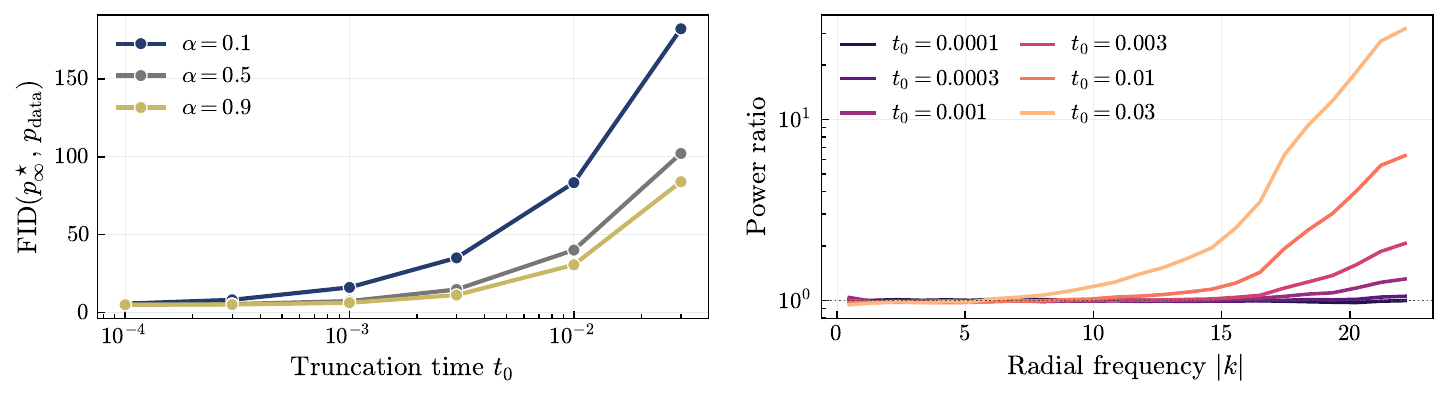}
    \caption{\textbf{Truncation time controls the severity of collapse (CIFAR-10 \cite{cifar10}).}
    \textit{Left}: FID between the limiting collapse distribution \(p_\infty^\star\) and the data distribution \(\pdata\) as a function of the reverse-time truncation \(t_0\), for several fresh-data proportions \(\alpha\). Increasing \(t_0\) amplifies the bias induced by each recursive generation, and the effect is strongest when \(\alpha\) is small, i.e. when synthetic samples dominate the next training distribution.
    \textit{Right}: absolute radial Fourier power of $\spinf$, divided by that of $\pdata$, for increasing values of $t_0$. Since the ratio is not normalized to unit total spectral mass, values above one indicate excess broadband power induced by the truncation noise rather than a redistribution of fixed spectral energy. As $t_0 \to 0$, the ratio approaches one, consistently with the recovery of
    $\pdata$.
    }
    \label{fig:truncation-affect-collapse}
\end{figure}

\subsection{CIFAR-10: limiting collapse and spectral signature}
\label{sec:exp-cifar-fixed}

\paragraph{Setup.} We use a U-Net \cite{unet-seminal} score network 
with self-attention \cite{vaswani2017attention} at the two coarsest spatial scales, trained with the min-SNR-$\gamma$ objective \cite{minSNR} ($\gamma=5$)  and EMA decay $0.9999$. Each generation is trained for $300$ epochs at  $\mathrm{lr}=2\times 10^{-4}$ with cosine decay, batch size $128$, and  AdamW \cite{adamw} (weight decay $10^{-2}$). The truncation time is fixed  at $t_0 = 10^{-3}$. At generation $i \ge 1$, the training loader returns each  example with probability $\alpha$ from CIFAR-10 and probability $1-\alpha$  from a pool of $50{,}000$ samples drawn from generation $i-1$ via  DPM-Solver++ \cite{DPM-solver++} with $50$ steps. We run $8$ generations  for $\alpha \in \{0.1, 0.3, 0.5, 0.7, 0.9\}$. FID \cite{Heuseletal2017} is computed against the full CIFAR-10 training split using $30{,}000$ generated samples per checkpoint.

\paragraph{Empirical proxy for $\spinf$.} The collapse distribution $\spinf$ is not directly accessible, since $\pdata$ is unknown. We consider an empirical proxy of $\spinf$ by approximating $\pdata \approx \phat_\mathrm{trained}$, where $\phat_\mathrm{trained}$ is a well-trained diffusion model on CIFAR-10, solely on real data, and from which one can thus sample. This allows to sample from $\spinf$ by leveraging its closed-form formula (infinite mixture of Gaussian-smoothed version of $\pdata$), which is required to compute the $\mathrm{FID}$.

\paragraph{High-frequency energy.} For a probability measure $p$ on 
$\R^{3 \times 32 \times 32}$ approximated by $N_{\mathrm{spec}} = 2000$ 
samples $\{x^{(j)}\}_{j=1}^{N_{\mathrm{spec}}}$, we define the 
\emph{high-frequency energy} as
\begin{equation}
\label{eq:hf-energy-def}
\mathrm{HF}(p)
\;:=\;
\frac{1}{N_{\mathrm{spec}}}
\sum_{j=1}^{N_{\mathrm{spec}}}
\frac{1}{|\mathcal{A}|}
\sum_{\mathbf{k} \in \mathcal{A}}
\bigl| \mathcal{F}[\bar{x}^{(j)}](\mathbf{k}) \bigr|^2,
\end{equation}
where $\bar{x}^{(j)} = 0.2989\,x^{(j)}_R + 0.5870\,x^{(j)}_G + 0.1140\,x^{(j)}_B$  is the luminance projection of the $j$-th sample, $\mathcal{F}$ is the  zero-centered 2D discrete Fourier transform, and $\mathcal{A} :=  \{\mathbf{k} \in \mathbb{Z}^2 : \|\mathbf{k}\|_2 \ge 15\}$ is the  high-frequency annulus in the $32 \times 32$ Fourier plane (cutoff radius  $k^\star = 15$, slightly below the Nyquist limit of $16$). We define the  \emph{high-frequency deficit} of generation $N$ as
\begin{equation}
\label{eq:hf-deficit-def}
\mathrm{D}_{\mathrm{HF}}(\phat^N)
\;:=\;
1 - \frac{\mathrm{HF}(\phat^N)}{\mathrm{HF}(p_{\mathrm{data}})},
\end{equation}
where $\mathrm{HF}(p_{\mathrm{data}})$ is computed on $N_{\mathrm{spec}}$ real CIFAR-10 images. Larger $\mathrm{D}_{\mathrm{HF}}$ indicates stronger loss of fine-scale Fourier content relative to the data. Note that $\mathrm{HF}$ is quadratic in the samples, so it depends on $p$ only through its second-moment matrix, i.e.\ through Hermite coefficients of degree at most $2$; it is therefore not a proxy for the high-degree Hermite attenuation of Proposition~\ref{prop:spectral-rpz-limit}. In particular, Proposition~\ref{prop:moments-collapse} gives
$$
\mathrm{HF}(\spinf) = c_2\,\mathrm{HF}(\pdata) + (1-c_2)\,\mathrm{HF}(\gamma),
$$
so truncation alone \emph{increases} high-frequency power whenever $\mathrm{HF}(\gamma) > \mathrm{HF}(\pdata)$, as is the case for CIFAR-10 (Figure~\ref{fig:truncation-affect-collapse}, right).

\paragraph{Convergence to $\spinf$ 
(Figure~\ref{fig:cifar-fid-pinf}, Theorem~\ref{thm:convergence-imperfect-regime}).} 
We report $\mathrm{FID}(\phat^N, \phat_\infty^{\,(\alpha)})$, where $\phat_\infty^{\,(\alpha)}$ denotes the limiting distribution of the recursion with proportion $\alpha$, across the  first $8$ recursive generations for each $\alpha$. The contraction is fastest  at large $\alpha$ and slowest at $\alpha = 0.1$, where the trajectory is  clearly transient: the iterate drops from $\mathrm{FID} \approx 33$ at $N=1$  to $\mathrm{FID} \approx 5$ by $N=8$. By the eighth generation, all 
trajectories settle near a common empirical floor of $\mathrm{FID} \approx 5$, which we interpret as the residual $\mathrm{FID}(\phat_\infty^{\,(\alpha)},\, \spinf)$ plus measurement noise from finite-sample FID. This pattern is consistent with the contraction-rate ordering predicted by $\kappa(\alpha) = \sqrt{1-\alpha}\,e^{-t_0/2}$ in Theorem~\ref{thm:convergence-imperfect-regime} and with the stability statement (ii) of the same theorem.

\paragraph{Spectral signature 
(Figure~\ref{fig:cifar-hf-heatmap}, 
Propositions~\ref{prop:spectral-rpz-limit} and \ref{prop:spectral-perturbation}).} 
The high-frequency deficit $\mathrm{D}_{\mathrm{HF}}(\phat^N)$ across the first $8$ generations is broadly increasing in both $N$ and $1-\alpha$. At $\alpha = 0.1$, the deficit reaches $\approx 0.5$ by $N = 8$, indicating that roughly half of the high-frequency Fourier energy of CIFAR-10 has been suppressed; at $\alpha = 0.9$, the deficit remains below $0.15$ over the same horizon. The mode-dependent contraction rate $\kappa_{\bn}(\alpha) = (1-\alpha)\,e^{-|\bn|t_0/2}$ in Proposition \ref{prop:spectral-perturbation} predicts exactly this joint monotonicity in $N$ (faster compounding for larger $|\bn|$) and in $1-\alpha$ (smaller fresh-data injection magnifies attenuation).

\paragraph{Truncation severity 
(Figure~\ref{fig:truncation-affect-collapse}, Corollary~\ref{cor:limiting-behavior}).} The FID gap $\mathrm{FID}(\spinf, p_{\mathrm{data}})$ grows with $t_0$ and decays toward zero as $t_0 \to 0$, while the radial Fourier spectrum of $\spinf$ approaches that of $p_{\mathrm{data}}$ at all frequencies as $t_0$ shrinks.

\end{document}